%% file: icml23_RepLinearBandit_main.tex
\documentclass[nohyperref]{article}

\usepackage{microtype}
\usepackage{graphicx}
\usepackage{subfigure}
\usepackage{booktabs} 

\usepackage{hyperref}

 \usepackage[accepted]{icml2023}

\usepackage{amsmath}
\usepackage{amssymb}
\usepackage{mathtools}
\usepackage{amsthm}

\usepackage[capitalize,noabbrev]{cleveref}

\theoremstyle{plain}
\newtheorem{theorem}{Theorem}[section]

\newtheorem{lemma}[theorem]{Lemma}

\theoremstyle{definition}

\newtheorem{assumption}[theorem]{Assumption}
\theoremstyle{remark}

\usepackage[textsize=tiny]{todonotes}

\usepackage{ifthen}
\usepackage{thm-restate}
\usepackage{tikz}
\usetikzlibrary{positioning,chains,fit,shapes,calc}

\usepackage{amsmath}
\usepackage{amssymb}
\usepackage{amsfonts}
\usepackage{enumerate}
\usepackage{flushend}
\usepackage{mathrsfs}
\usepackage{subfigure}
\usepackage{booktabs}
\usepackage{makecell}
\usepackage{setspace}
\usepackage{xcolor}
\usepackage{braket}
\usepackage{comment}
\usepackage{bbm}
\usepackage{bm}
\usepackage{titletoc}
\usepackage[page, header]{appendix}
\hypersetup{
	colorlinks = true,
	urlcolor   = black,
	linkcolor  = black,
	citecolor  = black
}

\newcommand{\R}{\mathbb{R}}

\newcommand{\cA}{\mathcal{A}}

\newcommand{\cD}{\mathcal{D}}
\newcommand{\cE}{\mathcal{E}}
\newcommand{\cF}{\mathcal{F}}
\newcommand{\cG}{\mathcal{G}}
\newcommand{\cH}{\mathcal{H}}

\newcommand{\cJ}{\mathcal{J}}
\newcommand{\cK}{\mathcal{K}}
\newcommand{\cL}{\mathcal{L}}

\newcommand{\cP}{\mathcal{P}}
\newcommand{\cQ}{\mathcal{Q}}

\newcommand{\cS}{\mathcal{S}}

\newcommand{\cX}{\mathcal{X}}
\newcommand{\cY}{\mathcal{Y}}
\newcommand{\cZ}{\mathcal{Z}}

\mathchardef\hyphen="2D
\newcommand{\argmax}{\operatornamewithlimits{argmax}}
\newcommand{\argmin}{\operatornamewithlimits{argmin}}
\newcommand{\ex}{\mathbb{E}}
\newcommand{\bs}{\boldsymbol}

\newcommand{\trace}{\textup{Trace}}
\newcommand{\batch}{\textup{batch}}

\newcommand{\sbr}[1]{\left( #1 \right)}
\newcommand{\mbr}[1]{\left[ #1 \right]}
\newcommand{\lbr}[1]{\left\{ #1 \right\}}
\newcommand{\abr}[1]{\left| #1 \right|}
\newcommand{\nbr}[1]{\left\| #1 \right\|}
\newcommand{\indicator}[1]{\mathbbm{1}\left\{ #1 \right\}}

\newcommand{\algrepbailb}{\mathtt{DouExpDes}}
\newcommand{\algrepbpiclb}{\mathtt{C \hyphen DouExpDes}}
\newcommand{\round}{\mathtt{ROUND}}
\newcommand{\probbai}{\textup{RepBAI-LB}}
\newcommand{\probbpi}{\textup{RepBPI-CLB}}
\newcommand{\polylog}{\textup{polylog}}
\newcommand{\algfeatrecover}{\mathtt{FeatRecover}}
\newcommand{\algelimlowrep}{\mathtt{EliLowRep}}
\newcommand{\algconfeatrecover}{\mathtt{C \hyphen FeatRecover}}
\newcommand{\algestlowrep}{\mathtt{EstLowRep}}

\newcommand{\compilehidecomments}{true}
\ifthenelse{ \equal{\compilehidecomments}{true} }{%
	\newcommand{\wen}[1]{}
	\newcommand{\longbo}[1]{}
	\newcommand{\yihan}[1]{}
	\newcommand{\cameraReady}[1]{#1}
}{
	\newcommand{\wen}[1]{{\color{red}  [\text{Wen:} #1]}}
	\newcommand{\longbo}[1]{{\color{purple} [\text{Longbo:} #1]}}
	\newcommand{\yihan}[1]{{\color{teal} [\text{Yihan:} #1]}}
	\newcommand{\cameraReady}[1]{{\color{blue} #1}}
}

\newcommand{\compilefullversion}{true} 
\ifthenelse{\equal{\compilefullversion}{false}}{%
	\newcommand{\OnlyInFull}[1]{}
	\newcommand{\OnlyInShort}[1]{#1}
}{%
	\newcommand{\OnlyInFull}[1]{#1}%
	\newcommand{\OnlyInShort}[1]{}%
}%

\allowdisplaybreaks
\icmltitlerunning{Multi-task Representation Learning for Pure Exploration in Linear Bandits}

\begin{document}
	
	
	\twocolumn[
	\icmltitle{Multi-task Representation Learning for Pure Exploration in Linear Bandits}
	
	
	
	\icmlsetsymbol{equal}{*}
	
	\begin{icmlauthorlist}
		\icmlauthor{Yihan Du}{thu}
		\icmlauthor{Longbo Huang}{thu}
		\icmlauthor{Wen Sun}{cornell}
	\end{icmlauthorlist}
	
	\icmlaffiliation{thu}{IIIS, Tsinghua University}
	\icmlaffiliation{cornell}{Cornell University}
	
	\icmlcorrespondingauthor{Yihan Du}{duyh18@mails.tsinghua.edu.cn}
	\icmlcorrespondingauthor{Longbo Huang}{longbohuang@tsinghua.edu.cn}
	\icmlcorrespondingauthor{Wen Sun}{ws455@cornell.edu}
	
	\icmlkeywords{Multi-task Representation Learning, Linear Bandits, Pure Exploration, Experimental Design}
	
	\vskip 0.3in
	]
	
	
	
	\printAffiliationsAndNotice{}  
	
	\begin{abstract}
		Despite the recent success of representation learning in sequential decision making, the study of the pure exploration scenario (i.e., identify the best option and minimize the sample complexity) is still limited. In this paper, we study multi-task representation learning for best arm identification in linear bandits ($\probbai$) and best policy identification in contextual linear bandits ($\probbpi$), two popular pure exploration settings with wide applications, e.g., clinical trials and web content optimization.
		In these two problems, all tasks share a common low-dimensional linear representation, and our goal is to leverage this feature to accelerate the best arm (policy) identification process for all tasks.
		For these problems, we design computationally and sample efficient algorithms $\algrepbailb$ and $\algrepbpiclb$, which perform double experimental designs to plan optimal sample allocations for learning the global representation.
		We show that by learning the common representation among tasks, our sample complexity is significantly better than that of the native approach which solves tasks independently. 
		To the best of our knowledge, this is the first work to demonstrate the benefits of representation learning for multi-task pure exploration. 
	\end{abstract}

	\section{Introduction}
	
	Multi-task representation learning~\cite{caruana1997multitask} is an important problem which aims to learn a common low-dimensional representation from multiple related tasks. Representation learning has received extensive attention in both empirical applications~\cite{ando2005framework,bengio2013representation,li2014joint} and theoretical study~\cite{maurer2016benefit,du2020few,tripuraneni2021provable}.
	
	Recently, an emerging number of works~\cite{yang2020impact,yang2022nearly,hu2021near,cella2022multi} investigate representation learning for sequential decision making, and show that if all tasks share a joint low-rank representation, then by leveraging such a joint representation, 
	it is possible to learn faster than treating each task independently. 
	Despite the accomplishments of these works, they mainly focus on the regret minimization setting, where the performance is measured by the cumulative reward gap between the optimal option and the actually chosen options. 
	
	
	However, in real-world applications where obtaining a sample is expensive and time-consuming, e.g., clinical trails~\cite{zhang2012multi},
	it is often desirable to identify the optimal option using as few samples as possible, i.e., we face the \emph{pure exploration} scenario rather than regret minimization.
	Moreover, in many decision-making applications, we often need to tackle multiple related tasks, e.g., treatment planning for different diseases~\cite{bragman2018uncertainty} and content optimization for multiple websites~\cite{agarwal2009explore}, and there
	usually exists a common representation among these tasks, e.g., the features of drugs and the representations of website items. Thus,  we desire to exploit the shared representation among tasks to expedite learning.
	For example, in clinical treatment planning, we want to identify the optimal treatment for multiple diseases, and there exists a joint representation of treatments. In this case, since conducting a clinical trial and collecting a sample is time-consuming, we desire to make use of the shared representation and reduce the number of samples required.  \looseness=-1
	
	Motivated by the above fact, in this paper, we study representation learning for multi-task pure exploration in sequential decision making. Following prior works~\cite{yang2020impact,yang2022nearly,hu2021near}, we consider the linear bandit setting, which is one of the most popular settings in sequential decision making and has various applications such as clinical trials and recommendation systems. 
	Specifically, we investigate two pure exploration problems, i.e., representation learning for best arm identification in linear bandits ($\probbai$) and best policy identification in contextual linear bandits ($\probbpi$). 
	
	In $\probbai$, an agent is given a confidence parameter $\delta$, an arm set $\cX:=\{\bs{x}_1,\dots,\bs{x}_n\} \subseteq \R^{d}$ and $M$ tasks. For each task $m \in [M]$, the expected reward of each arm $\bs{x} \in \cX$ is generated by $\bs{x}^\top \bs{\theta}_m$, where $\bs{\theta}_m \in \R^{d}$ is an underlying reward parameter. There exists an unknown global feature extractor $\bs{B} \in \R^{d \times k}$ and an underlying prediction parameter $\bs{w}_m$  such that $\bs{\theta}_m=\bs{B} \bs{w}_m$  
	for any $m \in [M]$, where $M \gg d \gg k$. We can understand the problem as that all tasks share a joint representation $\bs{f}(\bs{x}):= \bs{B}^\top \bs{x}$ for arms, where the dimension of $\bs{f}(\bs{x})$ is much smaller than that of $\bs{x}$.  The agent sequentially selects arms and tasks to sample, and observes noisy rewards. The goal of the agent is to identify the best arm with the maximum expected reward for each task with confidence $1-\delta$, using as few samples as possible. 
	
	The $\probbpi$ problem is an extension of $\probbai$ to environments with random and varying contexts. 
	In $\probbpi$, there are a context space $\cS$, an action space $\cA$, a known feature mapping $\bs{\phi}:\cS \times \cA \mapsto \R^d$ and an \emph{unknown} context distribution $\cD$. 
	For each task $m \in [M]$, the expected reward of each context-action pair $(s,a) \in \cS \times \cA$ is generated by $\bs{\phi}(s,a)^\top \bs{\theta}_m$, where $\bs{\theta}_m=\bs{B} \bs{w}_m$. 
	We can similarly interpret the problem as that all tasks share a low-dimensional context-action representation $\bs{B}^\top \bs{\phi}(s,a) \in \mathbb{R}^k$.  At each timestep, the agent first observes a context drawn from $\cD$, and chooses an action and a task to sample, and then observes a random reward.
	Given a confidence parameter $\delta$ and an accuracy parameter $\varepsilon$, the agent aims to identify an $\varepsilon$-optimal policy (i.e., a mapping $\cS \mapsto \cA$ that gives suboptimality within $\varepsilon$)
	for each task with confidence $1-\delta$, while minimizing the number of samples used.
	

	In contrast to existing representation learning works~\cite{yang2020impact,yang2022nearly,hu2021near,cella2022multi}, we focus on the pure exploration scenario and face several unique challenges: 
	(i) The sample complexity minimization objective  requires us to plan an optimal sample allocation for recovering the low-rank representation, in order to save samples to the highest degree. 
	(ii) Unlike prior works which either assume that the arm set is an ellipsoid/sphere~\cite{yang2020impact,yang2022nearly} or are computationally inefficient~\cite{hu2021near}, we allow an arbitrary arm set that spans $\R^d$, which poses challenges on how to efficiently schedule samples according to the shapes of arms.
	(iii) Different from prior works~\cite{huang2015efficient,li2022instance}, we do not assume prior knowledge of the context distribution. This imposes additional difficulties in sample allocation planning and estimator construction.
	To handle these challenges, we design computationally and sample efficient algorithms, which effectively estimate the context distribution and employ the experimental design approaches to plan samples.
	
	We summarize our contributions in this paper as follows.
	\begin{itemize}
		\item We formulate the problems of multi-task representation learning for best arm identification in linear bandits ($\probbai$) and best policy identification in contextual linear bandits ($\probbpi$). To the best of our knowledge, this is the first work to study representation learning in the multi-task pure exploration scenario.
		\item For $\probbai$, we propose an efficient algorithm $\algrepbailb$ equipped with \emph{double experimental designs}. The first design optimally schedules samples to learn the joint representation according to arm shapes, and the second design minimizes the estimation error for rewards using low-dimensional representations.
		Furthermore, we establish a sample complexity guarantee $\tilde{O}(\frac{Mk}{\Delta_{\min}^2})$, which shows superiority over the baseline result $\tilde{O}(\frac{Md}{\Delta_{\min}^2})$ (i.e., solving each task independently). Here $\Delta_{\min}$ denotes the minimum reward gap. 
		\item For $\probbpi$, we develop $\algrepbpiclb$, an algorithm which efficiently estimates the context distribution and conducts double experimental designs under the estimated context distribution to learn the global representation. A sample complexity result $\tilde{O}(\frac{Mk^2}{\varepsilon^2})$ is also provided for $\algrepbpiclb$, which significantly outperforms the baseline result $\tilde{O}(\frac{Md^2}{\varepsilon^2})$, and demonstrates the power of representation learning.
	\end{itemize}

	\section{Related Work}
	
	In this section, we introduce two lines of related works, and defer a more complete literature review to Appendix~\ref{apx:related_work}. \looseness=-1 
	
	\textbf{Representation Learning.}
	The study of representation learning has been initiated and developed in the supervised learning setting, e.g., \cite{baxter2000model,ando2005framework,maurer2016benefit,du2020few,tripuraneni2021provable}. 
	
	\cameraReady{
	Recently, representation learning for sequential decision making has attracted extensive attention.
	\citet{lale2019stochastic,jun2019bilinear,lu2021low,huang2021optimal} study linear bandits with a hidden low-rank structure (e.g., bilinear bandits), which is very related to the problem of representation learning.
	\citet{yang2020impact,yang2022nearly,hu2021near,cella2022multi} consider multi-task representation learning for linear bandits with the regret minimization objective. \citet{yang2020impact,yang2022nearly} assume that the arm set is an ellipsoid or sphere. \citet{hu2021near} relax this assumption and allow arbitrary arm sets, but their algorithms that build upon a multi-task joint least-square estimator are computationally inefficient. \citet{cella2022multi} design algorithms that do not need to know the dimension of the underlying representation. 
	%
	There are also other works~\cite{lu2021power,lu2022provable,pacchiano2022joint,zhang2021provably,chengprovable,agarwal2022provable} which investigate representation learning for reinforcement learning.
	}

	Different from the above works which consider regret minimization, we study representation learning for (contextual) linear bandits with the pure exploration objective, which brings unique challenges on how to optimally allocate samples to learn the feature extractor, and motivates us to design algorithms based on double experimental designs.
	
	\textbf{Pure Exploration in (Contextual) Linear Bandits.}
	Most existing linear bandit works focus on regret minimization, e.g.,~\cite{dani2008stochastic,chu2011contextual,abbasi2011improved}. Recently, there has been a surge of interests in the pure exploration objective for (contextual) linear bandits.
	For linear bandits, \citet{soare2014best} firstly apply the experimental design approach to distinguish the optimal arm, and establish sample complexity that heavily depends on the minimum reward gap.
	\citet{tao2018best} design a novel randomized estimator for the underlying reward parameter, and achieve tighter sample complexity which depends on the reward gaps of the best $d$ arms.
	\citet{fiez2019sequential} provide the first near-optimal sample complexity upper and lower bounds for best arm identification in linear bandits. 
	For contextual linear bandits, \citet{zanette2021design} develop a non-adaptive policy to collect data, from which a near-optimal policy can be computed. \citet{li2022instance} build instance-optimal sample complexity for best policy identification in contextual linear bandits, with prior knowledge of the context distribution. By contrast, our work studies a multi-task setting where tasks share a common representation, 
	and does not assume any prior knowledge of the context distribution.

	\section{Problem Formulation}
	
	In this section, we present the formal problem formulations of $\probbai$ and $\probbpi$. Before describing the formulations, we first introduce some useful notations.
	
	\textbf{Notations.}
	We use bold lower-case letters to denote vectors and bold upper-case letters to denote matrices.
	For any matrix $\bs{A}$,  $\|\bs{A}\|$ denotes the spectral norm of $\bs{A}$, and  $\sigma_{\min}(\bs{A})$ denotes the minimum singular value of $\bs{A}$. For any positive semi-definite matrix $\bs{A} \in \R^{d' \times d'}$ and vector $\bs{x} \in \R^{d'}$,  $\|\bs{x}\|_{\bs{A}}:=\sqrt{\bs{x}^\top \bs{A} \bs{x}}$. We use $\polylog(\cdot)$ to denote a polylogarithmic factor in given parameters, and $\tilde{O}(\cdot)$ to denote an expression that hides polylogarithmic factors in all problem parameters except $\delta$ and $\varepsilon$.

	\textbf{Representation Learning for Best Arm Identification in Linear Bandits ($\probbai$).}
	An agent is given a set of arms $\cX:=\{\bs{x}_1,\dots,\bs{x}_n\} \subseteq \R^d$ and $M$ best arm identification tasks. Without loss of generality, we assume that $\cX$ spans $\R^d$, as done in many prior works~\cite{fiez2019sequential,katz2020empirical,degenne2020gamification}. For any $\bs{x} \in \cX$, $\|\bs{x}\|\leq L_{x}$ for some constant $L_{x}$. For each task $m \in [M]$, the expected reward of each arm $\bs{x} \in \cX$ is $\bs{x}^\top \bs{\theta}_m$, where $\bs{\theta}_m  \in \R^d$ is an unknown reward parameter. Among all tasks, there exists a common underlying feature extractor $\bs{B} \in \R^{d \times k}$, which satisfies that for each task $m \in [M]$, $\bs{\theta}_m=\bs{B} \bs{w}_m$. Here $\bs{B}$ has orthonormal columns, $\bs{w}_m \in \R^k$ is an unknown prediction parameter, and $M \gg d \gg k$. For any $m \in [M]$, $\|\bs{w}_m\|\leq L_{w}$ for some constant $L_{w}$. 
	
	At each timestep $t$, the agent chooses an arm $\bs{x} \in \cX$ and a task $m \in [M]$, to sample arm $\bs{x}$ in task $m$. Then, she observes a random reward $r_t=\bs{x}^\top \bs{\theta}_m+\eta_t=\bs{x}^\top \bs{B} \bs{w}_m+ \eta_t$, where $\eta_t$ is an independent, zero-mean and sub-Gaussian noise. For simplicity of analysis, we assume that $\ex[\eta_t^2]=1$, which can be easily relaxed by using a more carefully-designed estimator in our algorithm.
	Given a confidence parameter $\delta \in (0,1)$, the agent aims to identify the best arms $\bs{x}^{*}_m:=\argmax_{\bs{x} \in \cX} \bs{x}^\top \bs{\theta}_m$ for all tasks $m \in [M]$ with probability at least $1-\delta$, using as few samples as possible. We define sample complexity as the total number of samples used over all tasks, which is the performance metric considered in our paper.
	
	To efficiently learn the underlying low-dimensional representation, 
	we make the following standard assumptions.
	
	\begin{assumption}[Diverse Tasks] \label{assumption:diverse_task}
		We assume that $\sigma_{\min}(\frac{1}{M} \sum_{m=1}^{M} \bs{w}_m \bs{w}_m^\top) = \Omega(\frac{1}{k})$.
	\end{assumption}
	
	This assumption indicates that the prediction parameters $\bs{w}_1,\dots,\bs{w}_{M}$  are uniformly spread out in all directions of $\R^k$, which was also assumed in~\cite{du2020few,tripuraneni2021provable,yang2020impact}, and is necessary for recovering the feature extractor $\bs{B}$.
	
	For any distribution $\bs{\lambda} \in \triangle_{\cX}$ and $\bs{B} \in \R^{d \times k}$, let $\bs{A}(\bs{\lambda}, \bs{B}):=\sum_{i=1}^{n} \lambda(\bs{x}_i) \bs{B}^\top \bs{x}_i \bs{x}_i^\top \bs{B}$.
	For any task $m \in [M]$, let
	\begin{align*}
		\bs{\lambda}^*_m := & \argmin_{\bs{\lambda} \in \triangle_{\cX}} \max_{\bs{x} \in \cX \setminus \{\bs{x}^{*}_{m}\}} \frac{\| \bs{B}^\top (\bs{x}^{*}_{m} - \bs{x}) \|^2_{\bs{A}(\bs{\lambda}, \bs{B})^{-1}} }{ ((\bs{x}^{*}_{m} - \bs{x})^\top \bs{\theta}_m)^2 } .
	\end{align*}
	Here $\bs{\lambda}^*_m$ denotes the optimal sample allocation that minimizes prediction error of arms (i.e., the solution of G-optimal design~\cite{pukelsheim2006optimal}) under the underlying low-dimensional representation.
	
	\begin{assumption}[Eigenvalue of G-optimal Design Matrix] \label{assumption:lambda^*_B_x_x_B_invertible}
		For any task $m \in [M]$, $\sigma_{\min}(\bs{A}(\bs{\lambda}^*_m, \bs{B})) \geq \omega$ for some constant $\omega>0$.
	\end{assumption}
	
	This assumption implies that the covariance matrix $\bs{A}(\bs{\lambda}^*_m, \bs{B})$ under the optimal sample allocation is invertible, which is necessary for estimating $\bs{w}_m$.  
	\cameraReady{Note that the quantities introduced in Assumptions~\ref{assumption:diverse_task} and \ref{assumption:lambda^*_B_x_x_B_invertible}, i.e., $\sigma_{\min}(\frac{1}{M} \sum_{m=1}^{M} \boldsymbol{w}_m \boldsymbol{w}_m^\top)$ and $\sigma_{\min}(\bs{A}(\bs{\lambda}^*_m, \bs{B}))$, are both defined on the low-dimensional subspace, which scale as $k$ instead of $d$.}
	
	\textbf{Representation Learning for Best Policy Identification in Contextual Linear Bandits ($\probbpi$).}
	In this problem, there are  a context space $\cS$, an action space $\cA$, a feature mapping $\bs{\phi}(\cdot,\cdot):\cS\times\cA \mapsto \R^d$ and an \emph{unknown} context distribution $\cD \in \triangle_{\cS}$. For any $(s,a) \in \cS \times \cA$, $\|\bs{\phi}(s,a)\|\leq L_{\phi}$ for some constant $L_{\phi}$. 
	An agent needs to solve $M$ best policy identification tasks. For each task $m \in [M]$, the expected reward of each context-action pair $(s,a) \in \cS \times \cA$ is $\bs{\phi}(s,a)^\top \bs{\theta}_m$, where $\bs{\theta}_m \in \R^d$ is an unknown reward parameter. Similar to $\probbai$, there exists a global feature extractor $\bs{B} \in \R^{d \times k}$ with orthonormal columns, such that for each task $m \in [M]$, $\bs{\theta}_m=\bs{B} \bs{w}_m$. Here $\bs{w}_m \in \R^k$ is an unknown prediction parameter, $\|\bs{w}_m\|\leq L_{w}$ for any $m \in [M]$, and $M \gg d \gg k$.
	
	At each timestep $t$, the agent first observes a random context $s_t$, which is i.i.d. drawn from $\cD$. Then, she selects an action $a_t \in \cA$ and a task $m \in [M]$, to sample action $a_t$ in context $s_t$ under task $m$. After sampling, she observes a random reward $r_t=\bs{\phi}(s_t,a_t)^\top \bs{\theta}_m+\eta_t=\bs{\phi}(s_t,a_t)^\top \bs{B} \bs{w}_m+ \eta_t$, where $\eta_t$ is an independent, zero-mean and $1$-sub-Gaussian noise. 
	
	We define a policy $\pi$ as a mapping from $\cS$ to $\cA$. For each task $m \in [M]$, we say a policy $\hat{\pi}_m$ is $\varepsilon$-optimal if\looseness=-1
	\begin{align*}
		\ex_{s \sim \cD} \mbr{ \max_{a \in \cA} \sbr{\bs{\phi}(s,a) - \bs{\phi}(s,\hat{\pi}_m(s) }^\top \bs{\theta}_m } \leq \varepsilon .
	\end{align*}
	Given a confidence parameter $\delta \in (0,1)$ and an accuracy parameter $\varepsilon>0$, the goal of the agent is to identify an $\varepsilon$-optimal policy $\hat{\pi}_m$ for each task $m \in [M]$ with probability at least $1-\delta$, and minimize the number of samples used, i.e., sample complexity.
	
	We also make two standard assumptions for $\probbpi$:  Assumption~\ref{assumption:diverse_task} and the following assumption on the context distribution and context-action features.
	
	\begin{assumption}\label{assumption:bpi_rho_E_D_is_finite}
		There exists some $\bs{\lambda} \in \triangle_{\cA}$ such that  
		$$
		\sigma_{\min}\sbr{\sum_{a \in \cA} \lambda(a) \ex_{s \sim \cD}\mbr{ \bs{\phi}(s,a) \bs{\phi}(s,a)^\top }} \geq \nu 
		$$ 
	\end{assumption}
	for some constant $\nu>0$.
	
	Assumption~\ref{assumption:bpi_rho_E_D_is_finite} manifests that there exists at least one sample allocation, under which the expected covariance matrix with respect to random contexts is invertible. This assumption enables one to reveal the feature extractor $\bs{B}$, despite stochastic and varying contexts. 
	Note that Assumption~\ref{assumption:bpi_rho_E_D_is_finite} only assumes the existence of a feasible sample allocation, rather than the knowledge of this sample allocation. 

	It is worth mentioning that in this work, we do not assume that we can sample arbitrary vectors in an ellipsoid/sphere as in \cite{yang2020impact,yang2022nearly}, or assume that each arm (action) has zero mean and identity covariance as in \cite{tripuraneni2021provable}. In contrast, we allow arbitrary shapes of arms (actions), and efficiently allocate samples according to their different shapes. Moreover, we do not assume prior knowledge of the context distribution as in \cite{huang2015efficient,li2022instance}. 
	Instead, we design an effective scheme to estimate the context distribution, and carefully bound the estimation error in our analysis. 
	
	Below we will introduce our algorithms and results. We defer all our proofs to Appendix due to space limit.

	\section{Representation Learning for Best Arm Identification in Linear Bandits}

	In this section, we design a computationally efficient algorithm $\algrepbailb$ for $\probbai$, which performs double delicate experimental designs to recover the feature extractor and distinguish the best arms using low-rank representations. Furthermore, we provide sample complexity guarantees that mainly depend on the underlying low dimension.
	
	To better describe our algorithm, we first introduce the notion of \emph{experimental design}.
	Experimental design is an important problem in statistics~\cite{pukelsheim2006optimal}. Consider a set of feature vectors and an unknown linear regression parameter. Sampling each feature vector will produce a noisy feedback of the inner-product of this feature vector and the unknown parameter.
	Experimental design investigates how to schedule samples to maximize the statistical power of estimating the unknown parameter. 
	In our algorithm, we mainly use two popular types of experimental design, i.e., \emph{E-optimal design}, which minimizes the spectral norm of the inverse of sample covariance matrix, and \emph{G-optimal design}, which minimizes the maximum prediction error for feature vectors.  \looseness=-1

	\subsection{Algorithm $\algrepbailb$}

	
	Now we present our algorithm $\algrepbailb$, whose pseudo-code is provided in Algorithm~\ref{alg:repbailb}. $\algrepbailb$ is a phased elimination algorithm, which first conducts the E-optimal design to optimally schedule samples for learning the feature extractor $\bs{B}$, and then performs the G-optimal design with low-dimensional representations to eliminate suboptimal arms.
	
	$\algrepbailb$ uses a \emph{rounding procedure} $\round$~\cite{allen2017near,fiez2019sequential}, which transforms a given continuous sample allocation (design) into a discrete sample sequence and maintains important properties (e.g., E-optimality and G-optimality) of the design. $\round(\{(\bs{q}_i,\bs{Q}_i)\}_{i=1}^{n'}, \bs{\lambda}, \zeta, N)$ takes $n'$ arm-matrix pairs $(\bs{q}_1,\bs{Q}_1),\dots,(\bs{q}_{n'},\bs{Q}_{n'}) \in \cX \times \R^{{d'} \times {d'}}$, a distribution $\bs{\lambda} \in \triangle_{\{\bs{q}_1,\dots,\bs{q}_{n'}\}}$, a rounding approximation parameter $\zeta>0$, and the number of samples $N$ such that $N \geq \frac{180d'}{\zeta^2}$ as inputs. It will return a sample sequence $\bs{s}_1, \dots, \bs{s}_N \in \cX$, which correspond to feature matrices $\bs{S}_1, \dots, \bs{S}_N \in \{\cQ_1,\dots,\cQ_{n'}\}$, and $\sum_{j=1}^N \bs{S}_j$ has similar properties as the covariance matrix of the inputted design $N \sum_{i=1}^{n'} \lambda(\bs{q}_i) \bs{Q}_i$ (see Appendix~\ref{apx:rounding_procedure} for more details).

	The procedure of $\algrepbailb$ is as follows. At the beginning, $\algrepbailb$ performs the E-optimal design with raw representations, to plan an optimal sample allocation $\bs{\lambda}^E$  for the purpose of recovering the feature extractor $\bs{B}$ (Line~\ref{line:bai_E_optimal_design}). 
	Then, $\algrepbailb$ calls $\round$ to convert the E-optimal sample allocation $\bs{\lambda}^E$ into a discrete sample batch $\bar{\bs{x}}_1,\dots,\bar{\bs{x}}_p$, which satisfies that
	$$
	\bigg\| \Big(\sum_{j=1}^p \bar{\bs{x}}_j \bar{\bs{x}}_j^\top \Big)^{-1} \bigg\| \leq (1+\zeta) \bigg\| \Big(p \sum_{i=1}^n \lambda^E(\bs{x}_i) \bs{x}_i \bs{x}_i^\top \Big)^{-1} \bigg\| .
	$$
	Next, $\algrepbailb$ enters multiple phases, and maintains a candidate arm set $\hat{\cX}_{t,m}$ for each task. 
	The specific value of $T_t$ in Line~\ref{line:bai_T_t} is presented in Eq.~\eqref{eq:value_T_t} of Appendix~\ref{apx:bai_feature_recover}.

	\begin{algorithm}[t]
		\caption{$\algrepbailb$ (Double Experimental Design)} \label{alg:repbailb}
		\begin{algorithmic}[1]
			\STATE {\bfseries Input:} $\cX$, $\delta$, rounding procedure $\round$, rounding approximation parameter $\zeta:=\frac{1}{10}$, and the size of sample batch $p:= \frac{180d}{\zeta^2}$.
			\STATE Let $\bs{\lambda}^{E}$ and $\rho^{E}$ be the optimal solution and the optimal value of the E-optimal design optimization:
			$$
			\min_{\bs{\lambda} \in \triangle_{\cX}} \Big\| \big( \sum_{i=1}^n \lambda(\bs{x}_i) \bs{x}_i \bs{x}_i^\top \big)^{-1} \Big\|
			$$\label{line:bai_E_optimal_design}\\
			\STATE $\bar{\bs{x}}_1,\dots,\bar{\bs{x}}_p \leftarrow \round(\{(\bs{x}_i, \bs{x}_i \bs{x}_i^\top)\}_{i=1}^{n}, \bs{\lambda}^{E}, \zeta, p)$ \label{line:bai_E_optimal_round}
			\STATE $\hat{\cX}_{1,m} \leftarrow \cX$ for any $m \in [M]$. $\delta_t \leftarrow \frac{\delta}{2 t^2}$ for any $t \geq 1$\; 
			\FOR{phase $t=1,2,\dots$}
			\STATE $T_t \leftarrow \lceil \frac{c_1 \sbr{1+\zeta}^3 (\rho^E)^2 k^4 L_{x}^4 L_{w}^4}{M} \max\{2^{2t},\ \frac{L_x^4}{\omega^2}\}\cdot$\\$ \polylog(\zeta,\rho^E,p,k,L_{x},L_{w},\frac{1}{\delta_t}, \frac{1}{\omega}) \rceil$, where $c_1$ is an absolute constant \label{line:bai_T_t}
			\STATE $\hat{\bs{B}}_t \leftarrow \algfeatrecover(T_t, \{\bar{\bs{x}}_i\}_{i \in [p]})$\;
			\STATE $\{\hat{\cX}_{t+1,m}\}_{m \in [M]} \leftarrow$\\$ \algelimlowrep (t, \cX, \{\hat{\cX}_{t,m}\}_{m \in [M]}, \delta_t, \round, \zeta, \hat{\bs{B}}_t)$
			\IF{$|\hat{\cX}_{t+1,m}|=1$, $\forall m \in [M]$}
			\STATE {\bfseries return}  $\hat{\cX}_{t+1,m}$ for all tasks $m \in [M]$\;
			\ENDIF
			\ENDFOR
		\end{algorithmic}
	\end{algorithm}
	
	\begin{algorithm}[t] 
		\caption{$\algfeatrecover(T,  \{\bar{\bs{x}}_i\}_{i \in [p]})$} \label{alg:feat_recover}
		\begin{algorithmic}[1]
			\FOR{task $m \in [M]$} \label{line:bai_stage2_sample_start}
			\FOR{round $j \in [T]$ }
			\FOR{arm $i \in [p]$}
			\STATE Sample $\bar{\bs{x}}_i$, and observe random reward $\alpha_{m,j,i}$\; \label{line:bai_stage2_sample} 
			\ENDFOR
			\STATE $\tilde{\bs{\theta}}_{m,j} \leftarrow (\sum_{i=1}^{p} \bar{\bs{x}}_i \bar{\bs{x}}_i^\top)^{-1} \sum_{i=1}^{p} \bar{\bs{x}}_i \alpha_{m,j,i}$
			\ENDFOR
			\ENDFOR \label{line:bai_stage2_sample_end}
			\STATE $\bs{Z} \leftarrow \frac{1}{M T} \sum_{m=1}^{M} \sum_{j=1}^{T} \tilde{\bs{\theta}}_{m,j} (\tilde{\bs{\theta}}_{m,j})^\top - (\sum_{i=1}^{p} \bar{\bs{x}}_i \bar{\bs{x}}_i^\top)^{-1} $ \label{line:bai_Z_t}\;
			\STATE Perform SVD decomposition on $\bs{Z}$, and let $\hat{\bs{B}}$ be the top-$k$ left singular vectors of $\bs{Z}$ \label{line:svd}\; 
			\STATE {\bfseries return} $\hat{\bs{B}}$\;
		\end{algorithmic}
	\end{algorithm}

	\begin{algorithm}[t!] 
		\caption{$\algelimlowrep(t, \cX\!,\! \{\hat{\cX}_{m}\}_{\! m \in [M]}, \delta'\!, \round,\! \zeta,\! \hat{\bs{B}})$} \label{alg:elim_low_rep}
		\begin{algorithmic}[1]
			\FOR{task $m \in [M]$}
			\STATE Let $\bs{\lambda}^{G}_{m}$ and $\rho^{G}_{m}$ be the optimal solution and the optimal value of the G-optimal design optimization:
			$$
			\argmin_{\bs{\lambda} \in \triangle_{\cX}} \max_{\bs{x},\bs{x}' \in \hat{\cX}_{m}} \nbr{ \hat{\bs{B}}^\top (\bs{x}-\bs{x}') }^2_{\bs{A}(\bs{\lambda},\hat{\bs{B}})^{-1}}
			$$ \label{line:bai_G_optimal_design}\\
			\STATE $N_{m} \leftarrow \lceil \max \{ 32 (1+\zeta) 2^{2t}  \rho^{G}_{m} \log (\frac{4n^2 M}{\delta'}),$ $\frac{180 k}{\zeta^2} \} \rceil$\;
			\STATE $\bs{z}_{m,1},\dots,\bs{z}_{m,N_{m}} \leftarrow$\\$ \round(\{(\bs{x}_i, \hat{\bs{B}}^\top \bs{x}_i \bs{x}_i^\top \hat{\bs{B}})\}_{i=1}^{n}, \bs{\lambda}^{G}_{m}, \zeta, N_{m})$ \label{line:bai_stage3_round}\;
			\STATE Sample the arms $\bs{z}_{m,1},\dots,\bs{z}_{m,N_{m}} \in \cX$, and observe random rewards $r_{m,1}, \dots, r_{m,N_{m}}$\; \label{line:bai_stage3_sample}
			\STATE Let $\tilde{\bs{z}}_{m,j} := \hat{\bs{B}}^\top \bs{z}_{m,j}$ for any $j \in [N_{m}]$
			\STATE $\hat{\bs{w}}_{m} \leftarrow ( \sum_{j=1}^{N_{m}}  \tilde{\bs{z}}_{m,j} \tilde{\bs{z}}_{m,j}^\top )^{-1}  \sum_{j=1}^{N_{m}}  \tilde{\bs{z}}_{m,j} r_{m,j}$ \label{line:bai_est_theta}\;
			\STATE $\hat{\bs{\theta}}_{m} \leftarrow \hat{\bs{B}} \hat{\bs{w}}_{m}$ \label{line:bai_est_w}\;
			\STATE $\hat{\cX}'_{m} \leftarrow \hat{\cX}_{m} \setminus \{ \bs{x} \in \hat{\cX}_{m} \ | \ \exists \bs{x}' \in \hat{\cX}_{m}: (\bs{x}'-\bs{x})^\top \hat{\bs{\theta}}_{m} > 2^{-t} \}$ \label{line:bai_elimination}\;
			\ENDFOR
			\STATE {\bfseries return}  $\{\hat{\cX}'_{m}\}_{m \in [M]}$
		\end{algorithmic}
	\end{algorithm}
	
	In each phase $t$, 
	$\algrepbailb$ first calls subroutine $\algfeatrecover$ to recover the feature extractor $\bs{B}$. In $\algfeatrecover$ (Algorithm~\ref{alg:feat_recover}), we repeatedly sample $\bar{\bs{x}}_1,\dots,\bar{\bs{x}}_p$ in all tasks, and construct an estimator $\bs{Z}$ for  $\frac{1}{M} \sum_{i=1}^{M} \bs{\theta}_m \bs{\theta}_m^\top$, which contains the information of underlying reward parameters (Line~\ref{line:bai_Z_t}). Then, we perform SVD  on $\bs{Z}$ and obtain the estimated feature extractor $\hat{\bs{B}}$ (Line~\ref{line:svd}). 
	
	Then, $\algrepbailb$ calls subroutine $\algelimlowrep$ to eliminate suboptimal arms using low-dimensional representations.
	In $\algelimlowrep$ (Algorithm~\ref{alg:elim_low_rep}), we conduct the G-optimal design with the reduced-dimensional representations $\hat{\bs{B}}^\top \bs{x}$, and obtain sample allocation $\bs{\lambda}^{G}_{m}$ for each task (Line~\ref{line:bai_G_optimal_design}). We further use $\round$ to transform $\bs{\lambda}^{G}_{m}$ into a sample sequence $\bs{z}_{m,1},\dots,\bs{z}_{m,N_{m}}$, which satisfies that
	\begin{align*}
		& \max_{\bs{x}, \bs{x}' \in \hat{\cX}_{m}} \nbr{\bs{x}-\bs{x}'}^2_{\sbr{\sum_{j=1}^{N_{m}} \hat{\bs{B}}^\top \bs{z}_{m,j} \bs{z}_{m,j}^\top \hat{\bs{B}}}^{-1}} 
		\\
		\leq & (1+\zeta) \max_{\bs{x}, \bs{x}' \in \hat{\cX}_{m}} \nbr{\bs{x}-\bs{x}'}^2_{\sbr{N_{m} \sum_{i=1}^{n} \lambda^{G}_{m}(\bs{x}_i) \hat{\bs{B}}^\top \bs{x}_i \bs{x}_i^\top \hat{\bs{B}}}^{-1}} .
	\end{align*}
	After sampling this sequence, we build estimators $\hat{\bs{w}}_{t,m}$ and $\hat{\bs{\theta}}_{t,m}$ for the underlying prediction parameter $\bs{w}_m$ and reward parameter $\bs{\theta}_m$, respectively (Lines~\ref{line:bai_est_theta}-\ref{line:bai_est_w}). Then, we discard the arms that show large gaps to the estimated optimal arm for each task (Line~\ref{line:bai_elimination}).

	\subsection{Theoretical Performance of $\algrepbailb$}
	
	In this subsection, we provide sample complexity guarantees for $\algrepbailb$.
	To formally present our sample complexity, 
	we first revisit existing results for conventional single-task best arm identification in linear bandits (BAI-LB). 
	
	For a single-task BAI-LB instance with arm set $\cX \in \R^d$ and underlying reward parameter $\bs{\theta} \in \R^{d}$, the instance-dependent hardness is defined as~\cite{fiez2019sequential}
	\begin{align*}
		\rho^{S}(\cX, \bs{\theta}) \! := \!\! \min_{\bs{\lambda} \in \triangle_{\cX}} \max_{\bs{x} \in \cX \setminus \{\bs{x}^{*}\}} \!\!\!\! \frac{\| \bs{x}^{*} - \bs{x} \|^2_{\sbr{\sum_{i=1}^{n} \lambda(\bs{x}_i) \bs{x}_i \bs{x}_i^\top }^{-1}} }{ ((\bs{x}^{*} - \bs{x})^\top \bs{\theta})^2 } ,
	\end{align*}
	and the best known sample complexity result is $\tilde{O}(\rho^{S}(\cX, \bs{\theta}) \log(\frac{1}{\delta}))=\tilde{O}(\frac{d}{(\Delta^{S}_{\min})^2} \log(\frac{1}{\delta}))$~\cite{fiez2019sequential}.
	Here $\bs{x}^{*}:=\argmax_{\bs{x} \in \cX} \bs{x}^\top \bs{\theta}$ denotes the best arm, and $\Delta^{S}_{\min}:=\min_{\bs{x} \in \cX \setminus \{\bs{x}^{*}\}}(\bs{x}^{*} - \bs{x})^\top \bs{\theta}$ refers to the minimum reward gap.
	
	It can be seen that a naive algorithm for RepBAI-LB is to run an existing single-task BAI-LB algorithm~\cite{fiez2019sequential,katz2020empirical} to solve $M$ tasks independently. Then, the sample complexity of such naive algorithm is
	\begin{align}
		\!\!\!\! \tilde{O} \! \sbr{ \sum_{m=1}^{M} \rho^{S}(\cX, \bs{\theta}_m) \log\sbr{ \frac{1}{\delta} } \!} \!\!=\! \tilde{O}\sbr{\! \frac{Md}{\Delta_{\min}^2} \log\sbr{ \frac{1}{\delta} } \!} \!, \label{eq:bai_naive_sample_complexity}
	\end{align}
	where $\Delta_{\min}:=\min_{m \in [M], \bs{x} \in \cX \setminus \{\bs{x}^{*}_m\}}(\bs{x}^{*}_m - \bs{x})^\top \bs{\theta}_m$ denotes the minimum reward gap among all tasks. In the following, we take Eq.~\eqref{eq:bai_naive_sample_complexity} as the baseline to demonstrate the power of representation learning.
	
	Now we state the sample complexity for $\algrepbailb$.
	\begin{theorem} \label{thm:bai_ub}
		With probability at least $1-\delta$, algorithm $\algrepbailb$   returns the best arms $\bs{x}^{*}_{m}$ for all tasks $m \in [M]$, and the number of samples used is bounded by
		\begin{align}
			\tilde{O}  \bigg( & \sum_{m=1}^{M} \! \min_{\bs{\lambda} \in \triangle_{\cX}} \! \max_{\bs{x} \in \cX \setminus \{\bs{x}^{*}_{m}\}} \!\!\!\!\! \frac{\| \bs{B}^\top (\bs{x}^{*}_{m} - \bs{x}) \|^2_{\bs{A}(\bs{\lambda},\bs{B})^{-1}} }{ ((\bs{x}^{*}_{m} - \bs{x})^\top \bs{\theta}_m)^2 } \! \log\Big(\frac{1}{\delta}\Big)
			\label{eq:bai_ub}
			\nonumber\\& 
			+  (\rho^E)^2 d k^4 L_{x}^2 L_{w}^2 D \log^4 \Big(\frac{1}{\delta}\Big) \bigg)
			\\
			= & \ \tilde{O} \bigg(  \frac{M k}{\Delta_{\min}^2}  \log\Big(\frac{1}{\delta}\Big) + (\rho^E)^2 d k^4  L_{x}^2 L_{w}^2  D \log^4 \Big(\frac{1}{\delta}\Big)  \bigg) , \nonumber
		\end{align}
		where 
		$D:=\max\{ \frac{1}{\Delta_{\min}^2} ,\ \frac{L_{x}^4}{\omega^2} \}$.
	\end{theorem}
	
	\textbf{Remark 1.} 
	\cameraReady{
		In Theorem~\ref{thm:bai_ub}, the  factors that have implicit dimensional dependency include $\min_{\bs{\lambda} \in \triangle_{\cX}} \max_{\bs{x} \in \cX \setminus \{\bs{x}^{*}_{m}\}} \frac{\| \bs{B}^\top (\bs{x}^{*}_{m} - \bs{x}) \|^2_{\bs{A}(\bs{\lambda},\bs{B})^{-1}} }{ ((\bs{x}^{*}_{m} - \bs{x})^\top \bs{\theta}_m)^2 }$, $\omega$ and $\rho^E$, which scale as $k$, $\frac{1}{k}$ and $d$, respectively.}
	
	In our sample complexity bound (Eq.~\eqref{eq:bai_ub}), the first term, $\sum_{m=1}^{M}  \min_{\bs{\lambda} \in \triangle_{\cX}} \max_{\bs{x} \in \cX \setminus \{\bs{x}^{*}_{m}\}} \frac{\| \bs{B}^\top (\bs{x}^{*}_{m} - \bs{x}) \|^2_{\bs{A}(\bs{\lambda},\bs{B})^{-1}} }{ ((\bs{x}^{*}_{m} - \bs{x})^\top \bs{\theta}_m)^2 }=O(\frac{M k}{\Delta_{\min}^2})$, represents the hardness of $M$ $k$-dimensional linear bandit instances with arm set $\{\bs{B}^\top \bs{x}: \bs{x} \in \cX\}$ and underlying reward parameters $\bs{w}_1,\dots,\bs{w}_{M}$. This term only depends on the reduced dimension $k$, instead of $d$. 
	In other words, it is an essential price that is needed for solving $M$ low-dimensional tasks, even if one knows the feature extractor $\bs{B}$. 
	The second term $(\rho^E)^2 d k^4 L_{x}^2 L_{w}^2 D$, which depends on the raw dimension $d$, is a cost paid for learning the feature extractor. Note that since this term does not contain $M$, the cost for learning the underlying features is paid only once, rather than for all tasks. 
	
	\cameraReady{When $M \gg d \gg k$, the first term dominates the bound, which only depends on the low dimension $k$.}
	This indicates that algorithm $\algrepbailb$ effectively learns the low-dimensional representation, and exploits the intrinsic problem structure to
	reduce the sample complexity from $\tilde{O}(\frac{M d}{\Delta_{\min}^2} \log (\frac{1}{\delta}))$ (i.e., learning each task independently)
	to only $\tilde{O}(\frac{M k}{\Delta_{\min}^2} \log (\frac{1}{\delta}))$. Our result corroborates the benefits of representation learning for multi-task pure exploration. 
	
	\textbf{Technical Novelty.} 
	We highlight the novelty in the analysis of Theorem~\ref{thm:bai_ub} as follows. 
	(i) 
	Prior low-rank bandit works~\cite{jun2019bilinear,lu2021low} use \emph{arbitrary} sample  distributions to recover the low-dimensional subspace, and their results depend on the eigenvalue of an arbitrary sample distribution $\|\boldsymbol{X}^{-1}\|$, where $\boldsymbol{X}=[\boldsymbol{x}^{(1)},\dots,\boldsymbol{x}^{(d_1)}]$ is a collection of arbitrary $d_1$ arms from the arm set.
	By contrast, we utilize the \emph{E-optimality} of the sample batch $\bar{\bs{x}}_1,\dots,\bar{\bs{x}}_p$ to obtain an optimized dependency $\rho^E \approx \min_{\boldsymbol{x}^{(1)},\dots,\boldsymbol{x}^{(d_1)} \in \mathcal{X}} \|\boldsymbol{X}^{-1}\|$, which is the best one can achieve at the subspace recovery stage.
	(ii) 
	If one naively applies existing single-task BAI-LB analysis~\cite{fiez2019sequential,katz2020empirical} in the estimated subspace $\hat{\boldsymbol{B}}_t$, one can only obtain a sample complexity $\|\hat{\boldsymbol{B}}_t^\top (\boldsymbol{x}-\boldsymbol{x}')\|^2_{(\sum_{i=1}^{n}\lambda_m^{*}(\boldsymbol{x}_i) \hat{\boldsymbol{B}}_t^\top \boldsymbol{x}_i \boldsymbol{x}_i^\top \hat{\boldsymbol{B}}_t )^{-1}}$ dependent on $\hat{\boldsymbol{B}}_t$, but this is not a valid upper bound.
	To tackle this challenge, we connect the low-dimensional sample complexity under the estimated subspace $\|\hat{\boldsymbol{B}}_t^\top (\boldsymbol{x}-\boldsymbol{x}')\|^2_{(\sum_{i=1}^{n}\lambda_m^{*}(\boldsymbol{x}_i) \hat{\boldsymbol{B}}_t^\top \boldsymbol{x}_i \boldsymbol{x}_i^\top \hat{\boldsymbol{B}}_t )^{-1}}$ with that under the true subspace $\|\boldsymbol{B}^\top (\boldsymbol{x}-\boldsymbol{x}')\|^2_{(\sum_{i=1}^{n}\lambda_m^{*}(\boldsymbol{x}_i) \boldsymbol{B}^\top \boldsymbol{x}_i \boldsymbol{x}_i^\top \boldsymbol{B} )^{-1}}$, and drive a tight sample complexity.
	

	\cameraReady{
	\textbf{Lower Bound Conjecture.}
	We conjecture that the lower bound for RepBAI-LB is $\Omega( \sum_{m=1}^{M} \rho^{S}(\cX, \bs{\theta}_m)  \log(\frac{1}{\delta}) )$. We describe the preliminary idea below.
	
	First, the lower bound for single-task BAI-LB with arm set $\cX$ and underlying reward parameter $\bs{\theta}_m$ is $\Omega(\rho^{S}(\cX, \bs{\theta}_m)  \log (\frac{1}{\delta}) )$~\cite{fiez2019sequential}. 
	If the global feature extractor $\boldsymbol{B}$ is known, then the RepBAI-LB problem will reduce to $M$ $k$-dimensional BAI-LB instances with arm set $\{\boldsymbol{B}^\top \boldsymbol{x}: \bs{x} \in \cX \}$ and underlying reward parameters $\boldsymbol{w}_1, \dots, \boldsymbol{w}_M$. Therefore, we conjecture that the lower bound for RepBAI-LB is $\Omega( \sum_{m=1}^{M} \rho^{S}(\cX, \bs{\theta}_m) \log(\frac{1}{\delta}) )$, which is the cost of solving $M$ $k$-dimensional BAI-LB instances. However, it is challenging to rigorously analyze the independence of these $M$ $k$-dimensional instances and drive the summation in our conjectured lower bound. We leave the formal lower bound proof for future work.
	
	When $M \gg d \gg k$, Theorem~\ref{thm:bai_ub} matches our conjectured lower bound, which implies that algorithm $\algrepbailb$ performs as well as an oracle that knows the low-rank representation $\boldsymbol{B}$ in advance.
	}

	\section{Representation Learning for Best Policy Identification in Contextual Linear Bandits}

		\begin{algorithm}[t]
		\caption{$\algrepbpiclb$ (Contextual Double Experimental Design)} \label{alg:repbpiclb}
		\begin{algorithmic}[1]
			\STATE {\bfseries Input:} $\delta$, $\varepsilon$, $\bs{\phi}(\cdot,\cdot)$, regularization parameter $\gamma \geq 1$, rounding procedure $\round$, rounding approximation parameter $\zeta:=\frac{1}{10}$, and the size of sample batch $p:= \lceil \frac{c_2 (1+\zeta)^2 L_{\phi}^4}{\nu^2} \polylog(\zeta,M,d,k,L_{\phi},L_{w},\gamma,\frac{1}{\nu},\frac{1}{\delta},\frac{1}{\varepsilon}) \rceil$, where $c_2$ is an absolute constant. \label{line:bpi_value_p}
			\STATE $T_0 \leftarrow \lceil \frac{32^2 (1+\zeta)^2 L_{\phi}^4}{\nu^2} \log^2 (\frac{20d |\cA|}{\delta}) \rceil$. $\hat{\cD} \leftarrow \emptyset$
			\FOR{$\tau \in [T_0]$}  \label{line:bpi_estimate_context_dis_start}
			\STATE \hspace*{-0.4em} Observe context $s_{\tau}$, and randomly sample an action\;
			\STATE $\hat{\cD} \leftarrow \hat{\cD} \cup \{s_{\tau}\}$\;
			\ENDFOR \label{line:bpi_estimate_context_dis_end}
			\STATE Let $\bs{\lambda}^{E}_{\hat{\cD}}$ and $\rho^{E}_{\hat{\cD}}$ be the optimal solution and the optimal value of the E-optimal design optimization: 
			$$
			\min_{\bs{\lambda} \in \triangle_{\cA}} \Big\| \big( \sum_{a \in \cA} \lambda(a) \ex_{s \sim \hat{\cD}}\mbr{\bs{\phi}(s,a) \bs{\phi}(s,a)^\top} \big)^{-1} \Big\|
			$$ \label{line:bpi_E_optimal_design}\\
			\STATE $\{\bar{a}_i\}_{i \in [p]} \!\! \leftarrow \!\! \round(\{(a, \ex_{\! s \sim \hat{\cD} \!\!}\mbr{\bs{\phi}(s,a) \bs{\phi}(s,a)^\top})\}_{a \in \cA},$\\$ \bs{\lambda}^{E}_{\hat{\cD}}, \zeta, p)$ \label{line:bpi_E_design_round}\; 
			\STATE $T \leftarrow \lceil \frac{c_3 (1+\zeta)^2 k^4 L_{\phi}^4 L_{w}^4 }{ M \nu^2 \varepsilon^2 } \polylog(\zeta,d,k,L_{\phi},L_{w},\gamma,\frac{1}{\nu},$\\$\frac{1}{\delta},\frac{1}{\varepsilon}) \rceil$, where $c_3$ is an absolute constant \label{line:bpi_value_T}
			\STATE $\hat{\bs{B}} \leftarrow \algconfeatrecover(T, \{\bar{a}_i\}_{i \in [p]})$
			\STATE $N \leftarrow \lceil \frac{ (k^2 + \gamma k L_{w}^2) }{\varepsilon^2} \log^4 ({\frac{  \gamma k  L_{w} }{\varepsilon \delta}} ) \rceil$\;
			\STATE $\{\hat{\bs{\theta}}_{m,N}\}_{m \in [M]} \leftarrow \algestlowrep(N, \gamma, \hat{\bs{B}})$
			\STATE {\bfseries return} $\hat{\pi}_m(\cdot):=\argmax_{a \in \cA} \bs{\phi}(\cdot,a)^\top \hat{\bs{\theta}}_{m,N}$ for all tasks $m \in [M]$ \label{line:bpi_return}
		\end{algorithmic}
	\end{algorithm}

	\begin{algorithm}[t]
		\caption{$\algconfeatrecover(T, \{\bar{a}_i\}_{i \in [p]})$} \label{alg:con_feat_recover}
		\begin{algorithmic}[1]
			\FOR{task $m \in [M]$} \label{line:bpi_stage2_sample_start}
			\FOR{round $j \in [T]$}
			\FOR{arm $i \in [p]$}
			\STATE Observe context $s_{m,j,i}^{(1)}$, sample action $\bar{a}_i$ in task $m$, and observe reward $\alpha^{(1)}_{m,j,i}$\; \label{line:bpi_stage2_sample1}
			\STATE Observe context $s_{m,j,i}^{(2)}$, sample action $\bar{a}_i$ in task $m$, and observe reward $\alpha^{(2)}_{m,j,i}$\; \label{line:bpi_stage2_sample2}
			\ENDFOR
			\STATE Let $\bs{\phi}^{(\ell)}_{m,j,i}\!:=\!\bs{\phi}(s_{m,j,i}^{(\ell)},\bar{a}_i)$, $\forall i \!\in\! [p]$, $\forall \ell \!\in\! \{1,2\}$\;
			\STATE $\tilde{\bs{\theta}}^{(\ell)}_{m,j} \! \leftarrow \!\! ( \sum_{i=1}^{p} \! \bs{\phi}^{(\ell)}_{m,j,i} {\bs{\phi}^{(\ell)}_{m,j,i}}^{\!\!\!\!\top} \!)^{-1} \! \sum_{i=1}^{p} \! \bs{\phi}^{(\ell)}_{m,j,i} \alpha^{(\ell)}_{m,j,i}$, $\forall \ell \in \{1,2\}$\;
			\ENDFOR
			\ENDFOR  \label{line:bpi_stage2_sample_end}
			\STATE $\bs{Z} \leftarrow \frac{1}{M T} \sum_{m=1}^{M} \sum_{j=1}^{T} \tilde{\bs{\theta}}^{(1)}_{m,j} (\tilde{\bs{\theta}}^{(2)}_{m,j})^\top$ \label{line:bpi_Z}\;
			\STATE Perform SVD decomposition on $\bs{Z}$, and let $\hat{\bs{B}}$ be the top-$k$ left singular vectors \label{line:bpi_svd}\;
			\STATE {\bfseries return} $\hat{\bs{B}}$\; 
		\end{algorithmic}
	\end{algorithm}

	\begin{algorithm}[t]
		\caption{$\algestlowrep(N, \gamma, \hat{\bs{B}})$} \label{alg:est_low_rep}
		\begin{algorithmic}[1]
			\STATE $\bs{\Sigma}_{m,0} \leftarrow \gamma I$ for any $m \in [M]$\;
			\FOR{task $m \in [M]$}
			\FOR{timestep $t \in [N]$}
			\STATE Observe context $s_{m,t}$\; \label{line:bpi_stage3_context}
			\STATE $a_{m,t} \leftarrow \argmax_{a \in \cA} \| \hat{\bs{B}}^\top \bs{\phi}(s_{m,t},a) \|_{\bs{\Sigma}_{m,t-1}^{-1}}$\;
			\STATE Sample action $a_{m,t}$, and observe reward $r_{m,t}$\; \label{line:bpi_stage3_sample}
			\STATE $\bs{\Sigma}_{m,t} \leftarrow \bs{\Sigma}_{m,t-1} +$\\$ \hat{\bs{B}}^\top \bs{\phi}(s_{m,t},a_{m,t}) \bs{\phi}(s_{m,t},a_{m,t})^\top \hat{\bs{B}}$\;
			\STATE $\hat{\bs{w}}_{m,t} \leftarrow \bs{\Sigma}_{m,t}^{-1} \sum_{\tau=1}^{t} \hat{\bs{B}}^\top \bs{\phi}(s_{m,\tau},a_{m,\tau}) r_{m,\tau}$ \label{line:bpi_est_w}\;
			\STATE $\hat{\bs{\theta}}_{m,t} \leftarrow \hat{\bs{B}} \hat{\bs{w}}_{m,t}$ \label{line:bpi_est_theta}\;
			\ENDFOR
			\ENDFOR
			\STATE {\bfseries return} $\{\hat{\bs{\theta}}_{m,N}\}_{m \in [M]}$\;
		\end{algorithmic}
	\end{algorithm}

	In this section, we turn to contextual linear bandits. 
	Different from prior contextual linear bandit works, e.g.,~\cite{huang2015efficient,li2022instance}, here we do not assume any knowledge of context distribution. 
	As a result, our $\probbpi$ problem faces several unique challenges: (i) how to plan an efficient sample allocation for recovering the feature extractor in advance under an \emph{unknown} context distribution, and (ii) how to construct an estimator for the feature extractor with a partially observed context space.
	
	
	We propose algorithm $\algrepbpiclb$, which first (i) efficiently estimates the context distribution and conducts experimental designs under the estimated context distribution, and then (ii) builds a delicate estimator for the feature extractor using instantaneous contexts. 
	Moreover, we also establish a sample complexity guarantee for $\algrepbpiclb$, which mainly depends on the low dimension of the common representation among tasks.

	\subsection{Algorithm $\algrepbpiclb$}

	Algorithm~\ref{alg:repbpiclb} presents the pseudo-code of $\algrepbpiclb$. 
	At the beginning, $\algrepbpiclb$ uses $T_0$ samples to estimate the context distribution $\cD$ (Lines~\ref{line:bpi_estimate_context_dis_start}-\ref{line:bpi_estimate_context_dis_end}). Then, it performs the E-optimal design under the estimated context distribution $\hat{\cD}$, and obtains an efficient sample allocation $\bs{\lambda}^{E}_{\hat{\cD}}$ for the purpose of recovering the feature extractor $\bs{B}$ (Line~\ref{line:bpi_E_optimal_design}). Further, $\algrepbpiclb$ calls the rounding procedure $\round$ to transform $\bs{\lambda}^{E}_{\hat{\cD}}$ into a sample batch $\bar{a}_1,\dots,\bar{a}_p$, such that
	\begin{align*}
		& \Big\| \big(\sum_{j=1}^p \ex_{s \sim \hat{\cD}}\mbr{\bs{\phi}(s,\bar{a}_j) \bs{\phi}(s,\bar{a}_j)^\top} \big)^{-1} \Big\| 
		\\ 
		\leq & (1+\zeta) \Big\| \big(p \sum_{a \in \cA} \lambda^{E}_{\hat{\cD}}(a) \ex_{s \sim \hat{\cD}}\mbr{\bs{\phi}(s,a) \bs{\phi}(s,a)^\top} \big)^{-1} \Big\| .
	\end{align*}
	The specific values of $p$ and $T$ in Lines~\ref{line:bpi_value_p},~\ref{line:bpi_value_T} are provided in Eq.~\eqref{eq:value_p} of Appendix~\ref{apx:bpi_sample_batch_planning} and Eq.~\eqref{eq:value_T} of Appendix~\ref{apx:bpi_feat_recover}, respectively.\looseness=-1
	
	Next, $\algrepbpiclb$ runs subroutine $\algconfeatrecover$ to estimate the feature extractor $\bs{B}$ using the sample batch $\bar{a}_1,\dots,\bar{a}_p$.
	In $\algconfeatrecover$ (Algorithm~\ref{alg:con_feat_recover}), we repeatedly sample $\bar{a}_1,\dots,\bar{a}_p$ in all tasks with random contexts. 
	In Lines~\ref{line:bpi_stage2_sample1}-\ref{line:bpi_stage2_sample2}, we sample this batch twice, and the superscripts $(1)$ and $(2)$ denotes the first and second samples, respectively.
	After sampling, we carefully establish an estimator $\bs{Z}$ for the reward parameter related matrix $\frac{1}{M} \sum_{m=1}^{M} \bs{\theta}_m \bs{\theta}_m^\top$, using instantaneous context-action features $\bs{\phi}(s_{m,j,i}^{(\ell)},\bar{a}_{i})^\top$.
	We then perform SVD decomposition on $\bs{Z}$ to obtain the estimated feature extractor $\hat{\bs{B}}$ (Lines~\ref{line:bpi_Z}-\ref{line:bpi_svd}). 
	
	Then, $\algrepbpiclb$ calls subroutine $\algestlowrep$, which adapts existing reward-free-exploration algorithm in \cite{zanette2021design} with low-rank representations to estimate $\bs{\theta}_m$.
	In $\algestlowrep$ (Algorithm~\ref{alg:est_low_rep}), we employ the estimated representation $\hat{\bs{B}}^\top \bs{\phi}(s,a)$  to sample the actions with the maximum uncertainty under the observed contexts.
	After that, we construct estimators $\hat{\bs{w}}_{m,t}$ and $\hat{\bs{\theta}}_{m,t}$ for the prediction parameter $\hat{\bs{w}}_m$ and reward parameter $\hat{\bs{\theta}}_m$ (Lines~\ref{line:bpi_est_w}-\ref{line:bpi_est_theta}). 
	At last, $\algrepbpiclb$ returns the greedy policy with respect to the estimated reward parameter $\hat{\bs{\theta}}_{m,N}$ for each task.

	\subsection{Theoretical Performance of $\algrepbpiclb$}
	
	Next, we establish sample complexity guarantees for algorithm $\algrepbpiclb$. In order to illustrate the advantages of representation learning, we first review existing results for traditional single-task best policy identification in contextual linear bandits (BPI-CLB).
	For a single  BPI-CLB instance with context-action features $\bs{\phi}(s,a) \in \R^d$ and reward parameter $\bs{\theta} \in \R^{d}$, the  best known sample complexity is $\tilde{O}(\frac{d^2}{\varepsilon^2}\log(\frac{1}{\delta}))$~\cite{zanette2021design,li2022instance}.
	
	Apparently, if one naively solves the $\probbpi$ problem by running single-task BPI-CLB algorithms to tackle $M$ tasks independently, one will have a  sample complexity 
	\begin{align*}
		\tilde{O}\bigg( \frac{Md^2}{\varepsilon^2} \log \Big(\frac{1}{\delta}\Big) \bigg) ,
	\end{align*}
	which heavily depends on the raw dimension $d$ of context-action features.
	The goal of representation learning is to leverage the common representation among tasks to alleviate the dependency of dimension and save samples.
	
	Now we present the sample complexity for $\algrepbpiclb$. \looseness=-1
	
	\begin{theorem} \label{thm:bpi_ub}
		With probability at least $1-\delta$,  $\algrepbpiclb$  returns an $\varepsilon$-optimal policy $\hat{\pi}_m$ such that
		$
		\ex_{s \sim \cD}  [ \max_{a \in \cA} (\bs{\phi}(s,a) - \bs{\phi}(s,\hat{\pi}_m(s) )^\top \bs{\theta}_m ] \leq \varepsilon
		$ 
		for each task $m \in [M]$, and the number of samples used is 
		\begin{align*}
			\tilde{O} \sbr{ \frac{ M \sbr{k^2 + \gamma k L_{w}^2 } }{\varepsilon^2} + \frac{ k^4 L_{\phi}^8 L_{w}^4 }{ \nu^4 \varepsilon^2 } } .
		\end{align*}
	\end{theorem}
	
	\textbf{Remark 2.} \cameraReady{In this result, only factor $\nu$ has implicit dimensional dependency, which scales as $\frac{1}{d}$.} The first term $\frac{ M (k^2 + \gamma k L_{w}^2) }{\varepsilon^2}$ is a cost of identifying optimal policies for $M$ tasks with $k$-dimensional features $\bs{B}^\top \bs{\phi}(s,a)$. The second term $\frac{ k^4 L_{\phi}^8 L_{w}^4 }{ \nu^4 \varepsilon^2 }$ is a price paid for learning global feature extractor $\bs{B}$ and does not depend on $M$. This indicates that we only need to pay this price once, and then enjoy the benefits of dimension reduction for all $M$ tasks.
	
	\cameraReady{When $M \gg \frac{1}{\nu} \gg k$, this result becomes
	$\tilde{O}(\frac{M k^2}{\varepsilon^2})$ and only depends on the low dimension $k$, which implies that $\algrepbpiclb$ performs as well as an oracle that knows the underlying low-rank subspace $\bs{B}$.} This sample complexity significantly outperforms the baseline result $\tilde{O}(\frac{M d^2}{\varepsilon^2})$ (i.e., solving $M$ tasks independently), and demonstrates the power of representation learning. 
	
	\cameraReady{
	\textbf{Analytical Novelty.}
	Below we elaborate the novelty in the proof of Theorem~\ref{thm:bpi_ub}.  
	(i) We carefully bound the deviation between the context-action features under the estimated context distribution $\mathbb{E}_{s \sim \hat{\mathcal{D}}} [\boldsymbol{\phi}(s,\bar{a}_i) \boldsymbol{\phi}(s,\bar{a}_i)^\top]$ and those under the true context distribution $\mathbb{E}_{s \sim \mathcal{D}} [\boldsymbol{\phi}(s,\bar{a}_i) \boldsymbol{\phi}(s,\bar{a}_i)^\top]$. We further bound the distance between $\mathbb{E}_{s \sim \mathcal{D}} [\boldsymbol{\phi}(s,\bar{a}_i) \boldsymbol{\phi}(s,\bar{a}_i)^\top]$ and the context-action features under actual instantaneous contexts $\boldsymbol{\phi}(s_{m,j,i}^{(\ell)},\bar{a}_i) \boldsymbol{\phi}(s_{m,j,i}^{(\ell)},\bar{a}_i)^\top$.
	(ii) We leverage the E-optimality of the sample batch $\bar{a}_1,\dots,\bar{a}_p$
	to bound $\|(\sum_{i=1}^{p} \bs{\phi}^{(\ell)}_{m,j,i} {\bs{\phi}^{(\ell)}_{m,j,i}}^\top)^{-1}\|$. Then, we establish a concentration inequality for $\|\bs{Z}-\frac{1}{M} \sum_{m=1}^{M} \bs{\theta}_m \bs{\theta}_m^\top\|$ using the bounded $\|(\sum_{i=1}^{p} \bs{\phi}^{(\ell)}_{m,j,i} {\bs{\phi}^{(\ell)}_{m,j,i}}^\top)^{-1}\|$ and matrix Bernstern inequality with truncated noises. 
	(iii) Furthermore, we decompose the prediction error $\bs{\phi}(s,a)^\top (\hat{\bs{\theta}}_{m,t} - \bs{\theta}_m)$ into three components, including the sample variance and bias of $\hat{\bs{w}}_{m,t}$, and the estimation error of $\hat{\bs{B}}$. This prediction error is bounded via self-normalized concentration inequalities with the reduced dimension $k$.
	}

	\section{Experiments}

	In this section, we present experiments to evaluate the empirical performance of our algorithms.
	
	In our experiments, we set $\delta=0.005$, $d=5$, $k=2$ and $M \in [50,230]$, where $k$ divides $M$. In $\probbai$, $\cX$ is the canonical basis of $\R^{d}$. 
	In $\probbpi$, we set $\varepsilon=0.1$, $|\cS|=5$ and $|\cA|=5$.
	$\cD$ is the uniform distribution on $\cS$. For any $s \in \cS$, $\{\bs{\phi}(s,a)\}_{a \in \cA}$ is the canonical basis of $\R^d$.  
	In both problems, $\bs{B}=[I_k;\bs{0}]$, where $I_k$ denotes the $k \times k$ identity matrix. 
	$\bs{w}_1,\dots,\bs{w}_M$ are divided into $k$ groups, with $\frac{M}{k}$ same members in each group. 
	The members in the $i$-th group ($i \in [k]$), i.e., $\bs{w}_{(M/k)\times(i-1)+1},\dots,\bs{w}_{(M/k)\times i}$, have $1$ in the $i$-th coordinate and $0$ in all other coordinates. 
	For any $m \in [M]$, $\bs{\theta}_m=\bs{B} \bs{w}_m$. We vary $M$ and perform $50$ independent runs  to report the average sample complexity across runs.

	\begin{figure}[t]
		\centering     
		\subfigure[$\probbai$] { 
			\label{fig:bai}        
			\includegraphics[width=0.45\columnwidth]{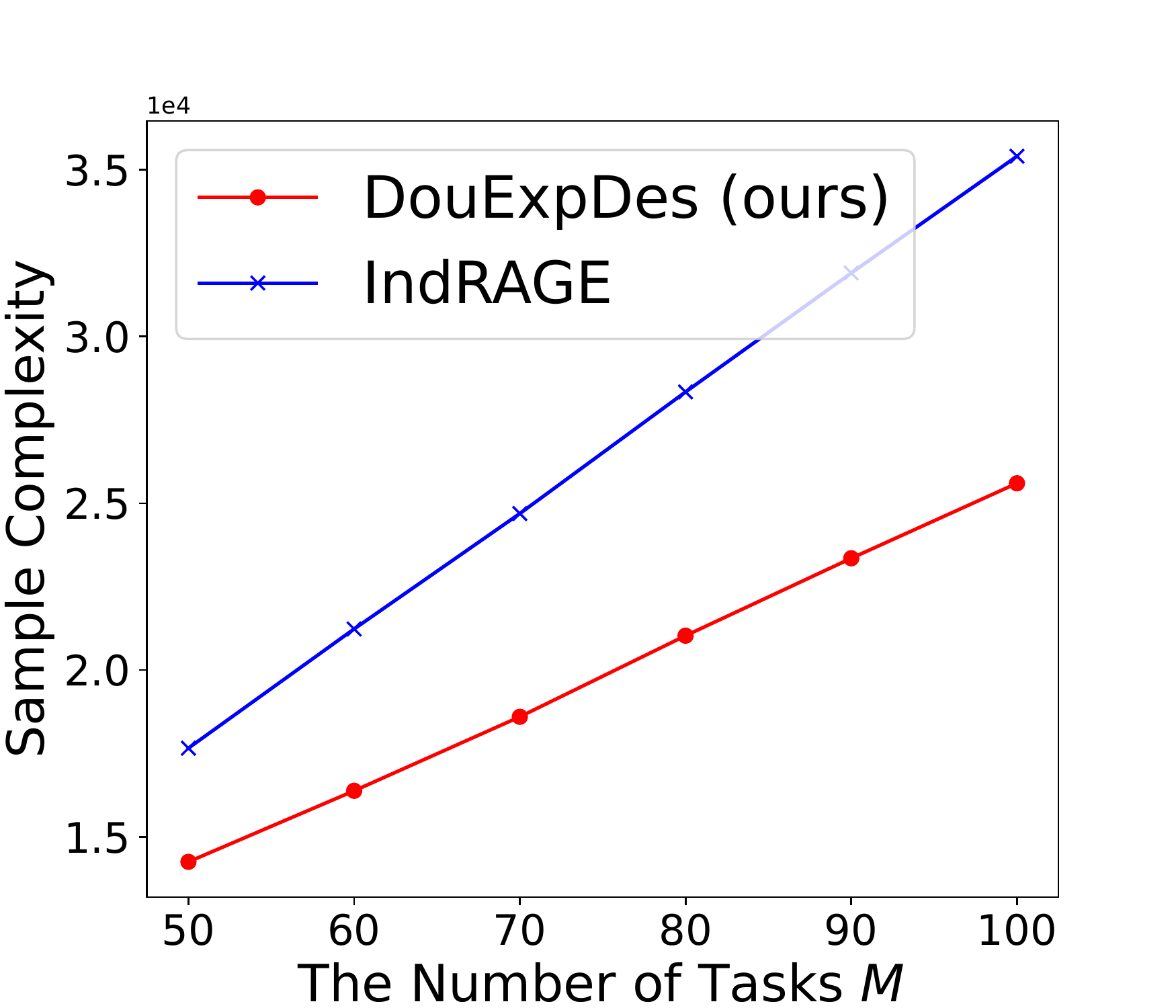} 
		}    
		\subfigure[$\probbpi$] { 
			\label{fig:bpi}       
			\includegraphics[width=0.45\columnwidth]{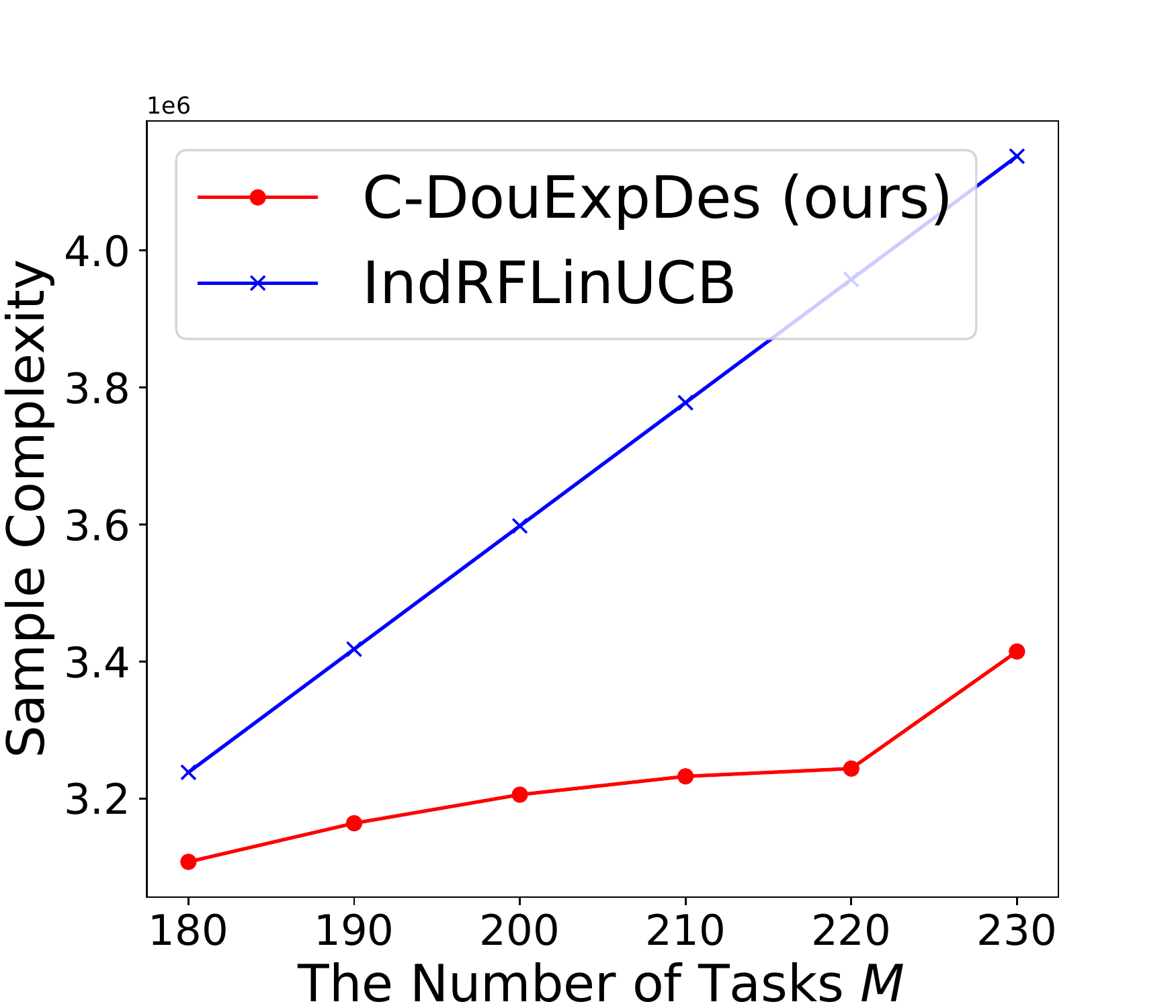}
		}
		\caption{Experimental results for $\probbai$ and $\probbpi$. The two figures compare the sample complexities of our algorithms with the naive algorithms which treat each task independently.
		}
		\label{fig:experiments}
	\end{figure} 
	
	For $\probbai$, we compare algorithm $\algrepbailb$ with the baseline $\mathtt{IndRAGE}$ which runs the state-of-the-art single-task BAI-LB algorithm $\mathtt{RAGE}$~\cite{fiez2019sequential} to solve $M$ tasks independently.
	Figure~\ref{fig:bai} shows the empirical results for $\probbai$. From Figure~\ref{fig:bai}, we can see that $\algrepbailb$ has a better sample complexity than $\mathtt{IndRAGE}$, and as the number of tasks $M$ increases, the sample complexity of $\algrepbailb$ increases at a lower rate than that of $\mathtt{IndRAGE}$. This demonstrates that $\algrepbailb$ effectively utilize the shared representation among tasks to reduce the number of samples needed for multi-task learning.
	
	For $\probbpi$, our algorithm $\algrepbpiclb$ is compared with the baseline $\mathtt{IndRFLinUCB}$, which tackles $M$ tasks independently by calling the state-of-the-art single-task BPI-CLB algorithm $\mathtt{Reward \hyphen free\ LinUCB}$~\cite{zanette2021design}. As presented in Figure~\ref{fig:bpi}, $\algrepbpiclb$ achieves a significantly lower sample complexity than $\mathtt{IndRFLinUCB}$. In addition, the slope of the sample complexity curve of $\algrepbpiclb$ with respect to $M$ is much smaller than that of $\mathtt{IndRFLinUCB}$, which validates that $\algrepbpiclb$ enjoys a lighter dependency on dimension in multi-task learning. These empirical results match our theoretical bounds, and corroborate the power of representation learning. 
	
	\section{Conclusion and Future Work}
	
	In this paper, we investigate representation learning for pure exploration in multi-task (contextual) linear bandits. We propose two efficient algorithms which conduct double experimental designs to optimally allocate samples for learning the low-rank representation. The sample complexities of our algorithms mainly depend on the low dimension of the underlying joint representation among tasks, instead of the raw high dimension. Our theoretical and experimental results demonstrate the benefit of representation learning for pure exploration in multi-task bandits. 
	There are many interesting directions for further exploration. One direction is to establish lower bounds to validate the optimality of our algorithms. Another direction is to extend this work to more complex (nonlinear) representation settings.
	
	\section*{Acknowledgements}
	
	The work of Yihan Du and Longbo Huang is supported by the Technology and Innovation Major Project of the Ministry of Science and Technology of China under Grant 2020AAA0108400 and 2020AAA0108403 and the Tsinghua Precision Medicine Foundation 10001020109.
	Wen Sun acknowledges funding support from NSF IIS-2154711.

	\bibliography{icml23_RepLinearBandit_ref}
	\bibliographystyle{icml2023}
	
	\clearpage
	
	\OnlyInFull{
		\input{icml23_RepLinearBandit_supp}
	}
	
\end{document}

%% file: icml23_RepLinearBandit_supp.tex
\appendix
\onecolumn

\renewcommand{\appendixpagename}{\centering \Large Appendix}
\appendixpage

%

%
%

\section{Related Work} \label{apx:related_work}

In this section, we present a full literature review for two lines of related works, i.e., representation learning and pure exploration in (contextual) linear bandits.

\cameraReady{
\textbf{Representation Learning.}
The study of representation learning has been initiated and developed in the supervised learning setting, e.g., \cite{baxter2000model,ben2003exploiting,ando2005framework,maurer2006bounds,cavallanti2010linear,maurer2016benefit,du2020few,tripuraneni2021provable}. A most related work is \cite{tripuraneni2021provable}, which proposes a method-of-moments estimator for recovering the feature extractor, and establishes error guarantees for transferring the learned representation from past tasks to a new task.

Recently, representation learning for sequential decision making (bandits and reinforcement learning) has attracted extensive attention.
%
We first introduce several works on low-rank bandits, which is a very similar topic to representation learning for bandits.
\citet{lale2019stochastic} study linear bandits with a hidden low-rank structure, and provide a regret bound dependent on the eigenvalue of the action distribution covariance.
\citet{jun2019bilinear,lu2021low} also investigate  low-rank linear bandits (bilinear bandits), and design algorithms which run traditional linear bandit algorithm LinUCB~\cite{abbasi2011improved} in the estimated low-dimensional subspace.
\citet{lattimore2021bandit} consider an instantiation of  low-rank bandits, called bandit phase retrieval.
\citet{huang2021optimal} study a large family of bandit problems with non-concave reward functions, including low-rank linear bandits. They design a stochastic gradient-based algorithm that achieves an improved regret bound over those in \cite{jun2019bilinear,lu2021low}.

Now we introduce related works on representation learning for bandits.
\citet{yang2020impact,yang2022nearly} study multi-task representation learning for linear bandits with the regret minimization objective, and assume that the action set at each timestep is an ellipsoid or sphere. \citet{hu2021near} further relax this assumption and allow arbitrary action sets, but their algorithms equipped with a multi-task joint least-square estimator are computationally inefficient. \citet{cella2022meta,cella2022multi} also investigate the problem in \cite{yang2020impact} and propose algorithms which do not need to know the dimension of the underlying representation. 
\citet{qin2022non} study multi-task representation learning for linear bandits in a non-stationary environment, and develop algorithms that learn and transfer non-stationary representations adaptively.

There are also other works studying multi-task representation learning for reinforcement learning (RL).
\citet{lu2021power,lu2022provable} consider multi-task representation learning for linear MDPs, where the agent learns a shared representation function from a given function class. 
\citet{pacchiano2022joint} investigate multi-task RL with a joint  low-dimensional linear representation, and design a computationally efficient algorithm using a bilinear optimization oracle. 
\citet{zhang2021provably} consider multi-task (multi-player) RL in tabular MDPs, where the relatedness of MDPs are measured by the similarity of reward functions and transition distributions.
\citet{chengprovable,agarwal2022provable} study multi-task representation learning and representational transfer for low-rank MDPs, where multiple low-rank MDPs share a common state-action feature mapping. 

Different from the above works which consider regret minimization, we study representation learning for (contextual) linear bandits with the pure exploration objective, which imposes unique challenges on how to optimally allocate samples to learn the feature extractor, and motivates us to design algorithms based on double experimental designs.
}

\textbf{Pure Exploration in (Contextual) Linear Bandits.}
Most linear bandit studies consider regret minimization, e.g.,~\cite{dani2008stochastic,rusmevichientong2010linearly,chu2011contextual,abbasi2011improved}. Recently, there is a surge of interests in pure exploration for (contextual) linear bandits, e.g., \cite{soare2014best,tao2018best,xu2018fully,fiez2019sequential,katz2020empirical,degenne2020gamification,jedra2020optimal,du2021combinatorial,zanette2021design,li2022instance}. 
For linear bandits, \citet{soare2014best} firstly apply the G-optimal design to identify the best arm, and  provide a sample complexity result that heavily depends on the minimum reward gap.
\citet{tao2018best} design a novel randomized estimator for the underlying reward parameter, and achieve tighter sample complexity which depends on the reward gaps of the best $d$ arms. \citet{du2021combinatorial} further extend the algorithm in \citep{tao2018best} to develop a polynomial-time algorithm for combinatorially large arm sets. \citet{xu2018fully} propose a fully-adaptive algorithm which changes the arm selection strategy at each timestep. \citet{fiez2019sequential} establish the first near-optimal sample complexity upper and lower bounds for best arm identification in linear bandits. \citet{katz2020empirical} further extend the algorithm in \cite{fiez2019sequential} and use empirical processes to avoid an explicit union bound over the number of arms. \citet{degenne2020gamification,jedra2020optimal} develop asymptotically optimal algorithms using the track-and-stop approaches. For contextual linear bandits, \citet{zanette2021design} design a single non-adaptive policy to collect a dataset, from which a near-optimal policy can be computed. \citet{li2022instance} build the first instance-dependent upper and lower bounds for best policy identification in contextual linear bandits, with the prior knowledge of the context distribution. By contrast, our work studies multi-task best arm/policy identification in (contextual) linear bandits with a shared representation among tasks, and does not assume any prior knowledge of the context distribution.

\section{Rounding Procedure} \label{apx:rounding_procedure}

In this section, we introduce the rounding procedure $\round$ in detail.

Let $\cX^{+}:=\cX \cup \cA$ denote the union space of arm set $\cX$ and action space $\cA$. There are $n$ arms or actions $p_1,\dots,p_n \in \cX^{+}$ and $n$ positive semi-definite matrices $\bs{Q}_1,\dots,\bs{Q}_n \in \mathbb{S}^{d}_{+}$, where $\bs{Q}_i$ represents the feature of arm or action $p_i$ for any $i \in [n]$. Denote $\cP:=\{p_1,\dots,p_n\}$ and $\cQ:=\{\bs{Q}_1,\dots,\bs{Q}_n\}$.

The rounding procedure $\round(\{(p_i,\bs{Q}_i)\}_{i=1}^{n}, \bs{\lambda}, \zeta, N)$~\cite{allen2017near,fiez2019sequential} takes $n$ arm-matrix or action-matrix pairs $(p_1,\bs{Q}_1),\dots,(p_n,\bs{Q}_n) \in \cX^{+} \times \mathbb{S}^{d}_{+}$, a distribution $\bs{\lambda} \in \triangle_{\cP}$ (or equivalently, $\bs{\lambda} \in \triangle_{\cQ}$), an approximation parameter $\zeta>0$, and the number of samples $N$ which satisfies that $N \geq \frac{180d}{\zeta^2}$ as inputs. 
Roughly speaking, it will find a $N$-length discrete arm or action sequence whose associated feature matrices maintain the similar property (e.g., G-optimality and E-optimality) as the continuous sample allocation $\bs{\lambda}$. 

Formally, $\round(\{(p_i,\bs{Q}_i)\}_{i=1}^{n}, \lambda, \zeta, N)$ returns a discrete sample sequence $s_1, \dots, s_N \in \cP^{N}$ associated with feature matrices $\bs{S}_1, \dots, \bs{S}_N \in \cQ^{N}$, which satisfy the following properties: 

(i) If $\bs{\lambda}$ is an E-optimal design, i.e., $\bs{\lambda}$ is the optimal solution of the optimization
$$
\min_{\bs{\lambda} \in \triangle_{\cQ}} \nbr{\sbr{\sum_{i=1}^n \lambda(\bs{Q}_i) \bs{Q}_i}^{-1}} ,
$$
then $\bs{S}_1, \dots, \bs{S}_N$ satisfy that
\begin{align*}
	\nbr{\sbr{\sum_{j=1}^N \bs{S}_j }^{-1}} \leq (1+\zeta) \nbr{\sbr{N \sum_{i=1}^n \lambda(\bs{Q}_i) \bs{Q}_i }^{-1}} .
\end{align*}

(ii) If $\bs{\lambda}$ is a G-optimal design, i.e., for a given prediction set $\cY \subseteq \R^d$, $\lambda$ is the optimal solution of the optimization
$$
\min_{\bs{\lambda} \in \triangle_{\cQ}} \max_{\bs{y} \in \cY} \nbr{\bs{y}}^2_{\sbr{\sum_{i=1}^n \lambda(\bs{Q}_i) \bs{Q}_i}^{-1}} ,
$$
then $\bs{S}_1, \dots, \bs{S}_N$ satisfy that
\begin{align*}
	\max_{\bs{y} \in \cY} \nbr{\bs{y}}^2_{\sbr{\sum_{j=1}^{N} \bs{S}_j}^{-1}} \leq (1+\zeta) \max_{\bs{y} \in \cY} \nbr{\bs{y}}^2_{\sbr{N \sum_{i=1}^{n} \lambda(\bs{Q}_i) \bs{Q}_i}^{-1}} .
\end{align*}


We implement $\round$ by setting $\bs{\pi}^*=N\bs{\lambda}$, $k=r=N$ and $\bs{x}_i \bs{x}_i^\top = (\sum_{i=1}^{n} \pi^*(\bs{Q}_i) \bs{Q}_i)^{-\frac{1}{2}} \bs{Q}_i (\sum_{i=1}^{n} \pi^*(\bs{Q}_i) \bs{Q}_i)^{-\frac{1}{2}}$ for any $i \in [n]$ in Algorithm 1 of \cite{allen2017near}. 
Note that Algorithm 1 in \cite{allen2017near} only needs to access the feature matrix $\bs{x}_i \bs{x}_i^\top$ rather than the separate feature vector $\bs{x}_i$, which allows us to apply it to our problem. We refer interested readers to \cite{allen2017near} and Appendix B in \cite{fiez2019sequential} for more implementation details of this rounding procedure.

\section{Proofs for Algorithm $\algrepbailb$}

In this section, we provide the proofs for Algorithm~$\algrepbailb$. 

Throughout our proofs, we use $L_{\theta}$ to denote the upper bound of $\|\bs{\theta}_m\|$ for any $m \in [M]$. 
Since $\bs{\theta}_m=\bs{B}\bs{w}_m$ for any $m \in [M]$, we have that $\|\bs{\theta}_m\| \leq \|\bs{B}\| \|\bs{w}_m\| \leq \|\bs{w}_m\| \leq L_{w}$, and thus $L_{\theta} \leq L_{w}$.

\subsection{Sample Batch Planning}

Recall that
\begin{align*}
	\bs{\lambda}^{E} := & \argmin_{\bs{\lambda} \in \triangle_{\cX}} \nbr{\sbr{\sum_{i=1}^n \lambda(\bs{x}_i) \bs{x}_i \bs{x}_i^\top}^{-1}} 
\end{align*}
and
\begin{align*}
	\rho^{E} := & \min_{\bs{\lambda} \in \triangle_{\cX}} \nbr{\sbr{\sum_{i=1}^n \lambda(\bs{x}_i) \bs{x}_i \bs{x}_i^\top}^{-1}} 
\end{align*}
are the optimal solution and the optimal value of the E-optimal design optimization, respectively (Line~\ref{line:bai_E_optimal_design} in Algorithm~\ref{alg:repbailb}). $\bar{\bs{x}}_1,\dots,\bar{\bs{x}}_p$ is an arm sequence generated according to sample allocation $\bs{\lambda}^{E}$ via rounding procedure $\round$ (Line~\ref{line:bai_E_optimal_round} in Algorithm~\ref{alg:repbailb}).

Let 
$$
\bs{X}_{\batch}:=
\begin{bmatrix}
	\bar{\bs{x}}_1^\top
	\\
	\dots
	\\
	\bar{\bs{x}}_p^\top
\end{bmatrix} ,
$$
and 
$$
\bs{X}_{\batch}^{+} := (\bs{X}_{\batch}^\top \bs{X}_{\batch})^{-1} \bs{X}_{\batch}^{\top} .
$$
According to the fact that $\cX$ spans $\R^d$, the definition of E-optimal design and the guarantee of $\round$, we have that $\bs{X}_{\batch}^\top \bs{X}_{\batch}$ is invertible.

Now, we first give an upper bound of $\|\bs{X}_{\batch}^{+}\|$.

\begin{lemma} \label{lemma:X_batch_ub}
	It holds that
	\begin{align*}
		\|\bs{X}_{\batch}^{+}\| \leq \sqrt{\frac{(1+\zeta) \rho^{E}}{p}} .
	\end{align*}
\end{lemma}
\begin{proof}[Proof of Lemma~\ref{lemma:X_batch_ub}]
	We have
	\begin{align*}
		\|\bs{X}_{\batch}^{+}\| = & \nbr{\sbr{\bs{X}_{\batch}^{\top} \bs{X}_{\batch}}^{-1} \bs{X}_{\batch}^{\top}}
		\\
		= & \sqrt{ \nbr{\sbr{\bs{X}_{\batch}^{\top} \bs{X}_{\batch}}^{-1} \bs{X}_{\batch}^{\top} \bs{X}_{\batch} \sbr{\bs{X}_{\batch}^{\top} \bs{X}_{\batch}}^{-1}} }
		\\
		= & \sqrt{ \nbr{\sbr{\bs{X}_{\batch}^{\top} \bs{X}_{\batch}}^{-1}} }
		\\
		= & \sqrt{ \nbr{\sbr{\sum_{i=1}^{p} \bar{\bs{x}}_i \bar{\bs{x}}_i^{\top}}^{-1}} }
		\\
		\leq & \sqrt{(1+\zeta) \nbr{\sbr{p \sum_{i=1}^{n} \lambda^{E}(\bs{x}_i) \bs{x}_i \bs{x}_i^{\top}}^{-1}} }
		\\
		= & \sqrt{\frac{(1+\zeta) \rho^{E}}{p}} .
	\end{align*}
\end{proof}

\subsection{Global Feature Extractor Recovery} \label{apx:bai_feature_recover}

For clarity of notation, we add subscript $t$ to the notations in subroutine $\algfeatrecover$ to denote the quantities generated in phase $t$. Specifically, we use $\alpha_{t,m,j,i}$, $\tilde{\bs{\theta}}_{t,m,j}$, $\bs{Z}_t$ and $\hat{\bs{B}}_t$ to denote the random reward, estimator of reward parameter, estimator of $\frac{1}{M} \sum_{i=1}^{M} \bs{\theta}_m \bs{\theta}_m^\top$ and estimator of feature extractor in phase $t$, respectively.

For any phase $t>0$, task $m \in [M]$, round $j \in [T_t]$ and arm $i \in [p]$, let $\eta_{t,m,j,i}$ denote the noise of the sample on arm $\bar{\bs{x}}_i$ in the $j$-th round for task $m$, during the execution of $\algfeatrecover$ in phase $t$ (Line~\ref{line:bai_stage2_sample} in Algorithm~\ref{alg:feat_recover}). The noise $\eta_{t,m,j,i}$ is zero-mean and sub-Gaussian, and has variance $1$. $\eta_{t,m,j,i}$ is independent for different $t,m,j,i$.

For any phase $t>0$, task $m \in [M]$, round $j \in [T_t]$, let $\bs{\alpha}_{t,m,j}:=[\alpha_{t,m,j,1},\dots,\alpha_{t,m,j,p}]^\top$.
Then, we have that
$$
\tilde{\bs{\theta}}_{t,m,j} = \bs{X}_{\batch}^{+} \bs{\alpha}_{t,m,j} ,
$$
and
$$
\bs{Z}_t = \frac{1}{M T_t} \sum_{m=1}^{M} \sum_{j=1}^{T_t} \tilde{\bs{\theta}}_{t,m,j} (\tilde{\bs{\theta}}_{t,m,j})^\top - \bs{X}_{\batch}^{+} (\bs{X}_{\batch}^{+})^\top .
$$

\begin{lemma}[Expectation of $\bs{Z}_t$] \label{lemma:bai_expectation_Z_t}
	It holds that
	\begin{align*}
		\ex \mbr{ \bs{Z}_t } = \frac{1}{M} \sum_{m=1}^{M} \bs{\theta}_m \bs{\theta}_m^\top .
	\end{align*}
\end{lemma}
\begin{proof}[Proof of Lemma~\ref{lemma:bai_expectation_Z_t}]
	$\bs{Z}_t$ can be written as
	\begin{align}
		\bs{Z}_t = & \frac{1}{M T_t} \sum_{m=1}^{M} \sum_{j=1}^{T_t} \tilde{\bs{\theta}}_{t,m,j} (\tilde{\bs{\theta}}_{t,m,j})^\top - \bs{X}_{\batch}^{+} (\bs{X}_{\batch}^{+})^\top 
		\nonumber\\
		= & \frac{1}{M T_t} \sum_{m=1}^{M} \sum_{j=1}^{T_t} \bs{X}_{\batch}^{+} \begin{bmatrix}
			\alpha_{t,m,j,1}
			\\
			\vdots
			\\
			\alpha_{t,m,j,p} 
		\end{bmatrix}  
		[\alpha_{t,m,j,1}, \dots, \alpha_{t,m,j,p}]^\top (\bs{X}_{\batch}^{+})^\top
		- \bs{X}_{\batch}^{+} (\bs{X}_{\batch}^{+})^\top 
		\nonumber\\
		= & \frac{1}{M T_t} \sum_{m=1}^{M} \sum_{j=1}^{T_t} \bs{X}_{\batch}^{+} \begin{bmatrix}
			\bar{\bs{x}}_1^\top \bs{\theta}_m + \eta_{t,m,j,1}
			\\
			\vdots
			\\
			\bar{\bs{x}}_p^\top \bs{\theta}_m + \eta_{t,m,j,p} 
		\end{bmatrix}  
		[\bar{\bs{x}}_1^\top \bs{\theta}_m + \eta_{t,m,j,1}, \dots,
		\bar{\bs{x}}_p^\top \bs{\theta}_m + \eta_{t,m,j,p} ]^\top (\bs{X}_{\batch}^{+})^\top
		- \bs{X}_{\batch}^{+} (\bs{X}_{\batch}^{+})^\top 
		\nonumber\\
		= & \frac{1}{M T_t} \sum_{m=1}^{M} \sum_{j=1}^{T_t} \bs{X}_{\batch}^{+} 
		\begin{bmatrix}
			(\bar{\bs{x}}_1^\top \bs{\theta}_m + \eta_{t,m,j,1})^2 & \!\!\!\!\!\cdots\!\!\!\!\! & (\bar{\bs{x}}_1^\top \bs{\theta}_m + \eta_{t,m,j,1}) (\bar{\bs{x}}_p^\top \bs{\theta}_m + \eta_{t,m,j,p}) \\
			\cdots	& \!\!\!\!\!\cdots\!\!\!\!\! & \cdots \\
			(\bar{\bs{x}}_p^\top \bs{\theta}_m + \eta_{t,m,j,p})(\bar{\bs{x}}_1^\top \bs{\theta}_m + \eta_{t,m,j,1}) & \!\!\!\!\!\cdots\!\!\!\!\! & (\bar{\bs{x}}_p^\top \bs{\theta}_m + \eta_{t,m,j,p})^2
		\end{bmatrix}
		(\bs{X}_{\batch}^{+})^\top
		\nonumber\\& - \bs{X}_{\batch}^{+} (\bs{X}_{\batch}^{+})^\top 
		\nonumber\\
		= & \frac{1}{M T_t} \sum_{m=1}^{M} \sum_{j=1}^{T_t} \bs{X}_{\batch}^{+} 
		\Bigg(
		\begin{bmatrix}
			(\bar{\bs{x}}_1^\top \bs{\theta}_m)^2 & \cdots & \bar{\bs{x}}_1^\top \bs{\theta}_m \bar{\bs{x}}_p^\top \bs{\theta}_m \\
			\cdots	& \cdots & \cdots \\
			\bar{\bs{x}}_1^\top \bs{\theta}_m \bar{\bs{x}}_p^\top \bs{\theta}_m & \cdots & (\bar{\bs{x}}_p^\top \bs{\theta}_m)^2 
		\end{bmatrix}
		\nonumber\\& +
		\begin{bmatrix}
			2 \bar{\bs{x}}_1^\top \bs{\theta}_m \eta_{t,m,j,1} & \cdots & \bar{\bs{x}}_1^\top \bs{\theta}_m \eta_{t,m,j,p} + \bar{\bs{x}}_p^\top \bs{\theta}_m \eta_{t,m,j,1} \\
			\cdots	& \cdots & \cdots \\
			\bar{\bs{x}}_1^\top \bs{\theta}_m \eta_{t,m,j,p} + \bar{\bs{x}}_p^\top \bs{\theta}_m \eta_{t,m,j,1} & \cdots & 2 \bar{\bs{x}}_p^\top \bs{\theta}_m \eta_{t,m,j,p}
		\end{bmatrix}
		\nonumber\\& +
		\begin{bmatrix}
			(\eta_{t,m,j,1})^2 & \cdots & \eta_{t,m,j,1} \eta_{t,m,j,p} \\
			\cdots	& \cdots & \cdots \\
			\eta_{t,m,j,1} \eta_{t,m,j,p} & \cdots & (\eta_{t,m,j,p})^2
		\end{bmatrix}
		\Bigg)
		(\bs{X}_{\batch}^{+})^\top
		- \bs{X}_{\batch}^{+} (\bs{X}_{\batch}^{+})^\top . \label{eq:Z_t_decompose}
	\end{align}
	
	Then, taking the expectation on $\bs{Z}_t$, we have
	\begin{align*}
		\ex [\bs{Z}_t] = & \frac{1}{M T_t} \sum_{m=1}^{M} \sum_{j=1}^{T_t} \bs{X}_{\batch}^{+} 
		\Bigg( \begin{bmatrix}
			(\bar{\bs{x}}_1^\top \bs{\theta}_m)^2 & \cdots & \bar{\bs{x}}_1^\top \bs{\theta}_m \bar{\bs{x}}_p^\top \bs{\theta}_m \\
			\cdots	& \cdots & \cdots \\
			\bar{\bs{x}}_p^\top \bs{\theta}_m \bar{\bs{x}}_1^\top \bs{\theta}_m  & \cdots & (\bar{\bs{x}}_p^\top \bs{\theta}_m)^2
		\end{bmatrix} + \bs{I}_d \Bigg)
		(\bs{X}_{\batch}^{+})^\top
		- \bs{X}_{\batch}^{+} (\bs{X}_{\batch}^{+})^\top
		\\
		= & \frac{1}{M T_t} \sum_{m=1}^{M} \sum_{j=1}^{T_t} \bs{X}_{\batch}^{+} 
		\begin{bmatrix}
			(\bar{\bs{x}}_1^\top \bs{\theta}_m)^2 & \cdots & \bar{\bs{x}}_1^\top \bs{\theta}_m \bar{\bs{x}}_p^\top \bs{\theta}_m \\
			\cdots	& \cdots & \cdots \\
			\bar{\bs{x}}_p^\top \bs{\theta}_m \bar{\bs{x}}_1^\top \bs{\theta}_m  & \cdots & (\bar{\bs{x}}_p^\top \bs{\theta}_m)^2
		\end{bmatrix} 
		(\bs{X}_{\batch}^{+})^\top
		\\
		= & \frac{1}{M T_t} \sum_{m=1}^{M} \sum_{j=1}^{T_t} \bs{X}_{\batch}^{+} 
		\begin{bmatrix}
			\bar{\bs{x}}_1^\top \bs{\theta}_m
			\\
			\vdots
			\\
			\bar{\bs{x}}_p^\top \bs{\theta}_m
		\end{bmatrix} [\bar{\bs{x}}_1^\top \bs{\theta}_m, \dots, \bar{\bs{x}}_p^\top \bs{\theta}_m]^\top
		(\bs{X}_{\batch}^{+})^\top
		\\
		= & \frac{1}{M T_t} \sum_{m=1}^{M} \sum_{j=1}^{T_t} \bs{X}_{\batch}^{+} 
		\bs{X}_{\batch} \bs{\theta}_m \bs{\theta}_m^\top \bs{X}_{\batch}^\top
		(\bs{X}_{\batch}^{+})^\top
		\\
		= & \frac{1}{M T_t} \sum_{m=1}^{M} \sum_{j=1}^{T_t} \bs{\theta}_m \bs{\theta}_m^\top 
		\\
		= & \frac{1}{M} \sum_{m=1}^{M}  \bs{\theta}_m \bs{\theta}_m^\top .
	\end{align*}
\end{proof}

Recall that for any $t>0$, $\delta_t:=\frac{\delta}{2t^2}$.

For any phase $t>0$, define events 
\begin{align*}
	\cE_t:= \lbr{ \nbr{\bs{Z}_t - \ex [\bs{Z}_t]} 
		\leq \frac{ 96 \nbr{\bs{X}_{\batch}^{+}}^2 p L_{x} L_{\theta} \log\sbr{\frac{16p}{\delta_t}} }{\sqrt{MT_t}} \log \sbr{\frac{16pMT_t}{\delta_t}} } ,
\end{align*}
and
\begin{align*}
	\cE:=\cap_{t=1}^{\infty} \cE_t .
\end{align*}

\begin{lemma}[Concentration of $\bs{Z}_t$] \label{lemma:Z_t_est_error}
	It holds that
	\begin{align*}
		\Pr \mbr{\cE} \geq \frac{\delta}{2} .  
	\end{align*}
\end{lemma}

\begin{proof}[Proof of Lemma~\ref{lemma:Z_t_est_error}]
	According to Eq.~\eqref{eq:Z_t_decompose}, we have
	\begin{align*}
		\bs{Z}_t - \ex [\bs{Z}_t] = & \frac{1}{M T_t} \sum_{m=1}^{M} \sum_{j=1}^{T_t} \bs{X}_{\batch}^{+} 
		\Bigg(
		\begin{bmatrix}
			2 \bar{\bs{x}}_1^\top \bs{\theta}_m \eta_{t,m,j,1} & \cdots & \bar{\bs{x}}_1^\top \bs{\theta}_m \eta_{t,m,j,p} + \bar{\bs{x}}_p^\top \bs{\theta}_m \eta_{t,m,j,1} \\
			\cdots	& \cdots & \cdots \\
			\bar{\bs{x}}_1^\top \bs{\theta}_m \eta_{t,m,j,p} + \bar{\bs{x}}_p^\top \bs{\theta}_m \eta_{t,m,j,1} & \cdots & 2 \bar{\bs{x}}_p^\top \bs{\theta}_m \eta_{t,m,j,p}
		\end{bmatrix}
		\\& -
		\ex
		\begin{bmatrix}
			2 \bar{\bs{x}}_1^\top \bs{\theta}_m \eta_{t,m,j,1} & \cdots & \bar{\bs{x}}_1^\top \bs{\theta}_m \eta_{t,m,j,p} + \bar{\bs{x}}_p^\top \bs{\theta}_m \eta_{t,m,j,1} \\
			\cdots	& \cdots & \cdots \\
			\bar{\bs{x}}_1^\top \bs{\theta}_m \eta_{t,m,j,p} + \bar{\bs{x}}_p^\top \bs{\theta}_m \eta_{t,m,j,1} & \cdots & 2 \bar{\bs{x}}_p^\top \bs{\theta}_m \eta_{t,m,j,p}
		\end{bmatrix}
		\nonumber\\& +
		\begin{bmatrix}
			(\eta_{t,m,j,1})^2 & \cdots & \eta_{t,m,j,1} \eta_{t,m,j,p} \\
			\cdots	& \cdots & \cdots \\
			\eta_{t,m,j,1} \eta_{t,m,j,p} & \cdots & (\eta_{t,m,j,p})^2
		\end{bmatrix}
		-
		\ex
		\begin{bmatrix}
			(\eta_{t,m,j,1})^2 & \cdots & \eta_{t,m,j,1} \eta_{t,m,j,p} \\
			\cdots	& \cdots & \cdots \\
			\eta_{t,m,j,1} \eta_{t,m,j,p} & \cdots & (\eta_{t,m,j,p})^2
		\end{bmatrix}
		\Bigg)
		(\bs{X}_{\batch}^{+})^\top .
	\end{align*}
	
	Define the following matrices:
	\begin{align*}
		\bs{A}_{t,m,j} &:= \frac{1}{M T_t} \begin{bmatrix}
			2 \bar{\bs{x}}_1^\top \bs{\theta}_m \eta_{t,m,j,1} & \cdots & \bar{\bs{x}}_1^\top \bs{\theta}_m \eta_{t,m,j,p} + \bar{\bs{x}}_p^\top \bs{\theta}_m \eta_{t,m,j,1} \\
			\cdots	& \cdots & \cdots \\
			\bar{\bs{x}}_1^\top \bs{\theta}_m \eta_{t,m,j,p} + \bar{\bs{x}}_p^\top \bs{\theta}_m \eta_{t,m,j,1} & \cdots & 2 \bar{\bs{x}}_p^\top \bs{\theta}_m \eta_{t,m,j,p} 
		\end{bmatrix} ,
		\\
		\bs{A}_t &:= \sum_{m=1}^{M} \sum_{j=1}^{T_t} \bs{A}_{t,m,j} ,
		\\ 
		\bs{C}_{t,m,j} &:= \frac{1}{M T_t} \begin{bmatrix}
			(\eta_{t,m,j,1})^2 & \cdots & \eta_{t,m,j,1} \eta_{t,m,j,p} \\
			\cdots	& \cdots & \cdots \\
			\eta_{t,m,j,1} \eta_{t,m,j,p} & \cdots & (\eta_{t,m,j,p})^2 
		\end{bmatrix} ,
		\\
		\bs{C}_t &:= \sum_{m=1}^{M} \sum_{j=1}^{T_t} \bs{C}_{t,m,j} .
	\end{align*}
	
	Then, we can write $\bs{Z}_t - \ex [\bs{Z}_t]$ as
	\begin{align*}
		\bs{Z}_t - \ex [\bs{Z}_t] = \bs{X}_{\batch}^{+} \sbr{ \bs{A}_t - \ex[\bs{A}_t] + \bs{C}_t - \ex[\bs{C}_t] } (\bs{X}_{\batch}^{+})^\top ,
	\end{align*}
	and thus,
	\begin{align}
		\nbr{\bs{Z}_t - \ex [\bs{Z}_t]} \leq \nbr{\bs{X}_{\batch}^{+}}^2 \sbr{ \nbr{\bs{A}_t - \ex[\bs{A}_t]} + \nbr{\bs{C}_t - \ex[\bs{C}_t]} } . \label{eq:Z_t_expressed_by_A}
	\end{align}
	
	Next, we analyze $\|\bs{A}_t - \ex[\bs{A}_t]\|$ and $\|\bs{C}_t - \ex[\bs{C}_t]\|$. In order to use the truncated matrix Bernstein inequality (Lemma~\ref{lemma:matrix_bernstein_tau}), we define the truncated noise and truncated matrices as follows.
	
	Let $R>0$ be a truncation level of noises, which will be chosen later.  For any $t>0$, $m \in [M]$, $j \in [T_t]$ and $i \in [p]$, let $\tilde{\eta}_{t,m,j,i}=\eta_{t,m,j,i} \mathbbm{1}\{|\eta_{t,m,j,i}| \leq R\}$ denote the truncated noise. Then, we define the following truncated matrices:
	\begin{align}
		\tilde{\bs{A}}_{t,m,j} &:= \frac{1}{M T_t} \begin{bmatrix}
			2 \bar{\bs{x}}_1^\top \bs{\theta}_m \tilde{\eta}_{t,m,j,1} & \cdots & \bar{\bs{x}}_1^\top \bs{\theta}_m \tilde{\eta}_{t,m,j,p} + \bar{\bs{x}}_p^\top \bs{\theta}_m \tilde{\eta}_{t,m,j,1} \\
			\cdots	& \cdots & \cdots \\
			\bar{\bs{x}}_1^\top \bs{\theta}_m \tilde{\eta}_{t,m,j,p} + \bar{\bs{x}}_p^\top \bs{\theta}_m \tilde{\eta}_{t,m,j,1} & \cdots & 2 \bar{\bs{x}}_p^\top \bs{\theta}_m \tilde{\eta}_{t,m,j,p}
		\end{bmatrix} 
		\nonumber\\
		\tilde{\bs{A}}_t &:= \sum_{m=1}^{M} \sum_{j=1}^{T_t} \tilde{\bs{A}}_{t,m,j} ,
		\nonumber\\
		\tilde{\bs{C}}_{t,m,j} &:= \frac{1}{M T_t} \begin{bmatrix}
			(\tilde{\eta}_{t,m,j,1})^2 & \cdots & \tilde{\eta}_{t,m,j,1} \tilde{\eta}_{t,m,j,p} \\
			\cdots	& \cdots & \cdots \\
			\tilde{\eta}_{t,m,j,1} \tilde{\eta}_{t,m,j,p} & \cdots & (\tilde{\eta}_{t,m,j,p})^2 
		\end{bmatrix}  \label{eq:tilde_C_t_m_j}
		\\
		\tilde{\bs{C}}_t &:= \sum_{m=1}^{M} \sum_{j=1}^{T_t} \tilde{\bs{C}}_{t,m,j}  \nonumber
	\end{align}
	
	First, we bound $\|\bs{A}_t-\ex[\bs{A}_t]\|$. 
	Since for any $t>0$, $m \in [M]$, $j \in [T_t]$ and $i \in [p]$, $|\tilde{\eta}_{t,m,j,i}| \leq R$ and $|\bar{\bs{x}}_i^\top \bs{\theta}_m| \leq L_{x} L_{\theta}$, we have $\|\tilde{\bs{A}}_{t,m,j}\| \leq \frac{1}{M T_t} \cdot 2p L_{x} L_{\theta} R$.

	Recall that for any $t>0$, $m \in [M]$, $j \in [T_t]$ and $i \in [p]$, $\eta_{t,m,j,i}$ is 1-sub-Gaussian. Using a union bound over $i \in [p]$, we have that for any $t>0$, $m \in [M]$, $j \in [T_t]$, with probability at least $1-2p\exp(-\frac{R^2}{2})$, $|\eta_{t,m,j,i}| \leq R$ for all $i \in [p]$. Thus, with probability at least $1-2p\exp(-\frac{R^2}{2})$, $\|\bs{A}_{t,m,j}\| \leq \frac{1}{M T_t} \cdot 2p L_{x} L_{\theta} R$. 
	
	Then, we have
	\begin{align*}
		\nbr{ \ex[\bs{A}_{t,m,j}]-\ex[\tilde{\bs{A}}_{t,m,j}] } \leq & \nbr{ \ex \mbr{\bs{A}_{t,m,j} \cdot \indicator{ \nbr{\bs{A}_{t,m,j}} \geq \frac{2p L_{x} L_{\theta} R}{M T_t} }} }
		\\
		\leq & \ex \mbr{ \nbr{\bs{A}_{t,m,j}} \cdot \indicator{ \nbr{\bs{A}_{t,m,j}} \geq \frac{2p L_{x} L_{\theta} R}{M T_t} } } 
		\\
		= & \ex\mbr{ \frac{2p L_{x} L_{\theta} R}{M T_t} \cdot \indicator{ \nbr{\bs{A}_{t,m,j}} \geq \frac{2p L_{x} L_{\theta} R}{M T_t} }} \\& + \mbr{\sbr{ \nbr{\bs{A}_{t,m,j}} - \frac{2p L_{x} L_{\theta} R}{M T_t}} \cdot \indicator{ \nbr{\bs{A}_{t,m,j}} \geq \frac{2p L_{x} L_{\theta} R}{M T_t} } } 
		\\
		= & \frac{2p L_{x} L_{\theta} R}{M T_t} \cdot \Pr\mbr{ \nbr{\bs{A}_{t,m,j}} \geq \frac{2p L_{x} L_{\theta} R}{M T_t} } + \int_0^{\infty}  \Pr \mbr{ \nbr{\bs{A}_{t,m,j}} - \frac{2p L_{x} L_{\theta} R}{M T_t} > x}  dx  
		\\
		\leq & \frac{2p L_{x} L_{\theta} R}{M T_t} \cdot 2p \cdot \exp\sbr{-\frac{R^2}{2}} + \frac{2p L_{x} L_{\theta}}{M T_t}  \int_R^{\infty} \Pr \mbr{ \nbr{\bs{A}_{t,m,j}} > \frac{2p L_{x} L_{\theta} y}{M T_t} }  dy  
		\\
		\leq & \frac{2p L_{x} L_{\theta} R}{M T_t} \cdot 2p \cdot \exp\sbr{-\frac{R^2}{2}} + \frac{2p L_{x} L_{\theta}}{M T_t}  \int_R^{\infty} 2p \exp\sbr{-\frac{y^2}{2}}  dy
		\\
		\leq & \frac{2p L_{x} L_{\theta} R}{M T_t} \cdot 2p \cdot \exp\sbr{-\frac{R^2}{2}} + \frac{2p L_{x} L_{\theta}}{M T_t} \cdot 2p \cdot \frac{1}{R} \cdot \exp\sbr{-\frac{R^2}{2}} 
		\\
		= & \frac{2p L_{x} L_{\theta}}{M T_t} \cdot 2p \cdot \sbr{R+\frac{1}{R}} \exp\sbr{-\frac{R^2}{2}} .
	\end{align*}
	
	Let $\delta' \in (0,1)$ be a confidence parameter which will be chosen later.
	Using the truncated matrix Bernstein inequality (Lemma~\ref{lemma:matrix_bernstein_tau}) with $n=MT_t$, $R=\sqrt{2\log \sbr{\frac{2pMT_t}{\delta'}}}$, $n \Pr[\|\bs{A}_{t,m,j}\| \geq \frac{1}{M T_t} \cdot 2p L_{x} L_{\theta} R] \leq \delta'$, $U=\frac{2p L_{x} L_{\theta} \sqrt{2\log \sbr{\frac{2pMT_t}{\delta'}}}}{MT_t}$, $\sigma^2=M T_t U^2$, $\tau=\frac{4\cdot 2p L_{x} L_{\theta} \sqrt{2\log \sbr{\frac{2pMT_t}{\delta'}}} \log\sbr{\frac{2p}{\delta'}}}{\sqrt{M T_t}} + \frac{4\cdot 2p L_{x} L_{\theta} \sqrt{2\log \sbr{\frac{2pMT_t}{\delta'}}} \log\sbr{\frac{2p}{\delta'}}}{M T_t}$ and $\Delta=\frac{2p L_{x} L_{\theta} \cdot 2 \sqrt{2\log \sbr{\frac{2pMT_t}{\delta'}}}}{M T_t} \cdot \frac{\delta'}{M T_t}$, we have that with probability at least $1-2\delta'$, 
	\begin{align}
		\nbr{\bs{A}_t - \ex[\bs{A}_t]} \leq & \frac{4\cdot 2p L_{x} L_{\theta} \sqrt{2\log \sbr{\frac{2pMT_t}{\delta'}}} \log\sbr{\frac{2p}{\delta'}}}{\sqrt{M T_t}} + \frac{4\cdot 2p L_{x} L_{\theta} \sqrt{2\log \sbr{\frac{2pMT_t}{\delta'}}} \log\sbr{\frac{2p}{\delta'}}}{M T_t} 
		\nonumber\\ 
		\leq & \frac{8\cdot 2p L_{x} L_{\theta} \sqrt{2\log \sbr{\frac{2pMT_t}{\delta'}}} \log\sbr{\frac{2p}{\delta'}}}{\sqrt{M T_t}} .
		\label{eq:A_t_concentration}
	\end{align}
	
	Now we investigate $\|\bs{C}_t-\ex[\bs{C}_t]\|$. Recall that in Eq.~\eqref{eq:tilde_C_t_m_j}, for any $t>0$, $m \in [M]$, $j \in [T_t]$ and $i \in [p]$, $|\tilde{\eta}_{t,m,j,i}| \leq R$. Then, we have $\|\tilde{\bs{C}}_{t,m,j}\| \leq \frac{1}{M T_t} \cdot pR^2$.
	%
	
	Recall that for any $t>0$, $m \in [M]$ and $j \in [T_t]$, with probability at least $1-2p\exp(-\frac{R^2}{2})$, $|\eta_{t,m,j,i}| \leq R$ for all $i \in [p]$. Thus, with probability at least $1-2p\exp(-\frac{R^2}{2})$, $\|\bs{C}_{t,m,j}\| \leq \frac{1}{M T_t} \cdot pR^2$. Then, we have
	\begin{align*}
		\nbr{ \ex[\bs{C}_{t,m,j}]-\ex[\tilde{\bs{C}}_{t,m,j}] } \leq & \nbr{ \ex \mbr{\bs{C}_{t,m,j} \cdot \indicator{ \nbr{\bs{C}_{t,m,j}} \geq \frac{pR^2}{M T_t} }} }
		\\
		\leq & \ex \mbr{ \nbr{\bs{C}_{t,m,j}} \cdot \indicator{ \nbr{\bs{C}_{t,m,j}} \geq \frac{pR^2}{M T_t} } } 
		\\
		= & \ex\mbr{ \frac{pR^2}{M T_t} \cdot \indicator{ \nbr{\bs{C}_{t,m,j}} \geq \frac{pR^2}{M T_t} }} + \mbr{\sbr{ \nbr{\bs{C}_{t,m,j}} - \frac{pR^2}{M T_t}} \cdot \indicator{ \nbr{\bs{C}_{t,m,j}} \geq \frac{pR^2}{M T_t} } } 
		\\
		= & \frac{pR^2}{M T_t} \cdot \Pr\mbr{ \nbr{\bs{C}_{t,m,j}} \geq \frac{pR^2}{M T_t} } + \int_0^{\infty}  \Pr \mbr{ \nbr{\bs{C}_{t,m,j}} - \frac{pR^2}{M T_t} > x}  dx  
		\\
		\leq & \frac{pR^2}{M T_t} \cdot 2p \cdot \exp\sbr{-\frac{R^2}{2}} + \frac{2p}{M T_t}  \int_R^{\infty} \bs{y} \cdot \Pr \mbr{ \nbr{\bs{C}_{t,m,j}} > \frac{dy^2}{M T_t} } dy  
		\\
		\leq & \frac{pR^2}{M T_t} \cdot 2p \cdot \exp\sbr{-\frac{R^2}{2}} + \frac{2p}{M T_t}  \int_R^{\infty}  \bs{y} \cdot 2p \exp\sbr{-\frac{y^2}{2}} dy
		\\
		\leq & \frac{pR^2}{M T_t} \cdot 2p \cdot \exp\sbr{-\frac{R^2}{2}} + \frac{2p}{M T_t} \cdot 2p \cdot \exp\sbr{-\frac{R^2}{2}} 
		\\
		= & \frac{p}{M T_t} \cdot 2p \cdot \sbr{R^2+2} \exp\sbr{-\frac{R^2}{2}} .
	\end{align*}
	
	Using the truncated matrix Bernstein inequality (Lemma~\ref{lemma:matrix_bernstein_tau}) with $n=MT_t$, $R=\sqrt{2\log \sbr{\frac{2pMT_t}{\delta'}}}$, $n \Pr[\|\bs{C}_{t,m,j}\| \geq \frac{1}{M T_t} \cdot pR^2] \leq \delta'$, $U=\frac{p \cdot 2\log \sbr{\frac{2pMT_t}{\delta'}} }{MT_t}$, $\sigma^2=\frac{32p}{M T_t}$, $\tau=\frac{4\cdot p \cdot 2\log \sbr{\frac{2pMT_t}{\delta'}} \log\sbr{\frac{2p}{\delta'}}}{\sqrt{M T_t}} + \frac{4\cdot p \cdot 2\log \sbr{\frac{2pMT_t}{\delta'}} \log\sbr{\frac{2p}{\delta'}}}{M T_t}$ and $\Delta=\frac{p \cdot 2 \cdot 2\log \sbr{\frac{2pMT_t}{\delta'}} }{M T_t} \cdot \frac{\delta'}{M T_t}$, we have that with probability at least $1-2\delta'$, 
	\begin{align}
		\nbr{\bs{C}_t - \ex\mbr{\bs{C}_t}} \leq & \frac{4\cdot 2p \log \sbr{\frac{2pMT_t}{\delta'}} \log\sbr{\frac{2p}{\delta'}}}{\sqrt{M T_t}} +  \frac{4 \cdot 2p \log \sbr{\frac{2pMT_t}{\delta'}} \log \sbr{\frac{2p}{\delta'}}}{M T_t} 
		\nonumber\\
		\leq & \frac{8\cdot 2p \log \sbr{\frac{2pMT_t}{\delta'}} \log\sbr{\frac{2p}{\delta'}}}{\sqrt{M T_t}}
		\label{eq:C_t_concentration}
	\end{align}
	
	Plugging Eqs.~\eqref{eq:A_t_concentration} and \eqref{eq:C_t_concentration} into Eq.~\eqref{eq:Z_t_expressed_by_A}, we have that with probability at least $1-4\delta'$,
	\begin{align*}
		\nbr{\bs{Z}_t - \ex [\bs{Z}_t]} \leq & \nbr{\bs{X}_{\batch}^{+}}^2 \sbr{\nbr{\bs{A}_t - \ex \mbr{\bs{A}_t}} + \nbr{\bs{C}_t  - \ex \mbr{\bs{C}_t}} }
		\\
		\leq & \nbr{\bs{X}_{\batch}^{+}}^2 \sbr{ \frac{8\cdot 2p L_{x} L_{\theta} \sqrt{2\log \sbr{\frac{2pMT_t}{\delta'}}} \log\sbr{\frac{2p}{\delta'}}}{\sqrt{M T_t}} + \frac{8\cdot 2p \log \sbr{\frac{2pMT_t}{\delta'}} \log\sbr{\frac{2p}{\delta'}}}{\sqrt{M T_t}} }
		\\
		\leq & \frac{ 96 \nbr{\bs{X}_{\batch}^{+}}^2 p L_{x} L_{\theta} \log\sbr{\frac{2p}{\delta'}} }{\sqrt{MT_t}} \log \sbr{\frac{2pMT_t}{\delta'}} .
	\end{align*}
	Let $\delta'=\frac{\delta_t}{8}$. Then, we obtain that with probability at least $1-\frac{\delta_t}{2}$,
	\begin{align*}
		\nbr{\bs{Z}_t - \ex [\bs{Z}_t]} \leq & \frac{ 96 \nbr{\bs{X}_{\batch}^{+}}^2 p L_{x} L_{\theta} \log\sbr{\frac{16p}{\delta_t}} }{\sqrt{MT_t}} \log \sbr{\frac{16pMT_t}{\delta_t}} , 
	\end{align*}
	which implies that
	$\Pr \mbr{\cE_t} \geq 1-\frac{\delta_t}{2}$.
	
	Taking a union bound over all phases $t\geq 1$ and recalling $\delta_t:=\frac{\delta}{2t^2}$, we obtain
	\begin{align*}
		\Pr \mbr{\cE} 
		\geq & 1- \sum_{t=1}^{\infty} \Pr \mbr{\bar{\cE_t}}
		\\
		\geq & 1- \sum_{t=1}^{\infty} \frac{\delta_t}{2}
		\\
		= & 1- \sum_{t=1}^{\infty} \frac{\delta}{4t^2}
		\\
		\geq & 1-\frac{\delta}{2} .
	\end{align*}
\end{proof}

For any matrix $\bs{A} \in \R^{m \times n}$ with $m \geq n$, let $\sigma_{\max}(\bs{A})$ and $\sigma_{\min}(\bs{A})$ denote the maximum and minimum singular values of $\bs{A}$, respectively. For any $i \in [m]$, let $\sigma_i(\bs{A})$ denote the $i$-th singular value of $\bs{A}$.

For any matrix $\bs{A} \in \R^{m \times n}$ with $m \geq n$, let $\bs{A}_{\bot}$ denote the orthogonal complement matrix of $\bs{A}$, where the columns of $\bs{A}_{\bot}$ are the orthogonal complement of those of $\bs{A}$. Then, it holds that $\bs{A} \bs{A}^\top + \bs{A}_{\bot} \bs{A}_{\bot}^\top = \bs{I}_m$, where $\bs{I}_m$ is the $m \times m$ identity matrix.

According to Assumption~\ref{assumption:diverse_task}, there exists an absolute constant $c_0$ which satisfies that $\sigma_{\min}(\frac{1}{M} \sum_{m=1}^{M} \bs{w}_m \bs{w}_m^\top) = \sigma_{\min}(\frac{1}{M} \sum_{m=1}^{M} \bs{\theta}_m \bs{\theta}_m^\top) \geq \frac{c_0}{k}$.

\begin{lemma}[Concentration of $\hat{\bs{B}}_t$] \label{lemma:concentration_B_hat_t}
	Suppose that event $\cE$ holds. Then, for any phase $t>0$,
	\begin{align*}
		\nbr{\hat{\bs{B}}_{t,\bot}^\top \bs{B}} \leq \frac{ 192 \nbr{\bs{X}_{\batch}^{+}}^2 k p L_{x} L_{\theta} \log\sbr{\frac{16p}{\delta_t}} }{\sqrt{MT_t}} \log \sbr{\frac{16pMT_t}{\delta_t}} .
	\end{align*}
	Furthermore, for any phase $t>0$, if 
	\begin{align}
		%
		T_t = \Bigg \lceil & \frac{68 \cdot 192^2 \cdot 8^2 \sbr{1+\zeta}^3 (\rho^E)^2 k^4 L_{x}^4 L_{\theta}^2 L_{w}^2}{c_0^2 M} \cdot \max\lbr{2^{2t},\ \frac{L_x^4}{\omega^2}} \cdot \log^2 \sbr{\frac{16p}{\delta_t}} \nonumber\\& \log^2 \sbr{ \frac{192 \cdot 16 \cdot 8  \sbr{1+\zeta}^{\frac{3}{2}} \rho^E k^2 p L_{x}^2 L_{\theta} L_{w}}{c_0} \cdot \max\lbr{2^t,\ \frac{L_x^2}{\omega}} \cdot \frac{1}{\delta_t} \cdot \log\sbr{\frac{16p}{\delta_t}} } \Bigg \rceil , \label{eq:value_T_t}
	\end{align}
	then
	\begin{align*}
		\nbr{\hat{\bs{B}}_{t,\bot}^\top \bs{B}} \leq \min\lbr{\frac{1}{8 k L_{x} L_{w} \cdot 2^t \sqrt{1+\zeta}},\ \frac{\omega}{6 L_x^2}} .
	\end{align*}
\end{lemma}
\begin{proof}[Proof of Lemma~\ref{lemma:concentration_B_hat_t}]
	From Assumption~\ref{assumption:diverse_task}, $\sigma_{k}(\ex[\bs{Z}_t]) - \sigma_{k+1}(\ex[\bs{Z}_t])=\sigma_{\min}( \frac{1}{M} \sum_{m=1}^{M} \bs{\theta}_m \bs{\theta}_m^\top ) \geq \frac{c_0}{k}$. Using the Davis-Kahan sin $\theta$ Theorem~\cite{bhatia2013matrix} and letting $T_t$ be large enough to satisfy $\nbr{\bs{Z}_t - \ex[\bs{Z}_t]} \leq \frac{c_0}{2k}$, we have
	\begin{align*}
		\nbr{\hat{\bs{B}}_{t,\bot}^\top \bs{B}} \leq & \frac{ \nbr{\bs{Z}_t - \ex[\bs{Z}_t]} }{ \sigma_{k}(\ex[\bs{Z}_t]) - \sigma_{k+1}(\ex[\bs{Z}_t]) - \nbr{\bs{Z}_t - \ex[\bs{Z}_t]} }
		\\
		\leq & \frac{2k}{c_0} \nbr{\bs{Z}_t - \ex[\bs{Z}_t]} 
		\\
		\overset{\textup{(a)}}{\leq} & \frac{ 192 \nbr{\bs{X}_{\batch}^{+}}^2 k p L_{x} L_{\theta} \log\sbr{\frac{16p}{\delta_t}} }{c_0 \sqrt{MT_t}} \log \sbr{\frac{16pMT_t}{\delta_t}} ,.
	\end{align*}
	where inequality (a) uses the definition of event $\cE$.
	
	Using Lemma~\ref{lemma:technical_tool_bai_stage2} with $A= \frac{192 \nbr{\bs{X}_{\batch}^{+}}^2 k p L_{x} L_{\theta}}{c_0} \log\sbr{\frac{16p}{\delta_t}} $, $B=\frac{16p}{\delta_t}$ and $\kappa=\min\{\frac{1}{8 k L_{x} L_{w} \cdot 2^t \sqrt{1+\zeta}}, \frac{\omega}{6 L_x^2}\}$, we have that if 
	\begin{align*}
		M T_t \geq & 68 \sbr{ \frac{192 \nbr{\bs{X}_{\batch}^{+}}^2 k p L_{x} L_{\theta}}{c_0} \log\sbr{\frac{16p}{\delta_t}}}^2 \cdot \max\lbr{ \sbr{8 k L_{x} L_{w} \cdot 2^t \sqrt{1+\zeta} }^2 ,\ \frac{6^2 L_x^4}{\omega^2} } \cdot \\& \log^2 \sbr{ \frac{192 \nbr{\bs{X}_{\batch}^{+}}^2 k p L_{x} L_{\theta}}{c_0} \log\sbr{\frac{16p}{\delta_t}} \cdot \frac{16p}{\delta_t} \cdot \max\lbr{ 8 k L_{x} L_{w} \cdot 2^t \sqrt{1+\zeta} ,\ \frac{6 L_x^2}{\omega} } } ,
	\end{align*}
	then $\nbr{\hat{\bs{B}}_{t,\bot}^\top \bs{B}} \leq \min\lbr{\frac{1}{8 k L_{x} L_{w} \cdot 2^t \sqrt{1+\zeta}},\ \frac{\omega}{6 L_x^2}}$.
	
	According to Lemma~\ref{lemma:X_batch_ub}, we have $\|\bs{X}_{\batch}^{+}\| \leq \sqrt{\frac{(1+\zeta) \rho^{E}}{p}}$.
	
	Then, further enlarging $M T_t$, we have that if
	\begin{align*}
		M T_t \geq & \frac{68 \cdot 192^2 \cdot 8^2 \sbr{1+\zeta}^3 (\rho^E)^2 k^4 L_{x}^4 L_{\theta}^2 L_{w}^2}{c_0^2} \cdot \max\lbr{2^{2t},\ \frac{L_x^4}{\omega^2}} \cdot \log^2 \sbr{\frac{16p}{\delta_t}} \\& \log^2 \sbr{ \frac{192 \cdot 16 \cdot 8  \sbr{1+\zeta}^{\frac{3}{2}} \rho^E k^2 p L_{x}^2 L_{\theta} L_{w}}{c_0} \cdot \max\lbr{2^t,\ \frac{L_x^2}{\omega}} \cdot \frac{1}{\delta_t} \cdot \log\sbr{\frac{16p}{\delta_t}} } ,
	\end{align*}
	then
	\begin{align*}
		\nbr{\hat{\bs{B}}_{t,\bot}^\top \bs{B}} \leq \min\lbr{\frac{1}{8 k L_{x} L_{w} \cdot 2^t \sqrt{1+\zeta}},\ \frac{\omega}{6 L_x^2}} .
	\end{align*}
\end{proof}


\subsection{Elimination with Low-dimensional Representations}

For clarity of notation, we also add subscript $t$ to the notations in subroutine $\algelimlowrep$ to denote the quantities generated in phase $t$. Specifically, we use the notations $\hat{\bs{B}}_t$, $\hat{\cX}_{t,m}$, $\bs{\lambda}^{G}_{t,m}$, $\rho^{G}_{t,m}$, $N_{t,m}$, $\{\bs{z}_{t,m,i}\}_{i \in [N_{t,m}]}$, $\{r_{t,m,i}\}_{i \in [N_{t,m}]}$, $\hat{\bs{w}}_{t,m}$ and $\hat{\bs{\theta}}_{t,m}$ to denote the corresponding quantities used in $\algelimlowrep$ in phase $t$.

Before analyzing the sample complexity of $\algelimlowrep$, we first prove that there exists a sample allocation $\bs{\lambda} \in \triangle_{\cX}$ such that $\sum_{i=1}^{n} \lambda(\bs{x}_i) \hat{\bs{B}}_t^\top \bs{x}_i \bs{x}_i^\top \hat{\bs{B}}_t$ is invertible, i.e., the G-optimal design optimization with $\hat{\bs{B}}_t$ is non-vacuous (Line~\ref{line:bai_G_optimal_design} in Algorithm~\ref{alg:elim_low_rep}).


For any task $m \in [M]$, let
\begin{align*}
	\bs{\lambda}^*_m := & \argmin_{\bs{\lambda} \in \triangle_{\cX}} \max_{\bs{x} \in \cX \setminus \{\bs{x}^{*}_{m}\}} \frac{\| \bs{B}^\top \bs{x}^{*}_{m} - \bs{B}^\top \bs{x} \|^2_{\sbr{ \sum_{i=1}^{n} \lambda(\bs{x}_i) \bs{B}^\top \bs{x}_i \bs{x}_i^\top \bs{B} }^{-1} } }{\sbr{ ({\bs{x}^{*}_{m}} - \bs{x})^\top \bs{\theta}_m }^2} .
\end{align*}
$\bs{\lambda}^*_m$ is the optimal solution of the G-optimal design optimization with true feature extractor $\bs{B}$.

\begin{lemma} \label{lemma:invertible_under_hat_B}
	For any phase $t>0$ and task $m \in [M]$, if $\|\hat{\bs{B}}_t^\top \bs{B}_{\bot}\| \leq \frac{\omega}{6 L_x^2}$, we have
	\begin{align*}
		\sigma_{\min}\sbr{\sum_{i=1}^{n} \lambda^{*}_m(\bs{x}_i) \hat{\bs{B}}_t^\top \bs{x}_i \bs{x}_i^\top \hat{\bs{B}}_t} > 0 .
	\end{align*}
\end{lemma}
\begin{proof}[Proof of Lemma~\ref{lemma:invertible_under_hat_B}]
	For any task $m \in [M]$, let $\bs{A}_m:=\sum_{i=1}^{n} \lambda^{*}_m(\bs{x}_i) \bs{x}_i \bs{x}_i^\top$.
	Then, for any phase $t>0$ and task $m \in [M]$, we have
	\begin{align*}
		\sum_{i=1}^{n} \lambda^{*}_m(\bs{x}_i) \hat{\bs{B}}_t^\top \bs{x}_i \bs{x}_i^\top \hat{\bs{B}}_t = & \hat{\bs{B}}_t^\top \bs{A}_m \hat{\bs{B}}_t
		\\
		= & \hat{\bs{B}}_t^\top \sbr{ \bs{B} \bs{B}^\top + \bs{B}_{\bot} \bs{B}_{\bot}^\top } \bs{A}_m \sbr{ \bs{B} \bs{B}^\top + \bs{B}_{\bot} \bs{B}_{\bot}^\top } \hat{\bs{B}}_t
		\\
		= & \hat{\bs{B}}_t^\top \bs{B} \bs{B}^\top \bs{A}_m \bs{B} \bs{B}^\top \hat{\bs{B}}_t + \hat{\bs{B}}_t^\top \bs{B} \bs{B}^\top \bs{A}_m \bs{B}_{\bot} \bs{B}_{\bot}^\top \hat{\bs{B}}_t 
		\\& + \hat{\bs{B}}_t^\top \bs{B}_{\bot} \bs{B}_{\bot}^\top \bs{A}_m \bs{B} \bs{B}^\top \hat{\bs{B}}_t + \hat{\bs{B}}_t^\top \bs{B}_{\bot} \bs{B}_{\bot}^\top \bs{A}_m \bs{B}_{\bot} \bs{B}_{\bot}^\top \hat{\bs{B}}_t .
	\end{align*}
	Hence, we have
	\begin{align*}
		\sigma_{\min}\sbr{\sum_{i=1}^{n} \lambda^{*}_m(\bs{x}_i) \hat{\bs{B}}_t^\top \bs{x}_i \bs{x}_i^\top \hat{\bs{B}}_t} \geq & \sigma_{\min}\sbr{ \hat{\bs{B}}_t^\top \bs{B} \bs{B}^\top \bs{A}_m \bs{B} \bs{B}^\top \hat{\bs{B}}_t } - \sigma_{\max}\sbr{ \hat{\bs{B}}_t^\top \bs{B} \bs{B}^\top \bs{A}_m \bs{B}_{\bot} \bs{B}_{\bot}^\top \hat{\bs{B}}_t }
		\\& - \sigma_{\max}\sbr{ \hat{\bs{B}}_t^\top \bs{B}_{\bot} \bs{B}_{\bot}^\top \bs{A}_m \bs{B} \bs{B}^\top \hat{\bs{B}}_t } - \sigma_{\max}\sbr{ \hat{\bs{B}}_t^\top \bs{B}_{\bot} \bs{B}_{\bot}^\top \bs{A}_m \bs{B}_{\bot} \bs{B}_{\bot}^\top \hat{\bs{B}}_t }
		\\
		\geq & \sigma_{\min}\sbr{ \hat{\bs{B}}_t^\top \bs{B} } \sigma_{\min}\sbr{ \bs{B}^\top \bs{A}_m \bs{B} } \sigma_{\min}\sbr{ \bs{B}^\top \hat{\bs{B}}_t } -  \nbr{ \bs{B}_{\bot}^\top \hat{\bs{B}}_t } \nbr{\bs{A}_m}
		\\& - \nbr{ \hat{\bs{B}}_t^\top \bs{B}_{\bot} } \nbr{\bs{A}_m} - \nbr{ \hat{\bs{B}}_t^\top \bs{B}_{\bot} } \nbr{\bs{A}_m}
		\\
		\geq & \sigma^2_{\min}\sbr{ \hat{\bs{B}}_t^\top \bs{B} } \sigma_{\min}\sbr{ \bs{B}^\top \bs{A}_m \bs{B} } - 3\nbr{ \hat{\bs{B}}_t^\top \bs{B}_{\bot} } L_x^2
		\\
		\overset{\textup{(a)}}{\geq} & \sbr{ 1 - \nbr{ \hat{\bs{B}}_t^\top \bs{B}_{\bot} }^2 } \omega - 3\nbr{ \hat{\bs{B}}_t^\top \bs{B}_{\bot} } L_x^2 ,
	\end{align*}
	where inequality (a) uses the fact that $\hat{\bs{B}}_t^\top \bs{B} \bs{B}^\top \hat{\bs{B}}_t + \hat{\bs{B}}_t^\top \bs{B}_{\bot} \bs{B}_{\bot}^\top \hat{\bs{B}}_t = \hat{\bs{B}}_t^\top (B \bs{B}^\top + \bs{B}_{\bot} \bs{B}_{\bot}^\top) \hat{\bs{B}}_t =\hat{\bs{B}}_t^\top \hat{\bs{B}}_t=\bs{I}_k$, and thus, $\sigma^2_{\min}( \hat{\bs{B}}_t^\top \bs{B} )=1 - \| \hat{\bs{B}}_t^\top \bs{B}_{\bot} \|^2$.

	Let $\|\hat{\bs{B}}_t^\top \bs{B}_{\bot}\| \leq \frac{\omega}{6 L_x^2}$. Then, we have
	\begin{align*}
		\sigma_{\min}\sbr{\sum_{i=1}^{n} \lambda^{*}_m(\bs{x}_i) \hat{\bs{B}}_t^\top \bs{x}_i \bs{x}_i^\top \hat{\bs{B}}_t} \geq & \sbr{ 1 - \frac{\omega^2}{36 L_x^4} } \omega - \frac{\omega}{2} 
		\\
		= & \frac{\omega}{2} - \frac{\omega^3}{36 L_x^4} 
		\\
		> & 0 ,
	\end{align*}
	where the last inequality is due to $\omega \leq L_x^2 < \sqrt{18} L_x^2$.
\end{proof}

Next, we bound the optimal value $\rho^{G}_{t,m}$ of the G-optimal design optimization with the estimated feature extractor $\hat{\bs{B}}_t$.

For any $\cZ \subseteq \cX$, let $\cY(\cZ):=\{ \bs{x}-\bs{x}':\ \forall \bs{x},\bs{x}' \in \cZ,\ \bs{x} \neq \bs{x}' \}$.
Recall that in Line~\ref{line:bai_G_optimal_design} of Algorithm~\ref{alg:elim_low_rep}, for any phase $t>0$ and task $m \in [M]$, 
$$
\rho^{G}_{t,m}:=\min_{\bs{\lambda} \in \triangle_{\cX}} \max_{\bs{y} \in \cY(\hat{\cX}_{t,m})} \| \hat{\bs{B}}_t^\top \bs{y} \|^2_{\sbr{\sum_{i=1}^n \lambda(\bs{x}_i) \hat{\bs{B}}_t^\top \bs{x}_i \bs{x}_i^\top \hat{\bs{B}}_t}^{-1}} .
$$

\begin{lemma} \label{lemma:rho_t_leq_4k}
	For any phase $t>0$ and task $m \in [M]$, 
	\begin{align*}
		\rho^{G}_{t,m} \leq 4k .
	\end{align*}
\end{lemma}

\begin{proof}[Proof of Lemma~\ref{lemma:rho_t_leq_4k}]
	For any phase $t>0$ and task $m \in [M]$, we have that $\hat{\cX}_{t,m} \subseteq \cX$ and $ \cY(\hat{\cX}_{t,m}) \subseteq \cY(\cX)$.
	
	For any fixed $\bs{\lambda} \in \triangle_{\cX}$, 
	\begin{align*}
		\max_{\bs{y} \in \cY(\hat{\cX}_{t,m})} \| \hat{\bs{B}}_t^\top \bs{y} \|^2_{\sbr{\sum_{i=1}^n \lambda(\bs{x}_i) \hat{\bs{B}}_t^\top \bs{x}_i \bs{x}_i^\top \hat{\bs{B}}_t}^{-1}} \leq & \max_{\bs{y} \in \cY(\cX)} \| \hat{\bs{B}}_t^\top \bs{y} \|^2_{\sbr{\sum_{i=1}^n \lambda(\bs{x}_i) \hat{\bs{B}}_t^\top \bs{x}_i \bs{x}_i^\top \hat{\bs{B}}_t}^{-1}}
		\\
		= & \| \hat{\bs{B}}_t^\top (\bs{x}'_1-\bs{x}'_2) \|^2_{\sbr{\sum_{i=1}^n \lambda(\bs{x}_i) \hat{\bs{B}}_t^\top \bs{x}_i \bs{x}_i^\top \hat{\bs{B}}_t}^{-1}}
		\\
		\leq & \sbr{\| \hat{\bs{B}}_t^\top \bs{x}'_1 \|_{\sbr{\sum_{i=1}^n \lambda(\bs{x}_i) \hat{\bs{B}}_t^\top \bs{x}_i \bs{x}_i^\top \hat{\bs{B}}_t}^{-1}}+\| \hat{\bs{B}}_t^\top \bs{x}'_2 \|_{\sbr{\sum_{i=1}^n \lambda(\bs{x}_i) \hat{\bs{B}}_t^\top \bs{x}_i \bs{x}_i^\top \hat{\bs{B}}_t}^{-1}}}^2
		\\
		\leq & 2\| \hat{\bs{B}}_t^\top \bs{x}'_1 \|^2_{\sbr{\sum_{i=1}^n \lambda(\bs{x}_i) \hat{\bs{B}}_t^\top \bs{x}_i \bs{x}_i^\top \hat{\bs{B}}_t}^{-1}} + 2\| \hat{\bs{B}}_t^\top \bs{x}'_2 \|^2_{\sbr{\sum_{i=1}^n \lambda(\bs{x}_i) \hat{\bs{B}}_t^\top \bs{x}_i \bs{x}_i^\top \hat{\bs{B}}_t}^{-1}}
		\\
		\leq & 4 \max_{\bs{x} \in \cX} \| \hat{\bs{B}}_t^\top \bs{x} \|^2_{\sbr{\sum_{i=1}^n \lambda(\bs{x}_i) \hat{\bs{B}}_t^\top \bs{x}_i \bs{x}_i^\top \hat{\bs{B}}_t}^{-1}} , 
	\end{align*}
	where $\bs{x}'_1$ and $\bs{x}'_2$ are the arms which satisfy that $\bs{y}=\bs{x}'_1-\bs{x}'_2$ achieves the maximum value $\max_{\bs{y} \in \cY(\cX)} \| \hat{\bs{B}}_t^\top \bs{y} \|^2_{\sbr{\sum_{i=1}^n \lambda(\bs{x}_i) \hat{\bs{B}}_t^\top \bs{x}_i \bs{x}_i^\top \hat{\bs{B}}_t}^{-1}}$.
	
	Since $\hat{\bs{B}}_t^\top \bs{x} \in \R^k$, according to the Equivalence Theorem in \cite{kiefer1960equivalence}, we have 
	\begin{align*}
		\min_{\bs{\lambda} \in \triangle_{\cX}} \max_{\bs{x} \in \cX} \| \hat{\bs{B}}_t^\top \bs{x} \|^2_{\sbr{\sum_{i=1}^n \lambda(\bs{x}_i) \hat{\bs{B}}_t^\top \bs{x}_i \bs{x}_i^\top \hat{\bs{B}}_t}^{-1}} = k .
	\end{align*}
	
	Therefore, we have
	\begin{align*}
		4k = & 4 \min_{\bs{\lambda} \in \triangle_{\cX}} \max_{\bs{x} \in \cX} \| \hat{\bs{B}}_t^\top \bs{x} \|^2_{\sbr{\sum_{i=1}^n \lambda(\bs{x}_i) \hat{\bs{B}}_t^\top \bs{x}_i \bs{x}_i^\top \hat{\bs{B}}_t}^{-1}} \\
		= & 4 \max_{\bs{x} \in \cX} \| \hat{\bs{B}}_t^\top \bs{x} \|^2_{\sbr{\sum_{i=1}^n \lambda'(\bs{x}_i) \hat{\bs{B}}_t^\top \bs{x}_i \bs{x}_i^\top \hat{\bs{B}}_t}^{-1}}
		\\
		\geq & \max_{\bs{y} \in \cY(\hat{\cX}_{t,m})} \| \hat{\bs{B}}_t^\top \bs{y} \|^2_{\sbr{\sum_{i=1}^n \lambda'(\bs{x}_i) \hat{\bs{B}}_t^\top \bs{x}_i \bs{x}_i^\top \hat{\bs{B}}_t}^{-1}}
		\\
		\geq & \min_{\bs{\lambda} \in \triangle_{\cX}} \max_{\bs{y} \in \cY(\hat{\cX}_{t,m})} \| \hat{\bs{B}}_t^\top \bs{y} \|^2_{\sbr{\sum_{i=1}^n \lambda(\bs{x}_i) \hat{\bs{B}}_t^\top \bs{x}_i \bs{x}_i^\top \hat{\bs{B}}_t}^{-1}}
		\\
		= & \rho^{G}_{t,m} ,
	\end{align*}
	where $\bs{\lambda}':=\argmin_{\bs{\lambda} \in \triangle_{\cX}} \max_{\bs{x} \in \cX} \| \hat{\bs{B}}_t^\top \bs{x} \|^2_{\sbr{\sum_{i=1}^n \lambda(\bs{x}_i) \hat{\bs{B}}_t^\top \bs{x}_i \bs{x}_i^\top \hat{\bs{B}}_t}^{-1}}$.
\end{proof}

Now we analyze the estimation error of the estimated reward parameter $\hat{\bs{\theta}}_{t,m}=\hat{\bs{B}}_t \hat{\bs{w}}_{t,m}$ in $\algelimlowrep$.

For any phase $t>0$, task $m \in [M]$ and arm $j \in [N_{t,m}]$, let $\xi_{t,m,j}$ denote the noise of the sample on arm $\bs{z}_{t,m,j}$ for task $m$, during the execution of $\algelimlowrep$ in phase $t$ (Line~\ref{line:bai_stage3_sample} in Algorithm~\ref{alg:elim_low_rep}).

For any phase $t>0$, define events
\begin{align}
	\cF_t := \Bigg \{ &
	\bs{y}^\top \hat{\bs{B}}_t  \sbr{\sum_{j=1}^{N_{t,m}} \hat{\bs{B}}_t^\top \bs{z}_{t,m,j} {\bs{z}_{t,m,j}}^\top \hat{\bs{B}}_t}^{-1} \sum_{j=1}^{N_{t,m}} \hat{\bs{B}}_t^\top \bs{z}_{t,m,j} \cdot \xi_{t,m,j} 
	\nonumber\\
	& \leq \nbr{\hat{\bs{B}}_t^\top \bs{y}}_{\sbr{\sum_{j=1}^{N_{t,m}} \hat{\bs{B}}_t^\top \bs{z}_{t,m,j} {\bs{z}_{t,m,j}}^\top \hat{\bs{B}}_t}^{-1}}  \sqrt{2\log \sbr{\frac{4n^2 M}{\delta_t}}}, \ \forall m \in [M] ,\ \forall \bs{y} \in \cY(\hat{\cX}_{t,m}) \Bigg\}, \label{eq:definition_cF}
\end{align}
and 
\begin{align*}
	\cF := \cap_{t=1}^{\infty} \cF_t .
\end{align*}

\begin{lemma}[Concentration of the Variance Term]\label{lemma:concentration_variance_bai}
	It holds that
	\begin{align*}
		\Pr \mbr{\cF} \geq 1-\frac{\delta}{2} .  
	\end{align*}
\end{lemma}

\begin{proof}[Proof of Lemma~\ref{lemma:concentration_variance_bai}]
	Let $\bs{\Sigma}_{t,m}:=\sum_{j=1}^{N_{t,m}} \hat{\bs{B}}_t^\top \bs{z}_{t,m,j} {\bs{z}_{t,m,j}}^\top \hat{\bs{B}}_t$.
	Then, we can write
	\begin{align*}
		\bs{y}^\top \hat{\bs{B}}_t  \sbr{\sum_{j=1}^{N_{t,m}} \hat{\bs{B}}_t^\top \bs{z}_{t,m,j} {\bs{z}_{t,m,j}}^\top \hat{\bs{B}}_t}^{-1} \sum_{j=1}^{N_{t,m}} \hat{\bs{B}}_t^\top \bs{z}_{t,m,j} \cdot \xi_{t,m,j} =\sum_{j=1}^{N_{t,m}} \bs{y}^\top \hat{\bs{B}}_t  \bs{\Sigma}_{t,m}^{-1}  \hat{\bs{B}}_t^\top \bs{z}_{t,m,j} \cdot \xi_{t,m,j} .
	\end{align*}
	
	For any phase $t>0$, task $m \in [M]$ and arm $j \in [N_{t,m}]$, $\hat{\bs{B}}_t$, $\bs{\Sigma}_{t,m}$ and $\{\bs{z}_{t,m,j}\}_{j=1}^{N_{t,m}}$ are fixed before the sampling in $\algelimlowrep$, and the noise $\xi_{t,m,j}$ is 1-sub-Gaussian (Line~\ref{line:bai_stage3_sample} in Algorithm~\ref{alg:elim_low_rep}).
	Thus, we have that for any $t>0$, $m \in [M]$ and $j \in [N_{t,m}]$, $\bs{y}^\top \hat{\bs{B}}_t  \bs{\Sigma}_{t,m}^{-1} \hat{\bs{B}}_t^\top \bs{z}_{t,m,j} \cdot \xi_{t,m,j} $ is $(\bs{y}^\top \hat{\bs{B}}_t  \bs{\Sigma}_{t,m}^{-1} \hat{\bs{B}}_t^\top \bs{z}_{t,m,j})$-sub-Gaussian.
	
	Using Hoeffding's inequality and taking a union bound over all $m \in [M]$ and $\bs{y} \in \cY(\hat{\cX}_{t,m})$, we have that with probability at least $1-\frac{\delta_t}{2}$,
	\begin{align*}
		& \sum_{j=1}^{N_{t,m}} \bs{y}^\top \hat{\bs{B}}_t  \bs{\Sigma}_{t,m}^{-1}  \hat{\bs{B}}_t^\top \bs{z}_{t,m,j} \cdot \xi_{t,m,j}
		\\
		\leq & \sqrt{2 \sum_{j=1}^{N_{t,m}} \sbr{\bs{y}^\top \hat{\bs{B}}_t  \bs{\Sigma}_{t,m}^{-1} \hat{\bs{B}}_t^\top \bs{z}_{t,m,j}}^2 \cdot \log \sbr{\frac{4n^2 M}{\delta_t}}} 
		\\
		= & \sqrt{2 \sum_{j=1}^{N_{t,m}} \bs{y}^\top \hat{\bs{B}}_t  \bs{\Sigma}_{t,m}^{-1} \hat{\bs{B}}_t^\top \bs{z}_{t,m,j} \cdot {\bs{z}_{t,m,j}}^\top \hat{\bs{B}}_t  \bs{\Sigma}_{t,m}^{-1} \hat{\bs{B}}_t^\top \bs{y} \cdot \log \sbr{\frac{4n^2 M}{\delta_t}}} 
		\\
		= & \sqrt{2 \bs{y}^\top \hat{\bs{B}}_t \bs{\Sigma}_{t,m}^{-1} \sbr{\hat{\bs{B}}_t^\top \sum_{j=1}^{N_{t,m}} \bs{z}_{t,m,j} \cdot {\bs{z}_{t,m,j}}^\top \hat{\bs{B}}_t}  \bs{\Sigma}_{t,m}^{-1} \hat{\bs{B}}_t^\top \bs{y} \cdot \log \sbr{\frac{4n^2 M}{\delta_t}}} 
		\\
		= & \sqrt{2 \bs{y}^\top \hat{\bs{B}}_t \bs{\Sigma}_{t,m}^{-1} \hat{\bs{B}}_t^\top \bs{y} \cdot \log \sbr{\frac{4n^2 M}{\delta_t}}} 
		\\
		= & \nbr{\hat{\bs{B}}_t^\top \bs{y}}_{\bs{\Sigma}_{t,m}^{-1}} \sqrt{2 \log \sbr{\frac{4n^2 M}{\delta_t}}} ,
	\end{align*}
	which implies that
	\begin{align*}
		\Pr \mbr{ \cF_t } \geq 1 - \frac{\delta_t}{2} .
	\end{align*}
	
	Taking a union bound over all phases $t \geq 1$ and recalling $\delta_t:=\frac{\delta}{2t^2}$, we obtain
	\begin{align*}
		\Pr \mbr{\cF} 
		\geq & 1- \sum_{t=1}^{\infty} \Pr \mbr{\bar{\cF_t}}
		\\
		\geq & 1- \sum_{t=1}^{\infty} \frac{\delta_t}{2}
		\\
		= & 1- \sum_{t=1}^{\infty} \frac{\delta}{4t^2}
		\\
		\geq & 1-\frac{\delta}{2} .
	\end{align*}
\end{proof}

\begin{lemma}[Concentration of $\hat{\bs{\theta}}_{t,m}$] \label{lemma:bai_concentration_theta_t_m}
	Suppose that event $\cE \cap \cF$ holds. Then, for any phase $t>0$, task $m \in [M]$ and $\bs{y} \in \cY(\hat{\cX}_{t,m})$,
	\begin{align*}
		\abr{\bs{y}^\top \sbr{\hat{\bs{\theta}}_{t,m}-\bs{\theta}_m}} \leq \frac{1}{2^t} .
	\end{align*}
\end{lemma}

\begin{proof}[Proof of Lemma~\ref{lemma:bai_concentration_theta_t_m}]
	
	For any phase $t>0$, task $m \in [M]$ and $\bs{y} \in \cY(\hat{\cX}_{t,m})$,
	\begin{align}
		\bs{y}^\top \sbr{\hat{\bs{\theta}}_{t,m}- \bs{\theta}_m} = & \bs{y}^\top \hat{\bs{B}}_t \hat{\bs{w}}_{t,m}-\bs{y}^\top \sbr{ \hat{\bs{B}}_t \hat{\bs{B}}_t^\top + \hat{\bs{B}}_{t,\bot} \hat{\bs{B}}_{t,\bot}^\top } \bs{\theta}_m
		\nonumber\\
		= & \bs{y}^\top \hat{\bs{B}}_t \sbr{\hat{\bs{w}}_{t,m} - \hat{\bs{B}}_t^\top \bs{\theta}_m } -\bs{y}^\top  \hat{\bs{B}}_{t,\bot} \hat{\bs{B}}_{t,\bot}^\top \bs{\theta}_m . \label{eq:y_theta_est_err_decompose}
	\end{align}

	Here, $\hat{\bs{w}}_{t,m}$ can be written as
	\begin{align}
		\hat{\bs{w}}_{t,m} = & \sbr{\sum_{j=1}^{N_{t,m}} \hat{\bs{B}}_t^\top \bs{z}_{t,m,j} {\bs{z}_{t,m,j}}^\top \hat{\bs{B}}_t}^{-1} \sum_{j=1}^{N_{t,m}} \hat{\bs{B}}_t^\top \bs{z}_{t,m,j} \cdot r_{t,m,j}
		\nonumber\\
		= & \sbr{\sum_{j=1}^{N_{t,m}} \hat{\bs{B}}_t^\top \bs{z}_{t,m,j} {\bs{z}_{t,m,j}}^\top \hat{\bs{B}}_t}^{-1} \sum_{j=1}^{N_{t,m}} \hat{\bs{B}}_t^\top \bs{z}_{t,m,j} \cdot \sbr{{\bs{z}_{t,m,j}}^\top \bs{\theta}_m + \xi_{t,m,j}}
		\nonumber\\
		= & \sbr{\sum_{j=1}^{N_{t,m}} \hat{\bs{B}}_t^\top \bs{z}_{t,m,j} {\bs{z}_{t,m,j}}^\top \hat{\bs{B}}_t}^{-1} \sum_{j=1}^{N_{t,m}} \hat{\bs{B}}_t^\top \bs{z}_{t,m,j} \cdot \sbr{{\bs{z}_{t,m,j}}^\top \sbr{ \hat{\bs{B}}_t \hat{\bs{B}}_t^\top + \hat{\bs{B}}_{t,\bot} \hat{\bs{B}}_{t,\bot}^\top } \bs{\theta}_m + \xi_{t,m,j}}
		\nonumber\\
		= & \hat{\bs{B}}_t^\top \bs{\theta}_m + \sbr{\sum_{j=1}^{N_{t,m}} \hat{\bs{B}}_t^\top \bs{z}_{t,m,j} {\bs{z}_{t,m,j}}^\top \hat{\bs{B}}_t}^{-1} \sum_{j=1}^{N_{t,m}} \hat{\bs{B}}_t^\top \bs{z}_{t,m,j} \cdot {\bs{z}_{t,m,j}}^\top \hat{\bs{B}}_{t,\bot} \hat{\bs{B}}_{t,\bot}^\top \bs{\theta}_m \nonumber\\& + \sbr{\sum_{j=1}^{N_{t,m}} \hat{\bs{B}}_t^\top \bs{z}_{t,m,j} {\bs{z}_{t,m,j}}^\top \hat{\bs{B}}_t}^{-1} \sum_{j=1}^{N_{t,m}} \hat{\bs{B}}_t^\top \bs{z}_{t,m,j} \cdot \xi_{t,m,j} . \label{eq:hat_w_decompose}
	\end{align}
	
	Plugging Eq.~\eqref{eq:hat_w_decompose} into Eq.~\eqref{eq:y_theta_est_err_decompose}, we can decompose the estimation error of $\hat{\bs{\theta}}_{t,m}$ in $\algelimlowrep$ into three parts as
	\begin{align*}
		\bs{y}^\top \sbr{\hat{\bs{\theta}}_{t,m}- \bs{\theta}_m} = & 
		\underbrace{\bs{y}^\top \hat{\bs{B}}_t \sbr{\sum_{j=1}^{N_{t,m}} \hat{\bs{B}}_t^\top \bs{z}_{t,m,j} {\bs{z}_{t,m,j}}^\top \hat{\bs{B}}_t}^{-1} \sum_{j=1}^{N_{t,m}} \hat{\bs{B}}_t^\top \bs{z}_{t,m,j} \cdot {\bs{z}_{t,m,j}}^\top \hat{\bs{B}}_{t,\bot} \hat{\bs{B}}_{t,\bot}^\top \bs{B} \bs{w}_m}_{\textup{Bias}}  \\& + \underbrace{\bs{y}^\top \hat{\bs{B}}_t \sbr{\sum_{j=1}^{N_{t,m}} \hat{\bs{B}}_t^\top \bs{z}_{t,m,j} {\bs{z}_{t,m,j}}^\top \hat{\bs{B}}_t}^{-1} \sum_{j=1}^{N_{t,m}} \hat{\bs{B}}_t^\top \bs{z}_{t,m,j} \cdot \xi_{t,m,j}}_{\textup{Variance}} - \underbrace{\bs{y}^\top \hat{\bs{B}}_{t,\bot} \hat{\bs{B}}_{t,\bot}^\top \bs{B} \bs{w}_m}_{\textup{Estimation error of $\hat{\bs{B}}_t$}} .
	\end{align*}

	Taking the absolute value on both sides, and using the Cauchy–Schwarz inequality and definition of event $\cF$ (Eq.~\eqref{eq:definition_cF}), we have
	\begin{align*}
		& \abr{\bs{y}^\top \hat{\bs{\theta}}_{t,m}-\bs{y}^\top \bs{\theta}_m} 
		\\
		\leq & \nbr{\hat{\bs{B}}_t^\top \bs{y}}_{\sbr{\sum_{j=1}^{N_{t,m}} \hat{\bs{B}}_t^\top \bs{z}_{t,m,j} {\bs{z}_{t,m,j}}^\top \hat{\bs{B}}_t}^{-1}} \cdot  \nbr{\sum_{j=1}^{N_{t,m}} \hat{\bs{B}}_t^\top \bs{z}_{t,m,j} \cdot {\bs{z}_{t,m,j}}^\top \hat{\bs{B}}_{t,\bot} \hat{\bs{B}}_{t,\bot}^\top \bs{B} \bs{w}_m}_{\sbr{\sum_{j=1}^{N_{t,m}} \hat{\bs{B}}_t^\top \bs{z}_{t,m,j} {\bs{z}_{t,m,j}}^\top \hat{\bs{B}}_t}^{-1}}  \\& + \nbr{\hat{\bs{B}}_t^\top \bs{y}}_{\sbr{\sum_{j=1}^{N_{t,m}} \hat{\bs{B}}_t^\top \bs{z}_{t,m,j} {\bs{z}_{t,m,j}}^\top \hat{\bs{B}}_t}^{-1}}  \sqrt{2\log \sbr{\frac{4n^2 M}{\delta_t}}} + \abr{\bs{y}^\top \hat{\bs{B}}_{t,\bot} \hat{\bs{B}}_{t,\bot}^\top \bs{B} \bs{w}_m}
		\\
		\overset{\textup{(a)}}{\leq} & \frac{ \sqrt{1+\zeta}  \nbr{\hat{\bs{B}}_t^\top \bs{y}}_{\sbr{\sum_{i=1}^{n} \lambda^{G}_{t,m}(\bs{x}_i) \cdot \hat{\bs{B}}_t^\top \bs{x}_i \bs{x}_i^\top \hat{\bs{B}}_t}^{-1}}}{\sqrt{N_{t,m}}} \cdot  L_{x} L_{w} \nbr{\hat{\bs{B}}_{t,\bot}^\top \bs{B}}   \cdot \sum_{j=1}^{N_{t,m}}  \nbr{\hat{\bs{B}}_t^\top \bs{z}_{t,m,j}}_{\sbr{\sum_{j=1}^{N_{t,m}} \hat{\bs{B}}_t^\top \bs{z}_{t,m,j} {\bs{z}_{t,m,j}}^\top \hat{\bs{B}}_t}^{-1}}  \\& + \frac{ \sqrt{1+\zeta} \nbr{\hat{\bs{B}}_t^\top \bs{y}}_{\sbr{\sum_{i=1}^{n} \lambda^{G}_{t,m}(\bs{x}_i) \cdot \hat{\bs{B}}_t^\top \bs{x}_i \bs{x}_i^\top \hat{\bs{B}}_t}^{-1}}}{\sqrt{N_{t,m}}} \cdot \sqrt{2\log \sbr{\frac{4n^2 M}{\delta_t}}}  +  2 L_{x} L_{w} \nbr{\hat{\bs{B}}_{t,\bot}^\top \bs{B}} 
		\\
		\overset{\textup{(b)}}{\leq} & \frac{ \sqrt{1+\zeta}  \nbr{\hat{\bs{B}}_t^\top \bs{y}}_{\sbr{\sum_{i=1}^{n} \lambda^{G}_{t,m}(\bs{x}_i) \cdot \hat{\bs{B}}_t^\top \bs{x}_i \bs{x}_i^\top \hat{\bs{B}}_t}^{-1}}}{\sqrt{N_{t,m}}} \cdot  L_{x} L_{w} \nbr{\hat{\bs{B}}_{t,\bot}^\top \bs{B}} \cdot \sqrt{k N_{t,m}}  \\& + \frac{ \sqrt{1+\zeta} \nbr{\hat{\bs{B}}_t^\top \bs{y}}_{\sbr{\sum_{i=1}^{n} \lambda^{G}_{t,m}(\bs{x}_i) \cdot \hat{\bs{B}}_t^\top \bs{x}_i \bs{x}_i^\top \hat{\bs{B}}_t}^{-1}}}{\sqrt{N_{t,m}}} \cdot \sqrt{2\log \sbr{\frac{4n^2 M}{\delta_t}}}  +  2 L_{x} L_{w} \nbr{\hat{\bs{B}}_{t,\bot}^\top \bs{B}} 
		\\
		\leq & \sqrt{1+\zeta} \nbr{\hat{\bs{B}}_t^\top \bs{y}}_{\sbr{\sum_{i=1}^{n} \lambda^{G}_{t,m}(\bs{x}_i) \cdot \hat{\bs{B}}_t^\top \bs{x}_i \bs{x}_i^\top \hat{\bs{B}}_t}^{-1}} \cdot  L_{x} L_{w} \nbr{\hat{\bs{B}}_{t,\bot}^\top \bs{B}} \cdot \sqrt{k}  \\& + \frac{ \sqrt{1+\zeta}  \nbr{\hat{\bs{B}}_t^\top \bs{y}}_{\sbr{\sum_{i=1}^{n} \lambda^{G}_{t,m}(\bs{x}_i) \cdot \hat{\bs{B}}_t^\top \bs{x}_i \bs{x}_i^\top \hat{\bs{B}}_t}^{-1}}}{\sqrt{N_{t,m}}} \cdot \sqrt{2\log \sbr{\frac{4n^2 M}{\delta_t}}} + 2 L_{x} L_{w} \nbr{\hat{\bs{B}}_{t,\bot}^\top \bs{B}} 
		\\
		\leq &  \sqrt{(1+\zeta) \cdot k \cdot \rho^{G}_{t,m}} \cdot  L_{x} L_{w} \nbr{\hat{\bs{B}}_{t,\bot}^\top \bs{B}}  + \frac{ \sqrt{(1+\zeta) \cdot \rho^{G}_{t,m} \cdot 2\log \sbr{\frac{4n^2 M}{\delta_t}} } }{\sqrt{N_{t,m}}} +  2 L_{x} L_{w} \nbr{\hat{\bs{B}}_{t,\bot}^\top \bs{B}}
		\\
		\overset{\textup{(c)}}{\leq} &  \sqrt{(1+\zeta) \cdot 4 k^2} \cdot  L_{x} L_{w} \nbr{\hat{\bs{B}}_{t,\bot}^\top \bs{B}}  + \frac{ \sqrt{(1+\zeta) \cdot \rho^{G}_{t,m} \cdot 2\log \sbr{\frac{4n^2 M}{\delta_t}} } }{\sqrt{N_{t,m}}} +  2 L_{x} L_{w} \nbr{\hat{\bs{B}}_{t,\bot}^\top \bs{B}}
		\\
		\overset{\textup{(d)}}{\leq} & \sqrt{(1+\zeta) \cdot 4 k^2} \cdot L_{x} L_{w} \cdot \frac{1}{8 k L_{x} L_{w} \cdot 2^t \sqrt{1+\zeta}}  + \frac{1}{4 \cdot 2^t} +  2 L_{x} L_{w} \cdot \frac{1}{8 k L_{x} L_{w}  \cdot 2^t \sqrt{1+\zeta}}
		\\
		\leq & \frac{1}{4 \cdot 2^t}  + \frac{1}{4 \cdot 2^t} +  \frac{1}{4 \cdot 2^t}
		\\
		\leq & \frac{1}{2^t} .
	\end{align*}
	Here inequality (a) is due to the guarantee of rounding procedure $\round$ and the triangle inequality. Inequality (b) uses Lemma~\ref{lemma:sqrt_n_k}, and inequality (c) follows from Lemma~\ref{lemma:rho_t_leq_4k}. Inequality (d) comes from Lemma~\ref{lemma:concentration_B_hat_t} and $N_{t,m}:=\max \{ \lceil  32 \cdot 2^{2t} (1+\zeta) \rho^{G}_{t,m} \log (\frac{4n^2 M}{\delta_t}) \rceil,\ \frac{180k}{\zeta^2} \}$.
\end{proof}


For any task $m \in [M]$ and arm $\bs{x} \in \cX$, let $\Delta_m(\bs{x}):=({\bs{x}^{*}_{m}} - \bs{x})^\top \bs{\theta}_m$ denote the reward gap between the optimal arm $\bs{x}^{*}_{m}$ and arm $\bs{x}$ in task $m$.
For any phase $t>0$ and task $m \in [M]$, let $\cZ_{t,m}:=\{ \bs{x} \in \cX : \Delta_{m}(\bs{x}) \leq 4 \cdot 2^{-t} \}$.

\begin{lemma} \label{lemma:cX_subseteq_cS}
	Suppose that event $\cE \cap \cF$ holds. For any phase $t>0$ and task $m \in [M]$,
	\begin{align*}
		\bs{x}^{*}_{m} \in \hat{\cX}_{t,m} ,
	\end{align*}
	and for any phase $t\geq2$ and task $m \in [M]$,
	\begin{align*}
		\hat{\cX}_{t,m} \subseteq \cZ_{t,m} .
	\end{align*}
\end{lemma}

\begin{proof}[Proof of Lemma~\ref{lemma:cX_subseteq_cS}]
	This proof follows a similar analytical procedure as that of Lemma~2 in \cite{fiez2019sequential}.
	
	First, we prove $\bs{x}^{*}_{m} \in \hat{\cX}_{t,m}$ for any phase $t>0$ and task $m \in [M]$ by contradiction.
	
	Suppose that for some $t>0$ and some $m \in [M]$, $\bs{x}^{*}_{m}$ is eliminated from $\hat{\cX}_{t,m}$ in phase $t$.
	Then, we have that there exists some $\bs{x}' \in \hat{\cX}_{t,m}$ such that
	\begin{align*}
		(\bs{x}'-\bs{x}^{*}_{m})^\top \hat{\bs{\theta}}_{t,m} > 2^{-t} .
	\end{align*}
	Then, we have
	\begin{align*}
		(\bs{x}'-\bs{x}^{*}_{m})^\top \bs{\theta}_m = & (\bs{x}'-\bs{x}^{*}_{m})^\top \hat{\bs{\theta}}_{t,m} - (\bs{x}'-\bs{x}^{*}_{m})^\top \sbr{\hat{\bs{\theta}}_{t,m} - \bs{\theta}_m} 
		\\
		\geq & (\bs{x}'-\bs{x}^{*}_{m})^\top \hat{\bs{\theta}}_{t,m} - 2^{-t}
		\\
		> & 2^{-t} - 2^{-t}
		\\
		= & 0 ,
	\end{align*}
	which contradicts the definition of $\bs{x}^{*}_{m}$. Thus, we obtain that $\bs{x}^{*}_{m} \in \hat{\cX}_{t,m}$ for any phase $t>0$ and task $m \in [M]$.
	
	Next, we prove $\hat{\cX}_{t,m} \subseteq \cZ_{t,m}$ for any phase $t\geq2$ and task $m \in [M]$, i.e., each $\bs{x} \in \hat{\cX}_{t,m}$ satisfies that $\Delta_m(\bs{x}) \leq 4 \cdot 2^{-t}$. 
	
	Suppose that there exists some phase $t$, some task $m$ and some $\bs{x} \in \hat{\cX}_{t,m}$ such that $\Delta_m(\bs{x}) > 4 \cdot 2^{-t}$. Then, in phase $t-1 \geq 1$, we have
	\begin{align*}
		(\bs{x}^{*}_{m}-\bs{x})^\top \hat{\bs{\theta}}_{t-1,m} = & (\bs{x}^{*}_{m}-\bs{x})^\top \bs{\theta}_m - (\bs{x}^{*}_{m}-\bs{x})^\top \sbr{\bs{\theta}_m-\hat{\bs{\theta}}_{t-1,m}} 
		\\
		\geq & (\bs{x}^{*}_{m}-\bs{x})^\top \bs{\theta}_m - 2^{-(t-1)}
		\\
		> & 4 \cdot 2^{-t} - 2^{-(t-1)}
		\\
		= & 2^{-(t-1)} , 
	\end{align*}
	which implies that $x$ should have been eliminated from $\hat{\cX}_{t,m}$ in phase $t-1$, and contradicts our supposition. Thus, we complete the proof.
\end{proof}


\subsection{Proof of Theorem~\ref{thm:bai_ub}}

Before proving Theorem~\ref{thm:bai_ub}, we first introduce a useful lemma.

For any task $m \in [M]$, let
\begin{align*}
	\bs{\lambda}^*_m := & \argmin_{\bs{\lambda} \in \triangle_{\cX}} \max_{\bs{x} \in \cX \setminus \{\bs{x}^{*}_{m}\}} \frac{\| \bs{B}^\top \bs{x}^{*}_{m} - \bs{B}^\top \bs{x} \|^2_{\sbr{ \sum_{i=1}^{n} \lambda(\bs{x}_i) \bs{B}^\top \bs{x}_i \bs{x}_i^\top \bs{B} }^{-1} } }{ \sbr{({\bs{x}^{*}_{m}} - \bs{x})^\top \bs{\theta}_m}^2 } ,
\end{align*}
and
\begin{align*}
	\rho^*_m := & \min_{\bs{\lambda} \in \triangle_{\cX}} \max_{\bs{x} \in \cX \setminus \{\bs{x}^{*}_{m}\}} \frac{\| \bs{B}^\top \bs{x}^{*}_{m} - \bs{B}^\top \bs{x} \|^2_{\sbr{ \sum_{i=1}^{n} \lambda(\bs{x}_i) \bs{B}^\top \bs{x}_i \bs{x}_i^\top \bs{B} }^{-1} } }{ \sbr{({\bs{x}^{*}_{m}} - \bs{x})^\top \bs{\theta}_m}^2 } .
\end{align*}
$\bs{\lambda}^*_m$ and $\rho^*_m$ are the optimal solution and the optimal value of the G-optimal design optimization with true feature extractor $\bs{B}$, respectively.

\begin{lemma}\label{lemma:connection_ub_lb}
	Suppose that event $\cE \cap \cF$ holds. For any task $m \in [M]$ and $\bs{y} \in \R^d$,
	\begin{align*}
		\| \hat{\bs{B}}_t^\top \bs{y} \|^2_{\sbr{ \sum_{i=1}^{n} \lambda^*_m(\bs{x}_i) \hat{\bs{B}}_t^\top \bs{x}_i \bs{x}_i^\top \hat{\bs{B}}_t }^{-1} } \leq \| \bs{B}^\top \bs{y} \|^2_{\sbr{ \sum_{i=1}^{n} \lambda^*_m(\bs{x}_i) \bs{B}^\top \bs{x}_i \bs{x}_i^\top \bs{B} }^{-1} } + \frac{ 11 L_{x}^4 }{ k \omega^2 \cdot 2^t } .
	\end{align*}
\end{lemma}

\begin{proof}[Proof of Lemma~\ref{lemma:connection_ub_lb}]
	
	We first handle the term $( \sum_{i=1}^{n} \lambda^*_m(\bs{x}_i) \hat{\bs{B}}_t^\top \bs{x}_i \bs{x}_i^\top \hat{\bs{B}}_t )^{-1}$. 
	
	For any task $m \in [M]$, we have
	\begin{align*}
		& \sum_{i=1}^{n} \lambda^*_m(\bs{x}_i) \hat{\bs{B}}_t^\top \bs{x}_i \bs{x}_i^\top \hat{\bs{B}}_t 
		\\
		= & \sum_{i=1}^{n} \lambda^*_m(\bs{x}_i) \sbr{\hat{\bs{B}}_t^\top \bs{B} \bs{B}^\top \bs{x}_i + \hat{\bs{B}}_t^\top \bs{B}_{\bot} \bs{B}_{\bot}^\top \bs{x}_i} \cdot \sbr{\hat{\bs{B}}_t^\top \bs{B} \bs{B}^\top \bs{x}_i + \hat{\bs{B}}_t^\top \bs{B}_{\bot} \bs{B}_{\bot}^\top \bs{x}_i}^\top
		\\
		= & \sum_{i=1}^{n} \lambda^*_m(\bs{x}_i) \Bigg( \sbr{\hat{\bs{B}}_t^\top \bs{B} \bs{B}^\top \bs{x}_i} \cdot \sbr{\hat{\bs{B}}_t^\top \bs{B} \bs{B}^\top \bs{x}_i}^\top + \sbr{\hat{\bs{B}}_t^\top \bs{B} \bs{B}^\top \bs{x}_i} \cdot \sbr{\hat{\bs{B}}_t^\top \bs{B}_{\bot} \bs{B}_{\bot}^\top \bs{x}_i}^\top 
		\\
		& + \sbr{\hat{\bs{B}}_t^\top \bs{B}_{\bot} \bs{B}_{\bot}^\top \bs{x}_i} \cdot \sbr{\hat{\bs{B}}_t^\top \bs{B} \bs{B}^\top \bs{x}_i}^\top + \sbr{\hat{\bs{B}}_t^\top \bs{B}_{\bot} \bs{B}_{\bot}^\top \bs{x}_i} \cdot \sbr{\hat{\bs{B}}_t^\top \bs{B}_{\bot} \bs{B}_{\bot}^\top \bs{x}_i}^\top \Bigg)
		\\
		= & \sum_{i=1}^{n} \lambda^*_m(\bs{x}_i) \sbr{\hat{\bs{B}}_t^\top \bs{B} \bs{B}^\top \bs{x}_i} \cdot \sbr{\hat{\bs{B}}_t^\top \bs{B} \bs{B}^\top \bs{x}_i}^\top
		+ \sum_{i=1}^{n} \lambda^*_m(\bs{x}_i) \Bigg( \sbr{\hat{\bs{B}}_t^\top \bs{B} \bs{B}^\top \bs{x}_i} \cdot \sbr{\hat{\bs{B}}_t^\top \bs{B}_{\bot} \bs{B}_{\bot}^\top \bs{x}_i}^\top 
		\\& + \sbr{\hat{\bs{B}}_t^\top \bs{B}_{\bot} \bs{B}_{\bot}^\top \bs{x}_i} \cdot \sbr{\hat{\bs{B}}_t^\top \bs{B} \bs{B}^\top \bs{x}_i}^\top + \sbr{\hat{\bs{B}}_t^\top \bs{B}_{\bot} \bs{B}_{\bot}^\top \bs{x}_i} \cdot \sbr{\hat{\bs{B}}_t^\top \bs{B}_{\bot} \bs{B}_{\bot}^\top \bs{x}_i}^\top \Bigg) .
	\end{align*}
	
	Let $\bs{P}_t:=\sum_{i=1}^{n} \lambda^*_m(\bs{x}_i) (\hat{\bs{B}}_t^\top \bs{B} \bs{B}^\top \bs{x}_i) \cdot (\hat{\bs{B}}_t^\top \bs{B} \bs{B}^\top \bs{x}_i)^\top$.
	Let $\bs{Q}_t:=\sum_{i=1}^{n} \lambda^*_m(\bs{x}_i) ( (\hat{\bs{B}}_t^\top \bs{B} \bs{B}^\top \bs{x}_i) \cdot (\hat{\bs{B}}_t^\top \bs{B}_{\bot} \bs{B}_{\bot}^\top \bs{x}_i)^\top 
	+ (\hat{\bs{B}}_t^\top \bs{B}_{\bot} \bs{B}_{\bot}^\top \bs{x}_i) \cdot (\hat{\bs{B}}_t^\top \bs{B} \bs{B}^\top \bs{x}_i)^\top + (\hat{\bs{B}}_t^\top \bs{B}_{\bot} \bs{B}_{\bot}^\top \bs{x}_i) \cdot (\hat{\bs{B}}_t^\top \bs{B}_{\bot} \bs{B}_{\bot}^\top \bs{x}_i)^\top )$.
	Then, we have $\sum_{i=1}^{n} \lambda^*_m(\bs{x}_i) \hat{\bs{B}}_t^\top \bs{x}_i \bs{x}_i^\top \hat{\bs{B}}_t = \bs{P}_t+\bs{Q}_t$.
	
	From Assumption~\ref{assumption:lambda^*_B_x_x_B_invertible}, we have that for any task $m \in [M]$, $\sum_{i=1}^{n} \lambda^*_m(\bs{x}_i) \bs{B}^\top \bs{x}_i \bs{x}_i^\top \bs{B}$ is invertible.
	Since $\hat{\bs{B}}_t^\top \bs{B}$ is also invertible, we have that $\bs{P}_t$ is invertible. 
	According to Lemmas~\ref{lemma:concentration_B_hat_t} and \ref{lemma:invertible_under_hat_B}, we have that $\sum_{i=1}^{n} \lambda^*_m(\bs{x}_i) \hat{\bs{B}}_t^\top \bs{x}_i \bs{x}_i^\top \hat{\bs{B}}_t$ is also invertible.
	Thus, we can write $(\sum_{i=1}^{n} \lambda^*_m(\bs{x}_i) \hat{\bs{B}}_t^\top \bs{x}_i \bs{x}_i^\top \hat{\bs{B}}_t)^{-1}$ as follows.
	\begin{align*}
		\sbr{ \sum_{i=1}^{n} \lambda^*_m(\bs{x}_i) \hat{\bs{B}}_t^\top \bs{x}_i \bs{x}_i^\top \hat{\bs{B}}_t }^{-1} = & \bs{P}_t^{-1}-\sbr{\bs{P}_t + \bs{Q}_t}^{-1} \bs{Q}_t \bs{P}_t^{-1} 
	\end{align*}

	Hence, for any task $m \in [M]$ and $\bs{y} \in \R^d$, we have
	\begin{align}
		\|\hat{\bs{B}}_t^\top \bs{y}\|^2_{\sbr{ \sum_{i=1}^{n} \lambda^*_m(\bs{x}_i) \hat{\bs{B}}_t^\top \bs{x}_i \bs{x}_i^\top \hat{\bs{B}}_t }^{-1}} = &
		\sbr{\hat{\bs{B}}_t^\top \bs{y}}^\top \sbr{ \sum_{i=1}^{n} \lambda^*_m(\bs{x}_i) \hat{\bs{B}}_t^\top \bs{x}_i \bs{x}_i^\top \hat{\bs{B}}_t }^{-1} \hat{\bs{B}}_t^\top \bs{y} 
		\nonumber\\
		= & \underbrace{\sbr{\hat{\bs{B}}_t^\top \bs{y}}^\top \bs{P}_t^{-1} \hat{\bs{B}}_t^\top \bs{y}}_{\textup{Term 1}} - \underbrace{\sbr{\hat{\bs{B}}_t^\top \bs{y}}^\top \sbr{\bs{P}_t + \bs{Q}_t}^{-1} \bs{Q}_t \bs{P}_t^{-1} \hat{\bs{B}}_t^\top \bs{y}}_{\textup{Term 2}} . \label{eq:norm_B_hat_y_decompose}
	\end{align}

	From Lemma~\ref{lemma:concentration_B_hat_t}, we have
	\begin{align*}
		\nbr{\hat{\bs{B}}_{t,\bot}^\top \bs{B}} \leq \min\lbr{\frac{1}{8 k \cdot 2^t \sqrt{1+\zeta}},\ \frac{\omega}{6 L_x^2}} \leq \min\lbr{\frac{1}{8 k \cdot 2^t},\ \frac{\omega}{6 L_x^2}} .
	\end{align*}

	Since $\bs{B}^\top \hat{\bs{B}}_{t} \hat{\bs{B}}_{t}^\top \bs{B} + \bs{B}^\top \hat{\bs{B}}_{t,\bot} \hat{\bs{B}}_{t,\bot}^\top \bs{B} = \bs{B}^\top (\hat{\bs{B}}_{t} \hat{\bs{B}}_{t}^\top + \hat{\bs{B}}_{t,\bot} \hat{\bs{B}}_{t,\bot}^\top) \bs{B} = \bs{B}^\top \bs{B} =\bs{I}_k$, we have $\sigma^2_{\min} ( \hat{\bs{B}}_{t}^\top \bs{B} )=1 - \| \hat{\bs{B}}_{t,\bot}^\top \bs{B} \|^2$.
	
	Thus, we have
	\begin{align*}
		\sigma_{\min}(\hat{\bs{B}}_{t}^\top \bs{B}) = \sqrt{1 - \nbr{\hat{\bs{B}}_{t,\bot}^\top \bs{B}}^2} \geq \sqrt{1 - \min\lbr{\frac{1}{64 k^2 \cdot 2^{2t}},\ \frac{\omega^2}{36 L_x^4}} } > 0 ,
	\end{align*}
	which implies that $\hat{\bs{B}}_{t}^\top \bs{B}$ is invertible.
	
	
	Now, we first analyze Term 1 in Eq.~\eqref{eq:norm_B_hat_y_decompose}.
	
	\begin{align*}
		\textup{Term 1} = & \sbr{\hat{\bs{B}}_t^\top \bs{y}}^\top \bs{P}_t^{-1} \hat{\bs{B}}_t^\top \bs{y}
		\\
		= & \sbr{\hat{\bs{B}}_t^\top \bs{B} \bs{B}^\top \bs{y} + \hat{\bs{B}}_t^\top \bs{B}_{\bot} \bs{B}_{\bot}^\top \bs{y} }^\top \bs{P}_t^{-1} \sbr{\hat{\bs{B}}_t^\top \bs{B} \bs{B}^\top \bs{y} + \hat{\bs{B}}_t^\top \bs{B}_{\bot} \bs{B}_{\bot}^\top \bs{y} }
		\\
		\\
		= & \underbrace{\sbr{\hat{\bs{B}}_t^\top \bs{B} \bs{B}^\top \bs{y}}^\top \bs{P}_t^{-1} \sbr{\hat{\bs{B}}_t^\top \bs{B} \bs{B}^\top \bs{y}}}_{\textup{Term 1-1}} + \underbrace{\sbr{\hat{\bs{B}}_t^\top \bs{B} \bs{B}^\top \bs{y}}^\top \bs{P}_t^{-1} \sbr{\hat{\bs{B}}_t^\top \bs{B}_{\bot} \bs{B}_{\bot}^\top \bs{y} } }_{\textup{Term 1-2}}
		\\
		& + \underbrace{\sbr{ \hat{\bs{B}}_t^\top \bs{B}_{\bot} \bs{B}_{\bot}^\top \bs{y} }^\top \bs{P}_t^{-1} \sbr{\hat{\bs{B}}_t^\top \bs{B} \bs{B}^\top \bs{y}}}_{\textup{Term 1-3}} + \underbrace{\sbr{ \hat{\bs{B}}_t^\top B_{\bot} \bs{B}_{\bot}^\top \bs{y} }^\top \bs{P}_t^{-1} \sbr{ \hat{\bs{B}}_t^\top \bs{B}_{\bot} \bs{B}_{\bot}^\top \bs{y} }}_{\textup{Term 1-4}} .
	\end{align*}

	In the following, we bound Terms 1-1, 1-2, 1-3 and 1-4, respectively.
	
	First, we have
	\begin{align*}
		\textup{Term 1-1} = & \sbr{\hat{\bs{B}}_t^\top \bs{B} \bs{B}^\top \bs{y}}^\top \sbr{\sum_{i=1}^{n} \lambda^*_m(\bs{x}_i) \sbr{\hat{\bs{B}}_t^\top \bs{B} \bs{B}^\top \bs{x}_i} \cdot \sbr{\hat{\bs{B}}_t^\top \bs{B} \bs{B}^\top \bs{x}_i}^\top}^{-1} \hat{\bs{B}}_t^\top \bs{B} \bs{B}^\top \bs{y}
		\\
		= & \sbr{\hat{\bs{B}}_t^\top \bs{B} \bs{B}^\top \bs{y}}^\top \sbr{ \hat{\bs{B}}_t^\top \bs{B} \sbr{\sum_{i=1}^{n} \lambda^*_m(\bs{x}_i)  \bs{B}^\top \bs{x}_i  \bs{x}_i^\top \bs{B}}  \sbr{\hat{\bs{B}}_t^\top \bs{B}}^\top }^{-1} \hat{\bs{B}}_t^\top \bs{B} \bs{B}^\top \bs{y}
		\\
		= & \sbr{\hat{\bs{B}}_t^\top \bs{B} \bs{B}^\top \bs{y}}^\top \sbr{\sbr{\hat{\bs{B}}_t^\top \bs{B}}^{-1} }^\top \sbr{ \sum_{i=1}^{n} \lambda^*_m(\bs{x}_i)  \bs{B}^\top \bs{x}_i  \bs{x}_i^\top \bs{B} }^{-1} \sbr{\hat{\bs{B}}_t^\top \bs{B}}^{-1} \hat{\bs{B}}_t^\top \bs{B} \bs{B}^\top \bs{y}
		\\
		= & \sbr{\bs{B}^\top \bs{y}}^\top \sbr{ \sum_{i=1}^{n} \lambda^*_m(\bs{x}_i)  \bs{B}^\top \bs{x}_i  \bs{x}_i^\top \bs{B} }^{-1} \bs{B}^\top \bs{y} 
		\\
		= & \nbr{\bs{B}^\top \bs{y}}^2_{\sbr{ \sum_{i=1}^{n} \lambda^*_m(\bs{x}_i)  \bs{B}^\top \bs{x}_i  \bs{x}_i^\top \bs{B} }^{-1}} .
	\end{align*}
	
	We note that since $\hat{\bs{B}}_t^\top \bs{B} \bs{B}^\top \hat{\bs{B}}_t + \hat{\bs{B}}_t^\top \bs{B}_{\bot} \bs{B}_{\bot}^\top \hat{\bs{B}}_t = \hat{\bs{B}}_t^\top (B \bs{B}^\top + \bs{B}_{\bot} \bs{B}_{\bot}^\top) \hat{\bs{B}}_t = \hat{\bs{B}}_t^\top \hat{\bs{B}}_t =\bs{I}_k$, $\sigma^2_{\min}\sbr{ \hat{\bs{B}}_t^\top \bs{B} }=1 - \nbr{\hat{\bs{B}}_t^\top \bs{B}_{\bot}}^2$.
	In addition, $\nbr{\sbr{\hat{\bs{B}}_t^\top \bs{B}}^{-1}}=\frac{1}{ \sigma_{\min}(\hat{\bs{B}}_t^\top \bs{B}) }=\frac{1}{ \sqrt{1 - \nbr{\hat{\bs{B}}_t^\top \bs{B}_{\bot}}^2} }$.
	
	Then, second, we have
	\begin{align*}
		\textup{Term 1-2} = & \sbr{\hat{\bs{B}}_t^\top \bs{B} \bs{B}^\top \bs{y}}^\top \sbr{\sum_{i=1}^{n} \lambda^*_m(\bs{x}_i) \sbr{\hat{\bs{B}}_t^\top \bs{B} \bs{B}^\top \bs{x}_i} \cdot \sbr{\hat{\bs{B}}_t^\top \bs{B} \bs{B}^\top \bs{x}_i}^\top}^{-1} \hat{\bs{B}}_t^\top \bs{B}_{\bot} \bs{B}_{\bot}^\top \bs{y}
		\\
		= & \sbr{\hat{\bs{B}}_t^\top \bs{B} \bs{B}^\top \bs{y}}^\top \sbr{ \hat{\bs{B}}_t^\top \bs{B} \sbr{\sum_{i=1}^{n} \lambda^*_m(\bs{x}_i)  \bs{B}^\top \bs{x}_i  \bs{x}_i^\top \bs{B}}  \sbr{\hat{\bs{B}}_t^\top \bs{B}}^\top }^{-1} \hat{\bs{B}}_t^\top \bs{B}_{\bot} \bs{B}_{\bot}^\top \bs{y}
		\\
		= & \sbr{\hat{\bs{B}}_t^\top \bs{B} \bs{B}^\top \bs{y}}^\top \sbr{\sbr{\hat{\bs{B}}_t^\top \bs{B}}^{-1} }^\top \sbr{ \sum_{i=1}^{n} \lambda^*_m(\bs{x}_i)  \bs{B}^\top \bs{x}_i  \bs{x}_i^\top \bs{B} }^{-1} \sbr{\hat{\bs{B}}_t^\top \bs{B}}^{-1} \hat{\bs{B}}_t^\top \bs{B}_{\bot} \bs{B}_{\bot}^\top \bs{y}
		\\
		= & \sbr{\bs{B}^\top \bs{y}}^\top \sbr{ \sum_{i=1}^{n} \lambda^*_m(\bs{x}_i)  \bs{B}^\top \bs{x}_i  \bs{x}_i^\top \bs{B} }^{-1} \sbr{\hat{\bs{B}}_t^\top \bs{B}}^{-1} \hat{\bs{B}}_t^\top \bs{B}_{\bot} \bs{B}_{\bot}^\top \bs{y}  
		\\
		\leq & 2 L_{x} \cdot \frac{1}{\omega} \cdot \frac{1}{ \sqrt{1-\nbr{\hat{\bs{B}}_t^\top \bs{B}_{\bot}}^2 } } \cdot \nbr{\hat{\bs{B}}_t^\top \bs{B}_{\bot}} \cdot 2 L_{x}
		\\
		\leq & 4 L_{x}^2 \cdot \frac{1}{\omega} \cdot \frac{1}{\sqrt{ 1-\sbr{\frac{1}{8k \cdot 2^t}}^2 }} \cdot \frac{1}{8k \cdot 2^t}
		\\
		\leq & 4 L_{x}^2 \cdot \frac{1}{\omega} \cdot \frac{1}{\sqrt{ 1-\frac{3}{4} }} \cdot \frac{1}{8k \cdot 2^t}
		\\
		= & \frac{ L_{x}^2 }{ k \omega \cdot 2^t} .
	\end{align*}
	
	Third, we have
	\begin{align*}
		\textup{Term 1-3} = & \sbr{\hat{\bs{B}}_t^\top \bs{B}_{\bot} \bs{B}_{\bot}^\top \bs{y}}^\top \sbr{\sum_{i=1}^{n} \lambda^*_m(\bs{x}_i) \sbr{\hat{\bs{B}}_t^\top \bs{B} \bs{B}^\top \bs{x}_i} \cdot \sbr{\hat{\bs{B}}_t^\top \bs{B} \bs{B}^\top \bs{x}_i}^\top}^{-1} \hat{\bs{B}}_t^\top \bs{B} \bs{B}^\top \bs{y}
		\\
		= & \sbr{\hat{\bs{B}}_t^\top \bs{B}_{\bot} \bs{B}_{\bot}^\top \bs{y}}^\top \sbr{ \hat{\bs{B}}_t^\top \bs{B} \sbr{\sum_{i=1}^{n} \lambda^*_m(\bs{x}_i)  \bs{B}^\top \bs{x}_i  \bs{x}_i^\top \bs{B}}  \sbr{\hat{\bs{B}}_t^\top \bs{B}}^\top }^{-1} \hat{\bs{B}}_t^\top \bs{B} \bs{B}^\top \bs{y}
		\\
		= & \sbr{\hat{\bs{B}}_t^\top \bs{B}_{\bot} \bs{B}_{\bot}^\top \bs{y}}^\top \sbr{\sbr{\hat{\bs{B}}_t^\top \bs{B}}^{-1} }^\top \sbr{ \sum_{i=1}^{n} \lambda^*_m(\bs{x}_i)  \bs{B}^\top \bs{x}_i  \bs{x}_i^\top \bs{B} }^{-1} \sbr{\hat{\bs{B}}_t^\top \bs{B}}^{-1} \hat{\bs{B}}_t^\top \bs{B} \bs{B}^\top \bs{y}
		\\
		= & \sbr{ \sbr{\hat{\bs{B}}_t^\top \bs{B}}^{-1} \hat{\bs{B}}_t^\top \bs{B}_{\bot} \bs{B}_{\bot}^\top \bs{y}}^\top \sbr{ \sum_{i=1}^{n} \lambda^*_m(\bs{x}_i)  \bs{B}^\top \bs{x}_i  \bs{x}_i^\top \bs{B} }^{-1} \bs{B}^\top \bs{y} ,
		\\
		\leq & \frac{1}{ \sqrt{1-\nbr{\hat{\bs{B}}_t^\top \bs{B}_{\bot}}^2 } } \cdot \nbr{\hat{\bs{B}}_t^\top \bs{B}_{\bot}} \cdot 2 L_{x} \cdot \frac{1}{\omega} \cdot 2 L_{x} 
		\\
		\leq & 4 L_{x}^2 \cdot \frac{1}{\omega} \cdot \frac{1}{\sqrt{ 1-\sbr{\frac{1}{8k \cdot 2^t}}^2 }} \cdot \frac{1}{8k \cdot 2^t}
		\\
		\leq & 4 L_{x}^2 \cdot \frac{1}{\omega} \cdot \frac{1}{\sqrt{ 1-\frac{3}{4} }} \cdot \frac{1}{8k \cdot 2^t}
		\\
		= & \frac{ L_{x}^2 }{k \omega \cdot 2^t} .
	\end{align*}
	
	Finally, we have
	\begin{align*}
		\textup{Term 1-4} = & \sbr{\hat{\bs{B}}_t^\top \bs{B}_{\bot} \bs{B}_{\bot}^\top \bs{y}}^\top \sbr{\sum_{i=1}^{n} \lambda^*_m(\bs{x}_i) \sbr{\hat{\bs{B}}_t^\top \bs{B} \bs{B}^\top \bs{x}_i} \cdot \sbr{\hat{\bs{B}}_t^\top \bs{B} \bs{B}^\top \bs{x}_i}^\top}^{-1} \hat{\bs{B}}_t^\top \bs{B}_{\bot} \bs{B}_{\bot}^\top \bs{y}
		\\
		= & \sbr{\hat{\bs{B}}_t^\top \bs{B}_{\bot} \bs{B}_{\bot}^\top \bs{y}}^\top \sbr{ \hat{\bs{B}}_t^\top \bs{B} \sbr{\sum_{i=1}^{n} \lambda^*_m(\bs{x}_i)  \bs{B}^\top \bs{x}_i  \bs{x}_i^\top \bs{B}}  \sbr{\hat{\bs{B}}_t^\top \bs{B}}^\top }^{-1} \hat{\bs{B}}_t^\top \bs{B}_{\bot} \bs{B}_{\bot}^\top \bs{y}
		\\
		= & \sbr{\hat{\bs{B}}_t^\top \bs{B}_{\bot} \bs{B}_{\bot}^\top \bs{y}}^\top \sbr{\sbr{\hat{\bs{B}}_t^\top \bs{B}}^{-1} }^\top \sbr{ \sum_{i=1}^{n} \lambda^*_m(\bs{x}_i)  \bs{B}^\top \bs{x}_i  \bs{x}_i^\top \bs{B} }^{-1} \sbr{\hat{\bs{B}}_t^\top \bs{B}}^{-1} \hat{\bs{B}}_t^\top \bs{B}_{\bot} \bs{B}_{\bot}^\top \bs{y}
		\\
		= & \sbr{ \sbr{\hat{\bs{B}}_t^\top \bs{B}}^{-1} \hat{\bs{B}}_t^\top \bs{B}_{\bot} \bs{B}_{\bot}^\top \bs{y}}^\top \sbr{ \sum_{i=1}^{n} \lambda^*_m(\bs{x}_i)  \bs{B}^\top \bs{x}_i  \bs{x}_i^\top \bs{B} }^{-1} \sbr{\hat{\bs{B}}_t^\top \bs{B}}^{-1} \hat{\bs{B}}_t^\top \bs{B}_{\bot} \bs{B}_{\bot}^\top \bs{y} ,
		\\
		\leq & \sbr{\frac{1}{ \sqrt{1-\nbr{\hat{\bs{B}}_t^\top \bs{B}_{\bot}}^2 } } \cdot \nbr{\hat{\bs{B}}_t^\top \bs{B}_{\bot}} \cdot 2 L_{x} }^2 \cdot \frac{1}{\omega}
		\\
		\leq & \sbr{ 2 L_{x} \cdot \frac{1}{\sqrt{ 1-\sbr{\frac{1}{8k \cdot 2^t}}^2 }} \cdot \frac{1}{8k \cdot 2^t} }^2 \cdot \frac{1}{\omega}
		\\
		\leq & \sbr{ 2 L_{x} \cdot \frac{1}{\sqrt{ 1-\frac{3}{4} }} \cdot \frac{1}{8k \cdot 2^t} }^2 \cdot \frac{1}{\omega}
		\\
		= & \frac{L_{x}^2}{4k^2 \omega \cdot 2^{2t}} .
	\end{align*}
	
	Thus, we have
	\begin{align}
		\textup{Term 1} = & \sbr{\hat{\bs{B}}_t^\top \bs{y}}^\top \bs{P}_t^{-1} \hat{\bs{B}}_t^\top \bs{y}
		\nonumber\\
		\leq & \nbr{\bs{B}^\top \bs{y}}^2_{\sbr{ \sum_{i=1}^{n} \lambda^*_m(\bs{x}_i)  \bs{B}^\top \bs{x}_i  \bs{x}_i^\top \bs{B} }^{-1}} + \frac{ 2L_{x}^2 }{k \omega \cdot 2^t} + \frac{L_{x}^2}{4k^2 \omega \cdot 2^{2t}} 
		\nonumber\\
		\leq & \nbr{\bs{B}^\top \bs{y}}^2_{\sbr{ \sum_{i=1}^{n} \lambda^*_m(\bs{x}_i)  \bs{B}^\top \bs{x}_i  \bs{x}_i^\top \bs{B} }^{-1}} + \frac{ 3L_{x}^2 }{k \omega \cdot 2^t} . \label{eq:term1_ub}
	\end{align}

	Next, we investigate Term 2. In order to bound Term 2, we first bound the minimum singular value of $\bs{P}_t$ and the maximum singular value of $\bs{Q}_t$.

	Since $\bs{P}_t=\hat{\bs{B}}_t^\top \bs{B} (\sum_{i=1}^{n} \lambda^*(\bs{x}_i) \bs{B}^\top \bs{x}_i \bs{x}_i^\top \bs{B}) (\hat{\bs{B}}_t^\top \bs{B})^\top$, we have
	\begin{align*}
		\sigma_{\min}(\bs{P}_t) \geq & \sigma^2_{\min}(\hat{\bs{B}}_t^\top \bs{B}) \cdot \omega
		\\
		= & \sbr{1-\|\hat{\bs{B}}_t^\top \bs{B}_{\bot}\|^2 } \omega 
		\\
		\geq & \sbr{ 1-\frac{1}{8^2 k^2 \cdot 2^{2t}} } \omega 
		\\
		\geq & \frac{3}{4} \omega .
	\end{align*}
	
	Since $\bs{Q}_t=\hat{\bs{B}}_t^\top \bs{B} (\sum_{i=1}^{n} \lambda^*(\bs{x}_i)  \bs{B}^\top \bs{x}_i \bs{x}_i^\top \bs{B}_{\bot} ) (\hat{\bs{B}}_t^\top \bs{B}_{\bot})^\top 
	+ \hat{\bs{B}}_t^\top \bs{B}_{\bot} (\sum_{i=1}^{n} \lambda^*(\bs{x}_i)  \bs{B}^\top_{\bot} \bs{x}_i \bs{x}_i^\top \bs{B} ) (\hat{\bs{B}}_t^\top \bs{B})^\top + \hat{\bs{B}}_t^\top \bs{B}_{\bot} (\sum_{i=1}^{n} \lambda^*(\bs{x}_i)  \bs{B}^\top_{\bot} \bs{x}_i \bs{x}_i^\top \bs{B}_{\bot} ) (\hat{\bs{B}}_t^\top \bs{B}_{\bot})^\top $, we have 
	\begin{align*}
		\sigma_{\max}(\bs{Q}_t) \leq & 3 L_{x}^2 \nbr{ \hat{\bs{B}}_t^\top \bs{B}_{\bot} } 
		\\
		\leq & \min \lbr{ \frac{3 L_{x}^2}{8k \cdot 2^t} , \ \frac{ \omega }{2} } .
	\end{align*}
	
	Then, we can bound Term 2 as 
	\begin{align}
		\textup{Term 2} = & \sbr{\hat{\bs{B}}_t^\top \bs{y}}^\top \sbr{\bs{P}_t + \bs{Q}_t}^{-1} \bs{Q}_t \bs{P}_t^{-1} \hat{\bs{B}}_t^\top \bs{y}
		\nonumber\\
		\leq & \nbr{\hat{\bs{B}}_t^\top \bs{y}}^2 \cdot \nbr{\sbr{\bs{P}_t + \bs{Q}_t}^{-1}} \cdot \nbr{\bs{Q}_t} \cdot \nbr{\bs{P}_t^{-1}}
		\nonumber\\
		\leq & \frac{ 4 L_{x}^2 \cdot \sigma_{\max}\sbr{\bs{Q}_t} }{ \sigma_{\min}\sbr{\bs{P}_t + \bs{Q}_t} \cdot \sigma_{\min}\sbr{\bs{P}_t} }
		\nonumber\\
		\leq & \frac{ 4 L_{x}^2 \cdot \sigma_{\max}\sbr{\bs{Q}_t} }{ \sbr{\sigma_{\min}\sbr{\bs{P}_t} - \sigma_{\max}\sbr{\bs{Q}_t}} \cdot \sigma_{\min}\sbr{\bs{P}_t} }
		\nonumber\\
		\leq & \frac{ 4 L_{x}^2 \cdot \frac{3 L_{x}^2}{8k \cdot 2^t} }{ \sbr{ \frac{3}{4} \omega - \frac{ 1 }{2} \omega } \cdot \frac{3}{4} \omega }
		\nonumber\\
		= & \frac{ 8 L_{x}^4 }{ k \omega^2 \cdot 2^t } . \label{eq:term2_ub}
	\end{align}

	Plugging Eqs.~\eqref{eq:term1_ub} and \eqref{eq:term2_ub} into Eq.~\eqref{eq:norm_B_hat_y_decompose}, we have
	\begin{align*}
		\|\hat{\bs{B}}_t^\top \bs{y}\|^2_{\sbr{ \sum_{i=1}^{n} \lambda^*_m(\bs{x}_i) \hat{\bs{B}}_t^\top \bs{x}_i \bs{x}_i^\top \hat{\bs{B}}_t }^{-1}} \leq & \nbr{\bs{B}^\top \bs{y}}^2_{\sbr{ \sum_{i=1}^{n} \lambda^*_m(\bs{x}_i)  \bs{B}^\top \bs{x}_i  \bs{x}_i^\top \bs{B} }^{-1}} + \frac{ 3L_{x}^2 }{k \omega \cdot 2^t} + \frac{ 8 L_{x}^4 }{ k \omega^2 \cdot 2^t } .
		\\
		\leq & \nbr{\bs{B}^\top \bs{y}}^2_{\sbr{ \sum_{i=1}^{n} \lambda^*_m(\bs{x}_i)  \bs{B}^\top \bs{x}_i  \bs{x}_i^\top \bs{B} }^{-1}} + \frac{ 11 L_{x}^4 }{ k \omega^2 \cdot 2^t } .
	\end{align*}
	
\end{proof}

Below we prove the sample complexity for algorithm $\algrepbailb$ (Theorem~\ref{thm:bai_ub}).

\begin{proof}[Proof of Theorem~\ref{thm:bai_ub}]
	According to Lemmas~\ref{lemma:Z_t_est_error} and \ref{lemma:concentration_variance_bai}, we have $\Pr [\cE \cap \cF] \geq 1-\delta$. Below, supposing that event $\cE \cap \cF$ holds, we prove the correctness and sample complexity.
	
	We first prove the correctness.
	
	For any task $m \in [M]$, let $t^{*}_m$ denote the first phase which satisfies $|\hat{\cX}_{t,m}|=1$. Let $t_*=\max_{m \in [M]} t^{*}_m$ denote the total number of phases used.
	For any task $m \in [M]$, let $\Delta_{m,\min}:=\min_{\bs{x} \in \cX \setminus \{\bs{x}^{*}_{m}\}} (\bs{x}^{*}_{m} - \bs{x})^\top \bs{\theta}_m$ denote the minimum reward gap for task $m$. Let $\Delta_{\min}:=\min_{m \in [M]} \Delta_{\min,m}$ denote the minimum reward gap among all tasks.
	
	From Lemma~\ref{lemma:cX_subseteq_cS}, we can obtain the following facts: (i) For any task $m \in [M]$, the optimal arm $\bs{x}^{*}_{m}$ will never be eliminated. (ii) $t^{*}_m \leq \lceil \log(\frac{4}{\Delta_{m,\min}}) \rceil +1$, and thus, $t_*\leq \lceil \log(\frac{4}{\Delta_{\min}}) \rceil +1$. Therefore, after at most $\lceil \log(\frac{4}{\Delta_{\min}}) \rceil +1$ phases, algorithm $\algrepbailb$ will return the optimal arms $\bs{x}^{*}_{m}$ for all tasks $m \in [M]$.
	
	Now we prove the sample complexity. 
	In the following, we first prove that the sample complexity of algorithm $\algrepbailb$ is bounded by $\tilde{O} ( \frac{M k}{\Delta_{\min}^2}  \log(\delta^{-1}) + (\rho^E)^2 d k^4  L_{x}^2 L_{w}^2  D \log^4(\delta^{-1}) )$.
	
	Recall that $p=\frac{180d}{\zeta^2}$ and $\zeta=\frac{1}{10}$. Then, summing the number of samples used in subroutines $\algfeatrecover$ and $\algelimlowrep$ in all phases (Line~\ref{line:bai_stage2_sample} in Algorithm~\ref{alg:feat_recover}, Line~\ref{line:bai_stage3_sample} in Algorithm~\ref{alg:elim_low_rep}), we have that the total number of samples is
	\begin{align}
		& \sum_{t=1}^{t^*} p M T_t + \sum_{m=1}^{M} \sum_{t=1}^{t^{*}_m} N_{t,m}
		\nonumber\\
		= & \sum_{t=1}^{t^*} p \cdot O \Bigg( \sbr{1+\zeta}^3 (\rho^E)^2 k^4 L_{x}^2 L_{\theta}^2  \max\lbr{2^{2t},\ \frac{L_{x}^4}{\omega^2}}  \log^2 \sbr{\frac{p}{\delta_t}} \cdot \nonumber\\& \hspace*{5em} \log^2 \sbr{ \sbr{1+\zeta} \rho^E k p L_{x} L_{\theta}  \max\lbr{2^t,\ \frac{L_{x}}{\omega}}  \frac{1}{\delta_t} \log\sbr{\frac{p}{\delta_t}} } \Bigg)
		\nonumber\\
		& + \sum_{m=1}^{M} \sum_{t=1}^{t^{*}_m} O\sbr{ 2^{2t} (1+\zeta) \rho^{G}_{t,m} \log \sbr{\frac{n^2 M}{\delta_t}} + \frac{k}{\zeta^2} }
		\nonumber\\
		= & \sum_{t=1}^{O(\log(\Delta_{\min}^{-1}))} O \Bigg( (\rho^E)^2 k^4 d L_{x}^2 L_{\theta}^2  \max\lbr{2^{2t},\ \frac{L_{x}^4}{\omega^2}}  \log^2 \sbr{\frac{d \log(\Delta_{\min}^{-1})}{\delta}} \cdot \nonumber\\& \hspace*{5em} \log^2 \sbr{ \rho^E k d L_{x} L_{\theta}  \max\lbr{\Delta_{\min}^{-1},\ \frac{L_{x}}{\omega}}  \frac{\log(\Delta_{\min}^{-1})}{\delta} \log\sbr{\frac{d \log(\Delta_{\min}^{-1})}{\delta}} } \Bigg)
		\nonumber\\
		& + \sum_{m=1}^{M} \sum_{t=1}^{O(\log(\Delta_{m,\min}^{-1}))} O\sbr{ 2^{2t} \rho^{G}_{t,m} \log \sbr{\frac{n^2 M \log(\Delta_{m,\min}^{-1})  }{\delta}}  + k}
		\label{eq:bai_sample_complexity} \\
		= & O \Bigg( (\rho^E)^2 k^4 d L_{x}^2 L_{\theta}^2  \max\lbr{ \Delta_{\min}^{-2} ,\ \frac{ L_{x}^4 \log(\Delta_{\min}^{-1}) }{\omega^2} }  \log^2 \sbr{\frac{d \log(\Delta_{\min}^{-1})}{\delta}} \cdot \nonumber\\& \hspace*{5em} \log^2 \sbr{ \rho^E k d L_{x} L_{\theta}  \max\lbr{\Delta_{\min}^{-1},\ \frac{L_{x}}{\omega}}  \frac{\log(\Delta_{\min}^{-1})}{\delta} \log\sbr{\frac{d \log(\Delta_{\min}^{-1})}{\delta}} } \Bigg)
		\nonumber\\
		& + O\sbr{ M k \Delta_{\min}^{-2} \log \sbr{\frac{n^2 M \log(\Delta_{\min}^{-1})}{\delta}} } . \nonumber
	\end{align}
\end{proof}

Next, we prove that the sample complexity of algorithm $\algrepbailb$ is bounded by 
$
\tilde{O} ( \sum_{m=1}^{M}  \min_{\bs{\lambda} \in \triangle_{\cX}} \max_{\bs{x} \in \cX \setminus \{\bs{x}^{*}_{m}\}} \frac{\| \bs{B}^\top (\bs{x}^{*}_{m} - \bs{x}) \|^2_{\bs{A}(\bs{\lambda})^{-1}} }{ ((\bs{x}^{*}_{m} - \bs{x})^\top \bs{\theta}_m)^2 } \log (\delta^{-1}) +  (\rho^E)^2 d k^4 L_{x}^2 L_{w}^2 D \log^4 (\delta^{-1}) ) .
$

From Eq.~\eqref{eq:bai_sample_complexity}, we have that with probability $1-\delta$, the number of samples used by algorithm $\algrepbailb$ is bounded by
\begin{align}
	\tilde{O} \Bigg( \sum_{t=1}^{\log(\Delta_{\min}^{-1})} (\rho^E)^2 k^4 d L_{x}^2 L_{\theta}^2  \max\lbr{2^{2t},\ \frac{L_{x}^4}{\omega^2}} + \sum_{m=1}^{M} \sum_{t=1}^{\log(\Delta_{m,\min}^{-1})} 2^{2t} \rho^{G}_{t,m} + Mk \Bigg) .	\label{eq:bai_sample_complexity_tilde}
\end{align}



For any $\cZ \subseteq \cX$, $\cY(\cZ):=\{ \bs{x}-\bs{x}':\ \forall \bs{x},\bs{x}' \in \cZ,\ \bs{x} \neq \bs{x}' \}$ and $\cY^*_m(\cZ):=\{ \bs{x}^{*}_{m}-\bs{x}:\ \forall \bs{x} \in \cZ,\ \bs{x} \neq \bs{x}^{*}_{m} \}$.
Then, we have that for any task $m \in [M]$ and phase $t\geq 2$,

\begin{align}
	\sbr{2^t}^2 \rho^{G}_{t,m} 
	= & \sbr{2^t}^2 \min_{\bs{\lambda} \in \triangle_{\cX}} \max_{\bs{y} \in \cY(\hat{\cX}_{t,m})} \| \hat{\bs{B}}_t^\top \bs{y} \|^2_{\sbr{\sum_{i=1}^n \lambda(\bs{x}_i) \hat{\bs{B}}_t^\top \bs{x}_i \bs{x}_i^\top \hat{\bs{B}}_t}^{-1}}
	\nonumber\\
	\leq & \sbr{2^t}^2 \max_{\bs{y} \in \cY(\hat{\cX}_{t,m})} \| \hat{\bs{B}}_t^\top \bs{y} \|^2_{\sbr{\sum_{i=1}^n \lambda^*_m(\bs{x}_i) \hat{\bs{B}}_t^\top \bs{x}_i \bs{x}_i^\top \hat{\bs{B}}_t}^{-1}}
	\nonumber\\
	\overset{\textup{(a)}}{\leq} & \sbr{2^t}^2 \max_{\bs{y} \in \cY(\cZ_{t,m})} \| \hat{\bs{B}}_t^\top \bs{y} \|^2_{\sbr{\sum_{i=1}^n \lambda^*_m(\bs{x}_i) \hat{\bs{B}}_t^\top \bs{x}_i \bs{x}_i^\top \hat{\bs{B}}_t}^{-1}}
	\nonumber\\
	\overset{\textup{(b)}}{\leq} & 4 \sbr{2^t}^2 \max_{\bs{y} \in \cY^*_m(\cZ_{t,m})} \| \hat{\bs{B}}_t^\top \bs{y} \|^2_{\sbr{\sum_{i=1}^n \lambda^*_m(\bs{x}_i) \hat{\bs{B}}_t^\top \bs{x}_i \bs{x}_i^\top \hat{\bs{B}}_t}^{-1}}
	\nonumber\\
	\overset{\textup{(c)}}{\leq} & 4 \sbr{2^t}^2 \sbr{\max_{\bs{y} \in \cY^*_m(\cZ_{t,m})} \| \bs{B}^\top \bs{y} \|^2_{\sbr{\sum_{i=1}^n \lambda^*_m(\bs{x}_i) \bs{B}^\top \bs{x}_i \bs{x}_i^\top \bs{B}}^{-1}} + \frac{ 11 L_{x}^4 }{ \omega^2 k \cdot 2^t } }
	\nonumber\\
	= & 4 \sbr{ \frac{16 \max_{\bs{y} \in \cY^*_m(\cZ_{t,m})} \| \bs{B}^\top \bs{y} \|^2_{\sbr{\sum_{i=1}^n \lambda^*_m(\bs{x}_i) \bs{B}^\top \bs{x}_i \bs{x}_i^\top \bs{B}}^{-1}}}{\sbr{4 \cdot 2^{-t}}^2} +  \frac{ 11 L_{x}^4 \cdot 2^t }{ \omega^2 k } }
	\nonumber\\
	\overset{\textup{(d)}}{\leq} & 4 \sbr{ 16 \max_{\bs{y} \in \cY^*_m(\cZ_{t,m})}  \frac{\| \bs{B}^\top \bs{y} \|^2_{\sbr{\sum_{i=1}^n \lambda^*_m(\bs{x}_i) \bs{B}^\top \bs{x}_i \bs{x}_i^\top \bs{B}}^{-1}}}{ \sbr{\bs{y}^\top \bs{\theta}_m}^2} + \frac{ 11 L_{x}^4 \cdot 2^t }{ \omega^2 k } }
	\nonumber\\
	\leq & 4 \sbr{ 16 \max_{\bs{y} \in \cY^*_m(\cX)}  \frac{\| \bs{B}^\top \bs{y} \|^2_{\sbr{\sum_{i=1}^n \lambda^*_m(\bs{x}_i) \bs{B}^\top \bs{x}_i \bs{x}_i^\top \bs{B}}^{-1}}}{ \sbr{\bs{y}^\top \bs{\theta}_m}^2} + \frac{ 11 L_{x}^4 \cdot 2^t }{ \omega^2 k } }
	\nonumber\\
	\overset{\textup{(e)}}{=} & 4 \sbr{ 16 \min_{\bs{\lambda} \in \triangle_{\cX}} \max_{\bs{y} \in \cY^*_m(\cX)}  \frac{\| \bs{B}^\top \bs{y} \|^2_{\sbr{\sum_{i=1}^n \lambda(\bs{x}_i) \bs{B}^\top \bs{x}_i \bs{x}_i^\top \bs{B}}^{-1}}}{ \sbr{\bs{y}^\top \bs{\theta}_m}^2} + \frac{ 11 L_{x}^4 \cdot 2^t }{ \omega^2 k } } . \label{eq:connection_ub_lb}
\end{align}
Here inequality (a) is due to $\hat{\cX}_{t,m} \subseteq \cZ_{t,m}$ (from Lemma~\ref{lemma:cX_subseteq_cS}). Inequality (b) uses the fact that for any $\bs{y}=\bs{x}_i-\bs{x}_j \in \cY(\cZ_{t,m})$, we can write $\bs{y}=(\bs{x}^{*}_{m}-\bs{x}_j)-(\bs{x}^{*}_{m}-\bs{x}_i)$, and the triangle inequality. Inequality (c) follows from Lemma~\ref{lemma:connection_ub_lb}, and inequality (d) is due to that for any  $\bs{y} \in \cY^*_m(\cZ_{t,m})$, $\bs{y}^\top \bs{\theta}_m \leq 4 \cdot 2^{-t}$ (from the definition of $\cZ_{t,m}$). Equality (e) comes from the definition of $\bs{\lambda}^*_m$.

Let $L:=\log^2 (\frac{d \log(\Delta_{\min}^{-1})}{\delta}) \cdot \log^2 (\rho^E k d L_{x} L_{\theta}  \max \{\Delta_{\min}^{-1},\ \frac{L_{x}}{\omega}\}  \frac{\log(\Delta_{\min}^{-1})}{\delta} \log (\frac{d \log(\Delta_{\min}^{-1})}{\delta}) ) $.
Plugging Eq.~\eqref{eq:connection_ub_lb} into Eq.~\eqref{eq:bai_sample_complexity_tilde}, we have that with probability $1-\delta$, the number of samples used by algorithm $\algrepbailb$ is bounded by
\begin{align*}
	& O \sbr{ \sum_{m=1}^{M} \!\!\! \sum_{t=1}^{\log(\Delta_{m,\min}^{-1})} \!\!\!\!\!\! 2^{2t} \rho^{G}_{t,m}  \log \sbr{\frac{n^2 M \log(\Delta_{m,\min}^{-1})  }{\delta}} + Mk \log(\Delta_{\min}^{-1}) + \sum_{t=1}^{\log(\Delta_{\min}^{-1})} (\rho^E)^2 k^4 d L_{x}^2 L_{\theta}^2  \max\lbr{2^{2t},\ \frac{L_{x}^4}{\omega^2}}  L } 
	\\
	= & O \Bigg( \sum_{m=1}^{M} \sum_{t=2}^{\log(\Delta_{m,\min}^{-1})} \Bigg( \min_{\bs{\lambda} \in \triangle_{\cX}} \max_{\bs{y} \in \cY^*_m(\cX)}  \frac{\| \bs{B}^\top \bs{y} \|^2_{\sbr{\sum_{i=1}^n \lambda(\bs{x}_i) \bs{B}^\top \bs{x}_i \bs{x}_i^\top \bs{B}}^{-1}}}{ \sbr{\bs{y}^\top \bs{\theta}_m}^2} + \frac{ L_{x}^4 \cdot 2^t }{ \omega^2 k }  \Bigg) \cdot \log \sbr{\frac{n^2 M \log(\Delta_{m,\min}^{-1})  }{\delta}}  \\& + \sum_{m=1}^{M} \rho^*_{1,m} \log \sbr{\frac{n^2 M \log(\Delta_{m,\min}^{-1})  }{\delta}} + Mk \log(\Delta_{\min}^{-1}) + \sum_{t=1}^{\log(\Delta_{\min}^{-1})} (\rho^E)^2 k^4 d L_{x}^2 L_{\theta}^2  \max\lbr{2^{2t},\ \frac{L_{x}^4}{\omega^2}} \cdot L \Bigg) 
	\\
	\overset{\textup{(a)}}{=} & O \Bigg( \sum_{m=1}^{M}  \min_{\bs{\lambda} \in \triangle_{\cX}} \max_{\bs{y} \in \cY^*_m(\cX)}  \frac{\| \bs{B}^\top \bs{y} \|^2_{\sbr{\sum_{i=1}^n \lambda(\bs{x}_i) \bs{B}^\top \bs{x}_i \bs{x}_i^\top \bs{B}}^{-1}}}{ \sbr{\bs{y}^\top \bs{\theta}_m}^2} \cdot \log \sbr{\frac{n^2 M \log(\Delta_{m,\min}^{-1})  }{\delta}} \cdot \log(\Delta_{m,\min}^{-1}) \\& + \frac{ M L_{x}^4 }{ \omega^2 k \cdot \Delta_{\min} } \cdot \log \sbr{\frac{n^2 M \log(\Delta_{\min}^{-1})  }{\delta}} + Mk \cdot \log \sbr{\frac{n^2 M \log(\Delta_{\min}^{-1})  }{\delta}} \cdot \log(\Delta_{\min}^{-1}) \\& \qquad +  (\rho^E)^2 k^4 d L_{x}^2 L_{\theta}^2  \max\lbr{ \Delta_{\min}^{-2} ,\ \frac{L_{x}^4 \cdot \log(\Delta_{\min}^{-1})}{\omega^2}} \cdot L  \Bigg) ,
\end{align*}
where equality (a) uses Lemma~\ref{lemma:rho_t_leq_4k}.

When $L_x=\omega=\Theta(1)$, we have that with probability $1-\delta$, the sample complexity of algorithm $\algrepbailb$ is bounded by
\begin{align*}
	\tilde{O} \Bigg( \sum_{m=1}^{M}  \min_{\bs{\lambda} \in \triangle_{\cX}} \max_{\bs{y} \in \cY^*_m(\cX)}  \frac{\| \bs{B}^\top \bs{y} \|^2_{\sbr{\sum_{i=1}^n \lambda(\bs{x}_i) \bs{B}^\top \bs{x}_i \bs{x}_i^\top \bs{B}}^{-1}}}{ \sbr{\bs{y}^\top \bs{\theta}_m}^2}  \log\sbr{\frac{1}{\delta}} +  (\rho^E)^2 k^4 d L_{x}^2 L_{\theta}^2  \max\lbr{ \Delta_{\min}^{-2} ,\ \frac{L_{x}^4}{\omega^2}}  \log^4\sbr{\frac{1}{\delta}}  \Bigg) .
\end{align*}

\section{Proofs for Algorithm~$\algrepbpiclb$}

In this section, we present the proofs for Algorithm~$\algrepbpiclb$.

\subsection{Context Distribution Estimation and Sample Batch Planning} \label{apx:bpi_sample_batch_planning}

Define $\lambda_{\cD}^{E}$ and $\rho_{\cD}^{E}$ as the optimal solution and the optimal value of the following E-optimal  design optimization:
\begin{align}
	\min_{\bs{\lambda} \in \triangle_{\cA}} \nbr{ \sbr{ \sum_{a \in \cA} \lambda(a) \ex_{s \sim \cD} \mbr{\bs{\phi}(s,a) \bs{\phi}(s,a)^\top} }^{-1} } . \label{eq:E_optimal_design}
\end{align}

\begin{lemma}\label{lemma:bound_rho_E_cD}
	It holds that
	\begin{align*}
		\rho_{\cD}^{E} \leq \frac{1}{\nu} .
	\end{align*}
\end{lemma}
\begin{proof}[Proof of Lemma~\ref{lemma:bound_rho_E_cD}]
	The optimization in Eq.~\eqref{eq:E_optimal_design} is equivalent to maximize the minimum singular value of the matrix $\sum_{a \in \cA} \lambda(a) \ex_{s \sim \cD} \mbr{\bs{\phi}(s,a) \bs{\phi}(s,a)^\top}$.
	
	Thus, $\lambda_{\cD}^{E}$ is the optimal solution of the following optimization:
	\begin{align*}
		\max_{\bs{\lambda} \in \triangle_{\cA}} \sigma_{\min} \sbr{ \sum_{a \in \cA} \lambda(a) \ex_{s \sim \cD} \mbr{\bs{\phi}(s,a) \bs{\phi}(s,a)^\top} } .
	\end{align*}
	
	Using Assumption~\ref{assumption:bpi_rho_E_D_is_finite}, we have 
	\begin{align*}
		\sigma_{\min} \sbr{ \sum_{a \in \cA} \lambda_{\cD}^{E} \ex_{s \sim \cD} \mbr{\bs{\phi}(s,a) \bs{\phi}(s,a)^\top} } \geq \nu .
	\end{align*}
	
	Then, we have
	\begin{align*}
		\rho_{\cD}^{E} = & \nbr{ \sbr{ \sum_{a \in \cA} \lambda_{\cD}^{E} \ex_{s \sim \cD} \mbr{\bs{\phi}(s,a) \bs{\phi}(s,a)^\top} }^{-1} }
		\\
		= & \frac{1}{\sigma_{\min} \sbr{ \sum_{a \in \cA} \lambda_{\cD}^{E} \ex_{s \sim \cD} \mbr{\bs{\phi}(s,a) \bs{\phi}(s,a)^\top} }}
		\\
		\leq & \frac{1}{\nu} .
	\end{align*}
\end{proof}

Define event
\begin{align*}
	\cK:= \lbr{ \nbr{ \ex_{s \sim \hat{\cD}} \mbr{ \bs{\phi}(s,a) \bs{\phi}(s,a)^\top } - \ex_{s \sim \cD} \mbr{ \bs{\phi}(s,a) \bs{\phi}(s,a)^\top } } \leq \frac{ 8 L_{\phi}^2 \log \sbr{\frac{20d |\cA|}{\delta}} }{ \sqrt{T_0} } , \ \forall a \in \cA } .
\end{align*}

\begin{lemma} \label{lemma:number_of_samples_T0}
	It holds that
	\begin{align*}
		\Pr \mbr{\cK} \geq 1-\frac{\delta}{5} .
	\end{align*}
	Furthermore, if event $\cK$ holds and 
	\begin{align*}
		%
		T_0 = \left \lceil \frac{32^2 (1+\zeta)^2 L_{\phi}^4}{\nu^2} \log^2 \sbr{\frac{20d |\cA|}{\delta}} \right \rceil , 
	\end{align*}
	we have that for any $a \in \cA$,
	\begin{align*}
		\nbr{ \ex_{s \sim \hat{\cD}} \mbr{ \bs{\phi}(s,a) \bs{\phi}(s,a)^\top } - \ex_{s \sim \cD} \mbr{ \bs{\phi}(s,a) \bs{\phi}(s,a)^\top } } \leq & \frac{\nu}{ 4(1+\zeta) } .
	\end{align*}
\end{lemma}

\begin{proof}[Proof of Lemma~\ref{lemma:number_of_samples_T0}]
	For any $(s,a) \in \cS \times \cA$, $\| \bs{\phi}(s,a) \bs{\phi}(s,a)^\top \| \leq L_{\phi}^2$. Then, using the matrix Bernstern inequality (Lemma~\ref{lemma:matrix_bernstein_tau}) and a union bound over $a \in \cA$, we have that with probability $1-\frac{\delta}{5}$, for any $a \in \cA$,
	\begin{align*}
		\nbr{ \ex_{s \sim \hat{\cD}} \mbr{ \bs{\phi}(s,a) \bs{\phi}(s,a)^\top } - \ex_{s \sim \cD} \mbr{ \bs{\phi}(s,a) \bs{\phi}(s,a)^\top } } \leq & 4 L_{\phi}^2 \sqrt{ \frac{ \log \sbr{\frac{10 \cdot 2d |\cA|}{\delta}} }{ T_0 } } + \frac{ 4 L_{\phi}^2 \log \sbr{\frac{10 \cdot 2d |\cA|}{\delta}} }{ T_0 } 
		\\
		\leq & \frac{ 8 L_{\phi}^2 \log \sbr{\frac{20d |\cA|}{\delta}} }{ \sqrt{T_0} } .
	\end{align*} 
	If $T_0 \geq 32^2 (1+\zeta)^2 \nu^{-2} L_{\phi}^4 \log^2 \sbr{\frac{20d |\cA|}{\delta}}$, we have
	\begin{align*}
		\nbr{ \ex_{s \sim \hat{\cD}} \mbr{ \bs{\phi}(s,a) \bs{\phi}(s,a)^\top } - \ex_{s \sim \cD} \mbr{ \bs{\phi}(s,a) \bs{\phi}(s,a)^\top } } \leq \frac{\nu}{ 4(1+\zeta) } ,
	\end{align*} 
	which completes the proof.
\end{proof}

Define event
\begin{align*}
	\cL:= \Bigg\{ & \nbr{ \sum_{i=1}^{p} \bs{\phi}(s_{m,j,i}^{(\ell)},\bar{a}_i) \bs{\phi}(s_{m,j,i}^{(\ell)},\bar{a}_i)^\top - \sum_{i=1}^{p} \ex_{s \sim \cD} \mbr{\bs{\phi}(s,\bar{a}_i) \bs{\phi}(s,\bar{a}_i)^\top} } \leq 8 L_{\phi}^2 \sqrt{ p } \log \sbr{\frac{40dMT}{\delta}} , \\& \forall m \in [M], \ \forall j \in [T], \ \forall \ell \in \{1,2\} \Bigg\} .
\end{align*}

\begin{lemma} \label{lemma:number_of_samples_p}
	It holds that
	\begin{align*}
		\Pr \mbr{\cL} \geq 1 - \frac{\delta}{5} .
	\end{align*}
	Furthermore, if event $\cL$ holds and 
	\begin{align}
		p = \left \lceil \frac{32^2 (1+\zeta)^2 L_{\phi}^4}{\nu^2} \log^2 \sbr{\frac{40dMT}{\delta}} \right \rceil , \label{eq:value_p}
	\end{align}
	we have that for any $m \in [M]$, $j \in [T]$ and $\ell \in \{1,2\}$,
	\begin{align*}
		\nbr{ \sum_{i=1}^{p} \bs{\phi}(s_{m,j,i}^{(\ell)},\bar{a}_i) \bs{\phi}(s_{m,j,i}^{(\ell)},\bar{a}_i)^\top - \sum_{i=1}^{p} \ex_{s \sim \cD} \mbr{\bs{\phi}(s,\bar{a}_i) \bs{\phi}(s,\bar{a}_i)^\top} } \leq & \frac{p \nu}{ 4(1+\zeta) } .
	\end{align*}
	Here, the value of $T$ is specified in Eq.~\eqref{eq:value_T}.
\end{lemma}

\begin{proof}[Proof of Lemma~\ref{lemma:number_of_samples_p}]
	For any $(s,a) \in \cS \times \cA$, $\| \bs{\phi}(s,a) \bs{\phi}(s,a)^\top \| \leq L_{\phi}^2$. Then, using the matrix Bernstern inequality (Lemma~\ref{lemma:matrix_bernstein_tau}) and a union bound over $m \in [M]$, $j \in [T]$ and $\ell \in \{1,2\}$, we have that with probability $1-\frac{\delta}{5}$, for any $m \in [M]$, $j \in [T]$ and $\ell \in \{1,2\}$,
	\begin{align*}
		\nbr{ \sum_{i=1}^{p} \! \bs{\phi}(s_{m,j,i}^{(\ell)},\bar{a}_i) \bs{\phi}(s_{m,j,i}^{(\ell)},\bar{a}_i)^\top \!\!-\!\! \sum_{i=1}^{p} \ex_{s \sim \cD} \mbr{\bs{\phi}(s,\bar{a}_i) \bs{\phi}(s,\bar{a}_i)^\top} } \leq & 4 L_{\phi}^2 \sqrt{ p \log \sbr{\frac{10 \cdot 4dMT}{\delta}} } \!+\! 4 L_{\phi}^2 \log \sbr{\frac{10 \cdot 4dMT}{\delta}}
		\\
		\leq & 8 L_{\phi}^2 \sqrt{ p } \log \sbr{\frac{40dMT}{\delta}} .
	\end{align*}
	In addition, if $p \geq 32^2 (1+\zeta)^2 \nu^{-2} L_{\phi}^4 \log^2 \sbr{\frac{40dMT}{\delta}}$, we have that 
	\begin{align*}
		8 L_{\phi}^2 \sqrt{ p } \log \sbr{\frac{40dMT}{\delta}} \leq \frac{p \nu}{ 4(1+\zeta) }
	\end{align*}
	and thus,
	\begin{align*}
		\nbr{ \sum_{i=1}^{p} \bs{\phi}(s_{m,j,i}^{(\ell)},\bar{a}_i) \bs{\phi}(s_{m,j,i}^{(\ell)},\bar{a}_i)^\top - \sum_{i=1}^{p} \ex_{s \sim \cD} \mbr{\bs{\phi}(s,\bar{a}_i) \bs{\phi}(s,\bar{a}_i)^\top} } \leq \frac{p \nu}{ 4(1+\zeta) } ,
	\end{align*} 
	which completes the proof.
\end{proof}

For any task $m \in [M]$, round $j \in [T]$ and $\ell \in \{1,2\}$, let
$$
\bs{\Phi}_{m,j}^{(\ell)} = 
\begin{bmatrix}
	\bs{\phi}(s_{m,j,1}^{(\ell)},\bar{a}_1)^\top
	\\
	\dots
	\\
	\bs{\phi}(s_{m,j,p}^{(\ell)},\bar{a}_p)^\top
\end{bmatrix} ,
$$ 
and
$$
(\bs{\Phi}_{m,j}^{(\ell)})^{+} = ((\bs{\Phi}_{m,j}^{(\ell)})^{\top} \bs{\Phi}_{m,j}^{(\ell)})^{-1} (\bs{\Phi}_{m,j}^{(\ell)})^{\top} .
$$

\begin{lemma} \label{lemma:norm_Phi_plus}
	Suppose that event $\cK \cap \cL$ holds. Then, for any $m \in [M]$, $j \in [T]$ and $\ell \in \{1,2\}$,
	\begin{align*}
		\nbr{ (\bs{\Phi}_{m,j}^{(\ell)})^{+} } 
		\leq 2 \sqrt{\frac{(1+\zeta)}{p \nu}} .
	\end{align*}
\end{lemma}

\begin{proof}[Proof of Lemma~\ref{lemma:norm_Phi_plus}]
	We first assume that $(\bs{\Phi}_{m,j}^{(\ell)})^{\top} \bs{\Phi}_{m,j}^{(\ell)}$ is invertible. In our later analysis, we will prove that as long as $T_0$ and $p$ are large enough, $(\bs{\Phi}_{m,j}^{(\ell)})^{\top} \bs{\Phi}_{m,j}^{(\ell)}$ is invertible.
	
	For any $m \in [M]$, $j \in [T]$ and $\ell \in \{1,2\}$, we have
	\begin{align}
		\nbr{ (\bs{\Phi}_{m,j}^{(\ell)})^{+} } = & \nbr{ ((\bs{\Phi}_{m,j}^{(\ell)})^{\top} \bs{\Phi}_{m,j}^{(\ell)})^{-1} (\bs{\Phi}_{m,j}^{(\ell)})^{\top} } 
		\nonumber\\
		= & \sqrt{ \nbr{ ((\bs{\Phi}_{m,j}^{(\ell)})^{\top} \bs{\Phi}_{m,j}^{(\ell)})^{-1} (\bs{\Phi}_{m,j}^{(\ell)})^{\top} \bs{\Phi}_{m,j}^{(\ell)} ((\bs{\Phi}_{m,j}^{(\ell)})^{\top} \bs{\Phi}_{m,j}^{(\ell)})^{-1} } }
		\nonumber\\
		= & \sqrt{ \nbr{ ((\bs{\Phi}_{m,j}^{(\ell)})^{\top} \bs{\Phi}_{m,j}^{(\ell)})^{-1} } }
		\nonumber\\
		= & \frac{1}{ \sqrt{ \sigma_{\min} \sbr{ (\bs{\Phi}_{m,j}^{(\ell)})^{\top} \bs{\Phi}_{m,j}^{(\ell)} } } } . \label{eq:norm_Phi_plus}
	\end{align}
	
	In addition, we have
	\begin{align}
		& \sigma_{\min} \sbr{ (\bs{\Phi}_{m,j}^{(\ell)})^{\top} \bs{\Phi}_{m,j}^{(\ell)} } 
		\nonumber\\ 
		= & \sigma_{\min} \sbr{ \sum_{i=1}^{p} \bs{\phi}(s_{m,j,i}^{(\ell)},\bar{a}_i) \bs{\phi}(s_{m,j,i}^{(\ell)},\bar{a}_i)^\top }
		\nonumber\\
		= & \sigma_{\min} \sbr{ \sum_{i=1}^{p} \ex_{s \sim \cD} \mbr{\bs{\phi}(s,\bar{a}_i) \bs{\phi}(s,\bar{a}_i)^\top} + \sum_{i=1}^{p} \bs{\phi}(s_{m,j,i}^{(\ell)},\bar{a}_i) \bs{\phi}(s_{m,j,i}^{(\ell)},\bar{a}_i)^\top - \sum_{i=1}^{p} \ex_{s \sim \cD} \mbr{\bs{\phi}(s,\bar{a}_i) \bs{\phi}(s,\bar{a}_i)^\top} }
		\nonumber\\
		\geq & \sigma_{\min} \sbr{ \sum_{i=1}^{p} \ex_{s \sim \cD} \mbr{\bs{\phi}(s,\bar{a}_i) \bs{\phi}(s,\bar{a}_i)^\top} } - \nbr{ \sum_{i=1}^{p} \bs{\phi}(s_{m,j,i}^{(\ell)},\bar{a}_i) \bs{\phi}(s_{m,j,i}^{(\ell)},\bar{a}_i)^\top - \sum_{i=1}^{p} \ex_{s \sim \cD} \mbr{\bs{\phi}(s,\bar{a}_i) \bs{\phi}(s,\bar{a}_i)^\top} } 
		\nonumber\\
		= & \sigma_{\min} \sbr{ \sum_{i=1}^{p} \ex_{s \sim \hat{\cD}} \mbr{\bs{\phi}(s,\bar{a}_i) \bs{\phi}(s,\bar{a}_i)^\top} + \sum_{i=1}^{p} \ex_{s \sim \cD} \mbr{\bs{\phi}(s,\bar{a}_i) \bs{\phi}(s,\bar{a}_i)^\top} - \sum_{i=1}^{p} \ex_{s \sim \hat{\cD}} \mbr{\bs{\phi}(s,\bar{a}_i) \bs{\phi}(s,\bar{a}_i)^\top} } \nonumber\\& - \nbr{ \sum_{i=1}^{p} \bs{\phi}(s_{m,j,i}^{(\ell)},\bar{a}_i) \bs{\phi}(s_{m,j,i}^{(\ell)},\bar{a}_i)^\top - \sum_{i=1}^{p} \ex_{s \sim \cD} \mbr{\bs{\phi}(s,\bar{a}_i) \bs{\phi}(s,\bar{a}_i)^\top} } 
		\nonumber\\
		\geq & \sigma_{\min} \sbr{ \sum_{i=1}^{p} \ex_{s \sim \hat{\cD}} \mbr{\bs{\phi}(s,\bar{a}_i) \bs{\phi}(s,\bar{a}_i)^\top} } - \nbr{ \sum_{i=1}^{p} \ex_{s \sim \cD} \mbr{\bs{\phi}(s,\bar{a}_i) \bs{\phi}(s,\bar{a}_i)^\top} - \sum_{i=1}^{p} \ex_{s \sim \hat{\cD}} \mbr{\bs{\phi}(s,\bar{a}_i) \bs{\phi}(s,\bar{a}_i)^\top} } \nonumber\\& - \nbr{ \sum_{i=1}^{p} \bs{\phi}(s_{m,j,i}^{(\ell)},\bar{a}_i) \bs{\phi}(s_{m,j,i}^{(\ell)},\bar{a}_i)^\top - \sum_{i=1}^{p} \ex_{s \sim \cD} \mbr{\bs{\phi}(s,\bar{a}_i) \bs{\phi}(s,\bar{a}_i)^\top} }
		\nonumber\\
		\geq & \sigma_{\min} \sbr{ \sum_{i=1}^{p} \ex_{s \sim \hat{\cD}} \mbr{\bs{\phi}(s,\bar{a}_i) \bs{\phi}(s,\bar{a}_i)^\top} } -  \sum_{i=1}^{p} \nbr{ \ex_{s \sim \cD} \mbr{\bs{\phi}(s,\bar{a}_i) \bs{\phi}(s,\bar{a}_i)^\top} - \ex_{s \sim \hat{\cD}} \mbr{\bs{\phi}(s,\bar{a}_i) \bs{\phi}(s,\bar{a}_i)^\top} } \nonumber\\& - \nbr{ \sum_{i=1}^{p} \bs{\phi}(s_{m,j,i}^{(\ell)},\bar{a}_i) \bs{\phi}(s_{m,j,i}^{(\ell)},\bar{a}_i)^\top - \sum_{i=1}^{p} \ex_{s \sim \cD} \mbr{\bs{\phi}(s,\bar{a}_i) \bs{\phi}(s,\bar{a}_i)^\top} } 
		\nonumber\\
		\geq & \sigma_{\min} \sbr{ \sum_{i=1}^{p} \ex_{s \sim \hat{\cD}} \mbr{\bs{\phi}(s,\bar{a}_i) \bs{\phi}(s,\bar{a}_i)^\top} } - \frac{p \nu}{ 4(1+\zeta) } - \frac{p \nu}{ 4(1+\zeta) }
		, \label{eq:sigma_min_Phi_top_times_Phi}
	\end{align}
	where the last inequality uses Lemmas~\ref{lemma:number_of_samples_T0} and \ref{lemma:number_of_samples_p}.
	
	In the following, we analyze $\sigma_{\min}  (\sum_{i=1}^{p} \ex_{s \sim \hat{\cD}} \mbr{\bs{\phi}(s,\bar{a}_i) \bs{\phi}(s,\bar{a}_i)^\top})$.
	According to the guarantee of the rounding procedure $\round$, we have
	\begin{align*}
		\nbr{ \sbr{ \sum_{i=1}^{p} \ex_{s \sim \hat{\cD}} \mbr{\bs{\phi}(s,\bar{a}_i) \bs{\phi}(s,\bar{a}_i)^\top} }^{-1} } \leq & (1+\zeta) \nbr{ \sbr{ p \sum_{a \in \cA} \lambda^{E}_{\hat{\cD}}(a) \ex_{s \sim \hat{\cD}} \mbr{\bs{\phi}(s,a) \bs{\phi}(s,a)^\top} }^{-1} }
		\\
		\leq & (1+\zeta) \nbr{ \sbr{ p \sum_{a \in \cA} \lambda^{E}(a) \ex_{s \sim \hat{\cD}} \mbr{\bs{\phi}(s,a) \bs{\phi}(s,a)^\top} }^{-1} } ,
	\end{align*}
	which implies that
	\begin{align}
		& \sigma_{\min}\sbr{ \sum_{i=1}^{p} \ex_{s \sim \hat{\cD}} \mbr{\bs{\phi}(s,\bar{a}_i) \bs{\phi}(s,\bar{a}_i)^\top} } 
		\nonumber\\
		\geq & \frac{p}{1+\zeta} \sigma_{\min} \sbr{ \sum_{a \in \cA} \lambda^{E}(a) \ex_{s \sim \hat{\cD}} \mbr{\bs{\phi}(s,a) \bs{\phi}(s,a)^\top} } 
		\nonumber\\
		\geq & \frac{p}{1+\zeta} \sigma_{\min} \Bigg( \sum_{a \in \cA} \lambda^{E}(a) \ex_{s \sim \cD} \mbr{\bs{\phi}(s,a) \bs{\phi}(s,a)^\top} \nonumber\\& + \sum_{a \in \cA} \lambda^{E}(a) \ex_{s \sim \hat{\cD}} \mbr{\bs{\phi}(s,a) \bs{\phi}(s,a)^\top} - \sum_{a \in \cA} \lambda^{E}(a) \ex_{s \sim \cD} \mbr{\bs{\phi}(s,a) \bs{\phi}(s,a)^\top} \Bigg)
		\nonumber\\
		\geq & \frac{p}{1+\zeta} \Bigg( \sigma_{\min} \sbr{ \sum_{a \in \cA} \lambda^{E}(a) \ex_{s \sim \cD} \mbr{\bs{\phi}(s,a) \bs{\phi}(s,a)^\top} } \nonumber\\& - \nbr{ \sum_{a \in \cA} \lambda^{E}(a) \ex_{s \sim \hat{\cD}} \mbr{\bs{\phi}(s,a) \bs{\phi}(s,a)^\top} - \sum_{a \in \cA} \lambda^{E}(a) \ex_{s \sim \cD} \mbr{\bs{\phi}(s,a) \bs{\phi}(s,a)^\top} } \Bigg)
		\nonumber\\
		\geq & \frac{p}{1+\zeta} \Bigg( \frac{1}{\rho_{\cD}^{E}} - \sum_{a \in \cA} \lambda^{E}(a) \nbr{ \ex_{s \sim \hat{\cD}} \mbr{\bs{\phi}(s,a) \bs{\phi}(s,a)^\top} - \ex_{s \sim \cD} \mbr{\bs{\phi}(s,a) \bs{\phi}(s,a)^\top} } \Bigg)
		\nonumber\\
		\overset{\textup{(a)}}{\geq} & \frac{p}{1+\zeta} \sbr{ \nu - \frac{\nu}{ 4(1+\zeta) } }
		\nonumber\\
		\geq & \frac{3p \nu}{ 4(1+\zeta) } , \label{eq:sigma_min_batch_under_hat_cD}
	\end{align}
	where inequality (a) uses Lemmas~\ref{lemma:bound_rho_E_cD} and \ref{lemma:number_of_samples_T0}. 
	
	Plugging Eq.~\eqref{eq:sigma_min_batch_under_hat_cD} into Eq.~\eqref{eq:sigma_min_Phi_top_times_Phi}, we have
	\begin{align}
		\sigma_{\min} \sbr{ (\bs{\Phi}_{m,j}^{(\ell)})^{\top} \bs{\Phi}_{m,j}^{(\ell)} } \geq & \frac{3p \nu}{ 4(1+\zeta) } - \frac{p \nu}{ 4(1+\zeta) } - \frac{p \nu}{ 4(1+\zeta) }
		\nonumber\\
		= & \frac{p \nu}{ 4(1+\zeta) } . \label{eq:sigma_min_Phi_top_Phi_value}
	\end{align}
	Equations~\eqref{eq:sigma_min_Phi_top_times_Phi} and \eqref{eq:sigma_min_Phi_top_Phi_value} show that if $T_0$ and $p$ are large enough to satisfy that $\nbr{ \ex_{s \sim \hat{\cD}} \mbr{ \bs{\phi}(s,a) \bs{\phi}(s,a)^\top } - \ex_{s \sim \cD} \mbr{ \bs{\phi}(s,a) \bs{\phi}(s,a)^\top } } \leq  \frac{\nu}{ 4(1+\zeta) }$ for any $a \in \cA$ and $\nbr{ \sum_{i=1}^{p} \bs{\phi}(s_{m,j,i}^{(\ell)},\bar{a}_i) \bs{\phi}(s_{m,j,i}^{(\ell)},\bar{a}_i)^\top - \sum_{i=1}^{p} \ex_{s \sim \cD} \mbr{\bs{\phi}(s,\bar{a}_i) \bs{\phi}(s,\bar{a}_i)^\top} } \leq \frac{p \nu}{ 4(1+\zeta) }$ for any $m \in [M]$, $j \in [T]$ and $\ell \in \{1,2\}$, respectively, then we have that $ (\bs{\Phi}_{m,j}^{(\ell)})^{\top} \bs{\Phi}_{m,j}^{(\ell)} $ is invertible.
	
	Continuing with Eq.~\eqref{eq:norm_Phi_plus}, we have
	\begin{align*}
		\nbr{ (\bs{\Phi}_{m,j}^{(\ell)})^{+} } 
		\leq 2 \sqrt{\frac{(1+\zeta)}{p \nu}} .
	\end{align*}
	
\end{proof}

\subsection{Global Feature Extractor Recovery with Stochastic Contexts} \label{apx:bpi_feat_recover}

In subroutine $\algconfeatrecover$, for any $m \in [M]$, $j \in [T]$, $i \in [p]$ and $\ell \in \{1,2\}$, let $s^{(\ell)}_{m,j,i}$ and $\eta^{(\ell)}_{m,j,i}$ denote the random context and noise of the $\ell$-th sample on action $\bar{a}_i$ in the $j$-th round for task $m$, respectively. Here, the superscript $\ell \in \{1,2\}$ refers to the first sample (Line~\ref{line:bpi_stage2_sample1} in Algorithm~\ref{alg:con_feat_recover}) or the second sample (Line~\ref{line:bpi_stage2_sample2} in Algorithm~\ref{alg:con_feat_recover}) on an action $\bar{a}_i$.

In $\algconfeatrecover$, for any $m \in [M]$, $j \in [T]$, $i \in [p]$ and $\ell \in \{1,2\}$, let $\bs{\alpha}^{(\ell)}_{m,j} \leftarrow [\alpha^{(\ell)}_{m,j,1}, \dots, \alpha^{(\ell)}_{m,j,p}]^\top$, and then,  $\tilde{\bs{\theta}}^{(\ell)}_{m,j} = (\bs{\Phi}_{m,j}^{(\ell)})^{+} \bs{\alpha}^{(\ell)}_{m,j}$. Recall that $\bs{Z} = \frac{1}{M T} \sum_{m=1}^{M} \sum_{j=1}^{T} \tilde{\bs{\theta}}^{(1)}_{m,j} (\tilde{\bs{\theta}}^{(2)}_{m,j})^\top$.

\begin{lemma}[Expectation of $\bs{Z}$] \label{lemma:bpi_expectation_Z}
	It holds that
	\begin{align*}
		\ex \mbr{ \bs{Z} } = \frac{1}{M} \sum_{m=1}^{M} \bs{\theta}_m \bs{\theta}_m^\top .
	\end{align*}
\end{lemma}

\begin{proof}[Proof of Lemma~\ref{lemma:bpi_expectation_Z}]
	$\bs{Z}$ can be written as
	\begin{align}
		\bs{Z} = & \frac{1}{M T} \sum_{m=1}^{M} \sum_{j=1}^{T} \tilde{\bs{\theta}}^{(1)}_{m,j} (\tilde{\bs{\theta}}^{(2)}_{m,j})^\top  
		\nonumber\\
		= & \frac{1}{M T} \sum_{m=1}^{M} \sum_{j=1}^{T} (\bs{\Phi}_{m,j}^{(1)})^{+} \begin{bmatrix}
			\alpha^{(1)}_{m,j,1}
			\\
			\vdots
			\\
			\alpha^{(1)}_{m,j,p} 
		\end{bmatrix}
		\mbr{\alpha^{(2)}_{m,j,1}, \dots, \alpha^{(2)}_{m,j,p}} ((\bs{\Phi}_{m,j}^{(2)})^{+})^{\top}
		\nonumber\\
		= & \frac{1}{M T} \sum_{m=1}^{M} \sum_{j=1}^{T} (\bs{\Phi}_{m,j}^{(1)})^{+} \Bigg( \!\!
		\begin{bmatrix}
			&\!\!\!\!\!\! \sbr{\bs{\phi}(s^{(1)}_{m,j,1},\bar{a}_1)^\top \bs{\theta}_m} \sbr{\bs{\phi}(s^{(2)}_{m,j,1},\bar{a}_1)^\top \bs{\theta}_m} &\!\!\!\!\!\!\dots\!\!\!\!\!\! &\sbr{\bs{\phi}(s^{(1)}_{m,j,1},\bar{a}_1)^\top \bs{\theta}_m} \sbr{\bs{\phi}(s^{(2)}_{m,j,p},\bar{a}_p)^\top \bs{\theta}_m}
			\\
			&\!\!\!\!\!\! \dots &\!\!\!\!\!\!\dots\!\!\!\!\!\! &\dots
			\\
			&\!\!\!\!\!\! \sbr{\bs{\phi}(s^{(1)}_{m,j,p},\bar{a}_p)^\top \bs{\theta}_m} \sbr{\bs{\phi}(s^{(2)}_{m,j,1},\bar{a}_1)^\top \bs{\theta}_m} &\!\!\!\!\!\!\dots\!\!\!\!\!\! &\sbr{\bs{\phi}(s^{(1)}_{m,j,p},\bar{a}_p)^\top \bs{\theta}_m} \sbr{\bs{\phi}(s^{(2)}_{m,j,p},\bar{a}_p)^\top \bs{\theta}_m}
		\end{bmatrix} 
		\nonumber\\& \hspace*{-2em} +
		\begin{bmatrix}
			&\!\!\!\!\!\! \bs{\phi}(s^{(1)}_{m,j,1},\bar{a}_1)^\top  \bs{\theta}_m \cdot \eta^{(2)}_{m,j,1} + \eta^{(1)}_{m,j,1} \cdot \bs{\phi}(s^{(2)}_{m,j,1},\bar{a}_1)^\top \bs{\theta}_m &\!\!\!\!\!\dots\!\!\!\!\! & \bs{\phi}(s^{(1)}_{m,j,1},\bar{a}_1)^\top  \bs{\theta}_m \cdot \eta^{(2)}_{m,j,p} + \eta^{(1)}_{m,j,1} \cdot \bs{\phi}(s^{(2)}_{m,j,p},\bar{a}_p)^\top \bs{\theta}_m
			\\
			&\!\!\!\!\!\! \dots &\!\!\!\!\!\dots\!\!\!\!\! &\dots
			\\
			&\!\!\!\!\!\! \bs{\phi}(s^{(1)}_{m,j,p},\bar{a}_p)^\top  \bs{\theta}_m \cdot \eta^{(2)}_{m,j,1} + \eta^{(1)}_{m,j,p} \cdot \bs{\phi}(s^{(2)}_{m,j,1},\bar{a}_1)^\top \bs{\theta}_m  & \!\!\!\!\!\dots\!\!\!\!\! & \bs{\phi}(s^{(1)}_{m,j,p},\bar{a}_p)^\top  \bs{\theta}_m \cdot \eta^{(2)}_{m,j,p} + \eta^{(1)}_{m,j,p} \cdot \bs{\phi}(s^{(2)}_{m,j,p},\bar{a}_p)^\top \bs{\theta}_m
		\end{bmatrix} 
		\nonumber\\& \hspace*{-2em} +
		\begin{bmatrix}
			&\!\!\!\!\!\! \eta^{(1)}_{m,j,1} \cdot \eta^{(2)}_{m,j,1} &\dots & \eta^{(1)}_{m,j,1} \cdot \eta^{(2)}_{m,j,p}
			\\
			&\!\!\!\!\!\! \dots &\dots &\dots
			\\
			&\!\!\!\!\!\! \eta^{(1)}_{m,j,p} \cdot \eta^{(2)}_{m,j,1} &\dots &\eta^{(1)}_{m,j,p} \cdot \eta^{(2)}_{m,j,p}
		\end{bmatrix} 
		\Bigg)
		((\bs{\Phi}_{m,j}^{(2)})^{+})^{\top} \label{eq:decompose_Z} .
	\end{align}
	
	For any task $m \in [M]$, $j \in [T]$, $i \in [p]$, the sample on action $a_i$ in the first round (i.e., $s^{(1)}_{m,j,i}$ and $\eta^{(1)}_{m,j,i}$) is independent of that in the second round (i.e., $s^{(2)}_{m,j,i}$ and $\eta^{(2)}_{m,j,i}$).  
	Hence, taking the expectation on $\bs{Z}$, we obtain
	
	\begin{align*}
		\ex[\bs{Z}] = & \frac{1}{M T} \sum_{m=1}^{M} \sum_{j=1}^{T} \ex \Bigg[ (\bs{\Phi}_{m,j}^{(1)})^{+} \cdot
		\\
		&
		\begin{bmatrix}
			&\sbr{\bs{\phi}(s^{(1)}_{m,j,1},\bar{a}_1)^\top \bs{\theta}_m} \sbr{\bs{\phi}(s^{(2)}_{m,j,1},\bar{a}_1)^\top \bs{\theta}_m} &\!\!\!\!\!\dots\!\!\!\!\! &\sbr{\bs{\phi}(s^{(1)}_{m,j,1},\bar{a}_1)^\top \bs{\theta}_m} \sbr{\bs{\phi}(s^{(2)}_{m,j,p},\bar{a}_p)^\top \bs{\theta}_m}
			\\
			&\dots &\!\!\!\!\!\dots\!\!\!\!\! &\dots
			\\
			&\sbr{\bs{\phi}(s^{(1)}_{m,j,p},\bar{a}_p)^\top \bs{\theta}_m} \sbr{\bs{\phi}(s^{(2)}_{m,j,1},\bar{a}_1)^\top \bs{\theta}_m} &\!\!\!\!\!\dots\!\!\!\!\! &\sbr{\bs{\phi}(s^{(1)}_{m,j,p},\bar{a}_p)^\top \bs{\theta}_m} \sbr{\bs{\phi}(s^{(2)}_{m,j,p},\bar{a}_p)^\top \bs{\theta}_m}
		\end{bmatrix}
		((\bs{\Phi}_{m,j}^{(2)})^{+})^{\top} \Bigg]
		\\
		= & \frac{1}{M T} \sum_{m=1}^{M} \sum_{j=1}^{T} \ex \Bigg[ ((\bs{\Phi}_{m,j}^{(1)})^{\top} \bs{\Phi}_{m,j}^{(1)})^{-1} (\bs{\Phi}_{m,j}^{(1)})^{\top}  
		\begin{bmatrix}
			\bs{\phi}(s^{(1)}_{m,j,1},\bar{a}_1)^\top \bs{\theta}_m
			\\
			\vdots
			\\
			\bs{\phi}(s^{(1)}_{m,j,p},\bar{a}_p)^\top \bs{\theta}_m
		\end{bmatrix}
		\mbr{ \bs{\phi}(s^{(2)}_{m,j,1},\bar{a}_1)^\top \bs{\theta}_m, \dots, \bs{\phi}(s^{(2)}_{m,j,p},\bar{a}_p)^\top \bs{\theta}_m }
		\cdot
		\\
		& \bs{\Phi}_{m,j}^{(2)} ((\bs{\Phi}_{m,j}^{(2)})^{\top} \bs{\Phi}_{m,j}^{(2)})^{-1} \Bigg]
		\\
		= & \frac{1}{M T} \sum_{m=1}^{M} \sum_{j=1}^{T} \ex \mbr{ ((\bs{\Phi}_{m,j}^{(1)})^{\top} \bs{\Phi}_{m,j}^{(1)})^{-1} (\bs{\Phi}_{m,j}^{(1)})^{\top} \cdot \bs{\Phi}_{m,j}^{(1)} \bs{\theta}_m
			(\bs{\theta}_m)^{\top} (\bs{\Phi}_{m,j}^{(2)})^{\top} \cdot
			\bs{\Phi}_{m,j}^{(2)} ((\bs{\Phi}_{m,j}^{(2)})^{\top} \bs{\Phi}_{m,j}^{(2)})^{-1} }
		\\
		= & \frac{1}{M T} \sum_{m=1}^{M} \sum_{j=1}^{T}  \bs{\theta}_m (\bs{\theta}_m)^\top
		\\
		= & \frac{1}{M} \sum_{m=1}^{M} \bs{\theta}_m \bs{\theta}_m^\top .
	\end{align*}
	
\end{proof}

Define event 
\begin{align*}
	\cG:= \lbr{ \nbr{\bs{Z} - \ex [\bs{Z}]} \leq \frac{ 256 (1+\zeta) L_{\phi} L_{\theta} \log\sbr{\frac{50d}{\delta}} }{\nu \sqrt{MT}} \log \sbr{\frac{100pMT}{\delta}} } .
\end{align*}

\begin{lemma}[Concentration of $Z$] \label{lemma:Z_est_error}
	Suppose that $\cK \cap \cL$ holds. Then, it holds that
	\begin{align*}
		\Pr \mbr{\cG} \geq 1-\frac{\delta}{5} .
	\end{align*}
\end{lemma}

\begin{proof}[Proof of Lemma~\ref{lemma:Z_est_error}]
	Define the following matrices:
	\begin{align*}
		\bs{D}_{m,j} &:= \frac{1}{M T} (\bs{\Phi}_{m,j}^{(1)})^{+} \cdot
		\\
		&
		\begin{bmatrix}
			&\sbr{\bs{\phi}(s^{(1)}_{m,j,1},\bar{a}_1)^\top \bs{\theta}_m} \sbr{\bs{\phi}(s^{(2)}_{m,j,1},\bar{a}_1)^\top \bs{\theta}_m} &\dots &\sbr{\bs{\phi}(s^{(1)}_{m,j,1},\bar{a}_1)^\top \bs{\theta}_m} \sbr{\bs{\phi}(s^{(2)}_{m,j,p},\bar{a}_p)^\top \bs{\theta}_m}
			\\
			&\dots &\dots &\dots
			\\
			&\sbr{\bs{\phi}(s^{(1)}_{m,j,p},\bar{a}_p)^\top \bs{\theta}_m} \sbr{\bs{\phi}(s^{(2)}_{m,j,1},\bar{a}_1)^\top \bs{\theta}_m} &\dots &\sbr{\bs{\phi}(s^{(1)}_{m,j,p},\bar{a}_p)^\top \bs{\theta}_m} \sbr{\bs{\phi}(s^{(2)}_{m,j,p},\bar{a}_p)^\top \bs{\theta}_m}
		\end{bmatrix} ((\bs{\Phi}_{m,j}^{(2)})^{+})^\top
		\\
		\bs{D} &:= \sum_{m=1}^{M} \sum_{j=1}^{T} \bs{D}_{m,j}
		\\
		\bs{E}_{m,j} &:= \frac{1}{M T} (\bs{\Phi}_{m,j}^{(1)})^{+} \cdot
		\\
		&
		\hspace*{-1.8em}
		\begin{bmatrix}
			&\!\!\!\!\! \bs{\phi}(s^{(1)}_{m,j,1},\bar{a}_1)^\top  \bs{\theta}_m \cdot \eta^{(2)}_{m,j,1} + \eta^{(1)}_{m,j,1} \cdot \bs{\phi}(s^{(2)}_{m,j,1},\bar{a}_1)^\top \bs{\theta}_m &\!\!\!\!\!\dots\!\!\!\!\! & \bs{\phi}(s^{(1)}_{m,j,1},\bar{a}_1)^\top  \bs{\theta}_m \cdot \eta^{(2)}_{m,j,p} + \eta^{(1)}_{m,j,1} \cdot \bs{\phi}(s^{(2)}_{m,j,p},\bar{a}_p)^\top \bs{\theta}_m
			\\
			&\!\!\!\!\!\dots &\!\!\!\!\!\dots\!\!\!\!\! &\dots
			\\
			&\!\!\!\!\! \bs{\phi}(s^{(1)}_{m,j,p},\bar{a}_p)^\top  \bs{\theta}_m \cdot \eta^{(2)}_{m,j,1} + \eta^{(1)}_{m,j,p} \cdot \bs{\phi}(s^{(2)}_{m,j,1},\bar{a}_1)^\top \bs{\theta}_m  & \!\!\!\!\!\dots\!\!\!\!\! & \bs{\phi}(s^{(1)}_{m,j,p},\bar{a}_p)^\top  \bs{\theta}_m \cdot \eta^{(2)}_{m,j,p} + \eta^{(1)}_{m,j,p} \cdot \bs{\phi}(s^{(2)}_{m,j,p},\bar{a}_p)^\top \bs{\theta}_m
		\end{bmatrix} \! \cdot 
		\\
		& ((\bs{\Phi}_{m,j}^{(2)})^{+})^\top
		\\
		\bs{E} &:= \sum_{m=1}^{M} \sum_{j=1}^{T} \bs{E}_{m,j}
		\\
		\bs{F}_{m,j} &:= \frac{1}{M T} (\bs{\Phi}_{m,j}^{(1)})^{+}
		\begin{bmatrix}
			&\eta^{(1)}_{m,j,1} \cdot \eta^{(2)}_{m,j,1} &\dots & \eta^{(1)}_{m,j,1} \cdot \eta^{(2)}_{m,j,p}
			\\
			&\dots &\dots &\dots
			\\
			&\eta^{(1)}_{m,j,p} \cdot \eta^{(2)}_{m,j,1} &\dots &\eta^{(1)}_{m,j,p} \cdot \eta^{(2)}_{m,j,p}
		\end{bmatrix} ((\bs{\Phi}_{m,j}^{(2)})^{+})^\top
		\\
		\bs{F} &:= \sum_{m=1}^{M} \sum_{j=1}^{T} \bs{F}_{m,j}
	\end{align*}

	From Eq.~\eqref{eq:decompose_Z}, we can bound $\|\bs{Z} - \ex [\bs{Z}]\|$ as
	\begin{align}
		\nbr{\bs{Z} - \ex [\bs{Z}]} \leq & \nbr{\bs{D} - \ex[\bs{D}]} + \nbr{\bs{E} - \ex[\bs{E}]} + \nbr{\bs{F} - \ex[\bs{F}]} . \label{eq:decompose_Z_est_err}
	\end{align}
	
	Similar to the proof of Lemma~\ref{lemma:Z_t_est_error}, in order to use the truncated matrix Bernstein inequality (Lemma~\ref{lemma:matrix_bernstein_tau}), we define the truncated noise and some truncated matrices as follows.
	
	Let $R>0$ be a truncation parameter of noises which will be chosen later.  For any $m \in [M]$, $j \in [T]$, $i \in [p]$ and $\ell \in \{1,2\}$, let $\tilde{\eta}^{(\ell)}_{m,j,i}=\eta^{(\ell)}_{m,j,i} \mathbbm{1}\{|\eta^{(\ell)}_{m,j,i}| \leq R\}$ denote the truncated noise. Furthermore, we define the following matrices with truncated noises:
	\begin{align*}
		\tilde{\bs{E}}_{m,j} &:= \frac{1}{M T} (\bs{\Phi}_{m,j}^{(1)})^{+} \cdot
		\\
		&
		\hspace*{-1.8em} 
		\begin{bmatrix}
			&\!\!\!\!\! \bs{\phi}(s^{(1)}_{m,j,1},\bar{a}_1)^\top  \bs{\theta}_m \cdot \tilde{\eta}^{(2)}_{m,j,1} + \tilde{\eta}^{(1)}_{m,j,1} \cdot \bs{\phi}(s^{(2)}_{m,j,1},\bar{a}_1)^\top \bs{\theta}_m &\!\!\!\!\!\dots\!\!\!\!\! & \bs{\phi}(s^{(1)}_{m,j,1},\bar{a}_1)^\top  \bs{\theta}_m \cdot \tilde{\eta}^{(2)}_{m,j,p} + \tilde{\eta}^{(1)}_{m,j,1} \cdot \bs{\phi}(s^{(2)}_{m,j,p},\bar{a}_p)^\top \bs{\theta}_m
			\\
			&\!\!\!\!\!\dots &\!\!\!\!\!\dots\!\!\!\!\! &\dots
			\\
			&\!\!\!\!\! \bs{\phi}(s^{(1)}_{m,j,p},\bar{a}_p)^\top  \bs{\theta}_m \cdot \tilde{\eta}^{(2)}_{m,j,1} + \tilde{\eta}^{(1)}_{m,j,p} \cdot \bs{\phi}(s^{(2)}_{m,j,1},\bar{a}_1)^\top \bs{\theta}_m  & \!\!\!\!\!\dots\!\!\!\!\! & \bs{\phi}(s^{(1)}_{m,j,p},\bar{a}_p)^\top  \bs{\theta}_m \cdot \tilde{\eta}^{(2)}_{m,j,p} + \tilde{\eta}^{(1)}_{m,j,p} \cdot \bs{\phi}(s^{(2)}_{m,j,p},\bar{a}_p)^\top \bs{\theta}_m
		\end{bmatrix} \! \cdot
		\\
		& ((\bs{\Phi}_{m,j}^{(2)})^{+})^\top
		\\
		\tilde{\bs{E}} &:= \sum_{m=1}^{M} \sum_{j=1}^{T} \tilde{\bs{E}}_{m,j}
		\\
		\tilde{\bs{F}}_{m,j} &:= \frac{1}{M T} (\bs{\Phi}_{m,j}^{(1)})^{+}
		\begin{bmatrix}
			&\tilde{\eta}^{(1)}_{m,j,1} \cdot \tilde{\eta}^{(2)}_{m,j,1} &\dots & \tilde{\eta}^{(1)}_{m,j,1} \cdot \tilde{\eta}^{(2)}_{m,j,p}
			\\
			&\dots &\dots &\dots
			\\
			&\tilde{\eta}^{(1)}_{m,j,p} \cdot \tilde{\eta}^{(2)}_{m,j,1} &\dots &\tilde{\eta}^{(1)}_{m,j,p} \cdot \tilde{\eta}^{(2)}_{m,j,p}
		\end{bmatrix} ((\bs{\Phi}_{m,j}^{(2)})^{+})^\top
		\\
		\tilde{\bs{F}} &:= \sum_{m=1}^{M} \sum_{j=1}^{T} \tilde{\bs{F}}_{m,j}
	\end{align*}

	Recall that from Lemma~\ref{lemma:norm_Phi_plus}, we have that for any $m \in [M]$, $j \in [T]$ and $\ell \in \{1,2\}$, $\|(\bs{\Phi}_{m,j}^{(\ell)})^{+}\| \leq 2 \sqrt{\frac{(1+\zeta)}{p \nu}}$. Let $B_{\Phi}:=2 \sqrt{\frac{(1+\zeta)}{p \nu}}$.
	
	We first analyze $\|\bs{D}-\ex[\bs{D}]\|$. Since  $|\bs{\phi}(s^{(\ell)}_{m,j,i},\bar{a}_i)^\top \bs{\theta}_m| \leq L_{\phi}L_{\theta}$ for any $m \in [M]$, $j \in [T]$, $i \in [p]$ and $\ell \in \{1,2\}$, we have that $\|\bs{D}_{m,j}\| \leq \frac{1}{MT} \cdot p L_{\phi} L_{\theta} B_{\Phi}^2$ and $\| \sum_{m=1}^{M} \sum_{j=1}^{T} \ex[\bs{D}^2_{m,j}]\| \leq MT \cdot \frac{1}{M^2 T^2} \cdot p^2 L_{\phi}^2 L_{\theta}^2 B_{\Phi}^4 = \frac{1}{M T} \cdot p^2 L_{\phi}^2 L_{\theta}^2 B_{\Phi}^4$ for any $m \in [M]$ and $j \in [T]$. 
	
	Let $\delta' \in (0,1)$ be a confidence parameter which will be chosen later.
	Using the matrix Bernstein inequality (Lemma~\ref{lemma:matrix_bernstein_tau}), we have that with probability at least $1-\delta'$,
	
	\begin{align}
		\nbr{ \bs{D} - \ex[\bs{D}] } \leq & 4 \sqrt{ \frac{p^2 L_{\phi}^2 L_{\theta}^2 B_{\Phi}^4 \log \sbr{\frac{2d}{\delta'}}}{M T} } +  \frac{4 p L_{\phi} L_{\theta} B_{\Phi}^2 \log \sbr{\frac{2d}{\delta'}}}{MT} 
		\nonumber\\
		\leq &  \frac{8 \cdot 4 p L_{\phi} L_{\theta} B_{\Phi}^2 \log \sbr{\frac{2d}{\delta'}}}{\sqrt{MT}}  \label{eq:D_est_err}.
	\end{align}
	
	Next, we bound $\|\bs{E}-\ex[\bs{E}]\|$. 
	Since $|\bs{\phi}(s^{(\ell)}_{m,j,i},\bar{a}_i)^\top  \bs{\theta}_m| \leq L_{\phi} L_{\theta}$ and $|\tilde{\eta}^{(\ell)}_{m,j,i}| \leq R$ for any $m \in [M]$, $j \in [T]$, $i \in [p]$ and $\ell \in \{1,2\}$, we have that  $\|\tilde{\bs{E}}_{m,j}\| \leq \frac{1}{M T} \cdot 2pR L_{\phi} L_{\theta} B_{\Phi}^2$ and $\nbr{\sum_{m=1}^{M} \sum_{j=1}^{T} \ex [\tilde{\bs{E}}_{m,j}^2]} \leq \frac{1}{M T} \cdot 4p^2R^2 L_{\phi}^2 L_{\theta}^2 B_{\Phi}^4$ for any $m \in [M]$ and $j \in [T]$.

	Since $\eta^{(\ell)}_{m,j,i}$ is 1-sub-Gaussian for any $m \in [M]$, $j \in [T]$, $i \in [p]$ and $\ell \in \{1,2\}$, using a union bound over $i \in [p]$ and $\ell \in \{1,2\}$, we have that for any $m \in [M]$ and $j \in [T]$, with probability at least $1-4p\exp(-\frac{R^2}{2})$, $|\eta^{(\ell)}_{m,j,i}| \leq R$ for all $i \in [p]$ and $\ell \in \{1,2\}$, and thus, $\|\bs{E}_{m,j}\| \leq \frac{1}{M T} \cdot 2pR L_{\phi} L_{\theta} B_{\Phi}^2$. Then, we have
	\begin{align*}
		& \nbr{ \ex[\bs{E}_{m,j}]-\ex[\tilde{\bs{E}}_{m,j}] } 
		\nonumber\\
		\leq & \nbr{ \ex \mbr{\bs{E}_{m,j} \cdot \indicator{ \nbr{\bs{E}_{m,j}} \geq \frac{2pR L_{\phi} L_{\theta} B_{\Phi}^2}{M T} }} }
		\\
		\leq & \ex \mbr{ \nbr{\bs{E}_{m,j}} \cdot \indicator{ \nbr{\bs{E}_{m,j}} \geq \frac{2pR L_{\phi} L_{\theta} B_{\Phi}^2}{M T} } } 
		\\
		= & \ex\mbr{ \frac{2pR L_{\phi} L_{\theta} B_{\Phi}^2}{M T} \!\cdot\! \indicator{ \nbr{\bs{E}_{m,j}} \!\geq\! \frac{2pR L_{\phi} L_{\theta} B_{\Phi}^2}{M T} }} \!+\! \ex \mbr{\sbr{ \nbr{\bs{E}_{m,j}} \!-\! \frac{2pR L_{\phi} L_{\theta} B_{\Phi}^2}{M T}} \!\cdot\! \indicator{ \nbr{\bs{E}_{m,j}} \!\geq\! \frac{2pR L_{\phi} L_{\theta} B_{\Phi}^2}{M T} } } 
		\\
		= & \frac{2pR L_{\phi} L_{\theta} B_{\Phi}^2}{M T} \cdot \Pr\mbr{ \nbr{\bs{E}_{m,j}} \geq \frac{2pR L_{\phi} L_{\theta} B_{\Phi}^2}{M T} } + \int_0^{\infty}  \Pr \mbr{ \nbr{\bs{E}_{m,j}} - \frac{2pR L_{\phi} L_{\theta} B_{\Phi}^2}{M T} > x}  dx  
		\\
		\leq & \frac{2pR L_{\phi} L_{\theta} B_{\Phi}^2}{M T} \cdot 4p \cdot \exp\sbr{-\frac{R^2}{2}} + \frac{2p L_{\phi} L_{\theta} B_{\Phi}^2}{M T}  \int_R^{\infty} \Pr \mbr{ \nbr{\bs{E}_{m,j}} > \frac{2p L_{\phi} L_{\theta} B_{\Phi}^2 y}{M T} }  dy  
		\\
		\leq & \frac{2pR L_{\phi} L_{\theta} B_{\Phi}^2}{M T} \cdot 4p \cdot \exp\sbr{-\frac{R^2}{2}} + \frac{2p L_{\phi} L_{\theta} B_{\Phi}^2}{M T}  \int_R^{\infty} 4p \exp\sbr{-\frac{y^2}{2}}  dy
		\\
		\leq & \frac{2pR L_{\phi} L_{\theta} B_{\Phi}^2}{M T} \cdot 4p \cdot \exp\sbr{-\frac{R^2}{2}} + \frac{2p L_{\phi} L_{\theta} B_{\Phi}^2}{M T} \cdot 4p \cdot \frac{1}{R} \cdot \exp\sbr{-\frac{R^2}{2}} 
		\\
		= & \frac{2p L_{\phi} L_{\theta} B_{\Phi}^2}{M T} \cdot 4p \cdot \sbr{R+\frac{1}{R}} \exp\sbr{-\frac{R^2}{2}} .
	\end{align*}

	Using the truncated matrix Bernstein inequality (Lemma~\ref{lemma:matrix_bernstein_tau}) with $n=MT$, $R=\sqrt{2\log \sbr{\frac{4pMT}{\delta'}}}$, $U=\frac{2p L_{\phi} L_{\theta} B_{\Phi}^2 \sqrt{2\log \sbr{\frac{4pMT}{\delta'}}}}{MT}$, $\sigma^2=\frac{(2p L_{\phi} L_{\theta} B_{\Phi}^2 \sqrt{2\log \sbr{\frac{4pMT}{\delta'}}})^2}{MT}$, $\tau=4\sqrt{\frac{ (2p L_{\phi} L_{\theta} B_{\Phi}^2 \sqrt{2\log \sbr{\frac{4pMT}{\delta'}}})^2 \cdot \log\sbr{\frac{2d}{\delta'}}}{MT}}+\frac{4 \cdot 2p L_{\phi} L_{\theta} B_{\Phi}^2 \sqrt{2\log \sbr{\frac{4pMT}{\delta'}}} \cdot \log\sbr{\frac{2d}{\delta'}}}{M T}$ and $\Delta=\frac{2p L_{\phi} L_{\theta} B_{\Phi}^2 \cdot 2 \sqrt{2\log \sbr{\frac{4pMT}{\delta'}}}}{M T} \cdot \frac{\delta'}{M T}$, we have that with probability at least $1-2\delta'$, 
	\begin{align}
		\nbr{\bs{E} - \ex[\bs{E}]} \leq  \frac{8 \cdot 2p L_{\phi} L_{\theta} B_{\Phi}^2 \sqrt{2\log \sbr{\frac{4pMT}{\delta'}}} \cdot \log\sbr{\frac{2d}{\delta'}}}{\sqrt{M T}} . \label{eq:E_est_err}
	\end{align}
	
	Now we investigate $\nbr{\bs{F} - \ex[\bs{F}]}$. Since $|\tilde{\eta}^{(\ell)}_{m,j,i}| \leq R$ for any $m \in [M]$, $j \in [T]$, $i \in [p]$ and $\ell \in \{1,2\}$, we have that $\|\tilde{\bs{F}}_{m,j}\| \leq \frac{1}{M T} \cdot pR^2 B_{\Phi}^2$ and $\nbr{\sum_{m=1}^{M} \sum_{j=1}^{T} \ex \mbr{\tilde{\bs{F}}_{m,j}^2} } \leq \frac{1}{M T} \cdot p^2 R^4 B_{\Phi}^4$.
	%
	
	Recall that for any $m \in [M]$ and $j \in [T]$, with probability at least $1-4p\exp(-\frac{R^2}{2})$, $|\eta^{(\ell)}_{m,j,i}| \leq R$ for all $i \in [p]$ and $\ell \in \{1,2\}$, and thus, $\|\bs{F}_{m,j}\| \leq \frac{1}{M T} \cdot p B_{\Phi}^2 R^2 $. Then, we have
	\begin{align*}
		\nbr{ \ex[\bs{F}_{m,j}]-\ex[\tilde{\bs{F}}_{m,j}] } \leq & \nbr{ \ex \mbr{\bs{F}_{m,j} \cdot \indicator{ \nbr{\bs{F}_{m,j}} \geq \frac{p B_{\Phi}^2 R^2}{M T} }} }
		\\
		\leq & \ex \mbr{ \nbr{\bs{F}_{m,j}} \cdot \indicator{ \nbr{\bs{F}_{m,j}} \geq \frac{p B_{\Phi}^2 R^2}{M T} } } 
		\\
		= & \ex\mbr{ \frac{p B_{\Phi}^2 R^2}{M T} \cdot \indicator{ \nbr{\bs{F}_{m,j}} \geq \frac{p B_{\Phi}^2 R^2}{M T} }} + \mbr{\sbr{ \nbr{\bs{F}_{m,j}} - \frac{p B_{\Phi}^2 R^2}{M T}} \cdot \indicator{ \nbr{\bs{F}_{m,j}} \geq \frac{p B_{\Phi}^2 R^2}{M T} } } 
		\\
		= & \frac{p B_{\Phi}^2 R^2}{M T} \cdot \Pr\mbr{ \nbr{\bs{F}_{m,j}} \geq \frac{p B_{\Phi}^2 R^2}{M T} } + \int_0^{\infty}  \Pr \mbr{ \nbr{\bs{F}_{m,j}} - \frac{p B_{\Phi}^2 R^2}{M T} > x}  dx  
		\\
		\leq & \frac{p B_{\Phi}^2 R^2}{M T} \cdot 4p \cdot \exp\sbr{-\frac{R^2}{2}} + \frac{2p B_{\Phi}^2}{M T}  \int_R^{\infty} \bs{y} \cdot \Pr \mbr{ \nbr{\bs{F}_{m,j}} > \frac{p B_{\Phi}^2 y^2}{M T} } dy  
		\\
		\leq & \frac{p B_{\Phi}^2 R^2}{M T} \cdot 4p \cdot \exp\sbr{-\frac{R^2}{2}} + \frac{2p B_{\Phi}^2}{M T}  \int_R^{\infty}  \bs{y} \cdot 4p \exp\sbr{-\frac{y^2}{2}} dy
		\\
		\leq & \frac{p B_{\Phi}^2 R^2}{M T} \cdot 4p \cdot \exp\sbr{-\frac{R^2}{2}} + \frac{2p B_{\Phi}^2}{M T} \cdot 4p \cdot \exp\sbr{-\frac{R^2}{2}} 
		\\
		= & \frac{p B_{\Phi}^2}{M T} \cdot 4p \cdot \sbr{R^2+2} \exp\sbr{-\frac{R^2}{2}} .
	\end{align*}

	Using the truncated matrix Bernstein inequality (Lemma~\ref{lemma:matrix_bernstein_tau}) with $n=MT$, $R=\sqrt{2\log \sbr{\frac{4pMT}{\delta'}}}$, $U=\frac{p B_{\Phi}^2 \cdot 2\log \sbr{\frac{4pMT}{\delta'}}}{MT}$, $\sigma^2=\frac{(p B_{\Phi}^2 \cdot 2\log \sbr{\frac{4pMT}{\delta'}})^2}{MT}$, $\tau=4\sqrt{\frac{ (p B_{\Phi}^2 \cdot 2\log \sbr{\frac{4pMT}{\delta'}})^2 \cdot \log\sbr{\frac{2d}{\delta'}}}{MT}}+\frac{4\cdot p B_{\Phi}^2 \cdot 2\log \sbr{\frac{4pMT}{\delta'}} \cdot \log\sbr{\frac{2d}{\delta'}}}{M T}$ and $\Delta=\frac{p B_{\Phi}^2 \cdot 2 \cdot 2\log \sbr{\frac{4pMT}{\delta'}} }{M T} \cdot \frac{\delta'}{M T}$, we have that with probability at least $1-2\delta'$, 
	\begin{align}
		\nbr{\bs{F} - \ex\mbr{\bs{F}}} \leq \frac{8 \cdot p B_{\Phi}^2 \cdot 2\log \sbr{\frac{4pMT}{\delta'}} \cdot \log\sbr{\frac{2d}{\delta'}}}{\sqrt{MT}} \label{eq:F_est_err} .
	\end{align}
	
	Plugging Eqs.~\eqref{eq:D_est_err}-\eqref{eq:F_est_err} into Eq.~\eqref{eq:decompose_Z_est_err}, we have that with probability at least $1-5\delta'$,
	\begin{align*}
		\nbr{\bs{Z} - \ex [\bs{Z}]} \leq & \nbr{\bs{D} - \ex \mbr{\bs{D}}} + \nbr{\bs{E} - \ex \mbr{\bs{E}}} + \nbr{\bs{F} - \ex \mbr{\bs{F}}}
		\\
		\leq & \frac{64 p L_{\phi} L_{\theta} B_{\Phi}^2 \log \sbr{\frac{4pMT}{\delta'}} \log\sbr{\frac{2d}{\delta'}}}{\sqrt{MT}} .
	\end{align*}
	
	Let $\delta'=\frac{\delta}{25}$. Recall that $B_{\Phi}:=2 \sqrt{\frac{(1+\zeta)}{p \nu}}$. Then, we obtain that with probability at least $1-\frac{\delta}{5}$,
	\begin{align*}
		\nbr{\bs{Z} - \ex [\bs{Z}]} \leq \frac{ 256 (1+\zeta) L_{\phi} L_{\theta} \log\sbr{\frac{50d}{\delta}} }{\nu \sqrt{MT}} \log \sbr{\frac{100pMT}{\delta}} ,
	\end{align*}
	which implies that
	$\Pr[\cG]\geq 1-\frac{\delta}{5}$.
\end{proof}

According to Assumption~\ref{assumption:diverse_task}, there exists an absolute constant $c_0$ which satisfies that $\sigma_{\min}(\frac{1}{M} \sum_{m=1}^{M} \bs{w}_m \bs{w}_m^\top) = \sigma_{\min}(\frac{1}{M} \sum_{m=1}^{M} \bs{\theta}_m \bs{\theta}_m^\top) \geq \frac{c_0}{k}$.

\begin{lemma}[Concentration of $\hat{\bs{B}}$] \label{lemma:concentration_hat_B_clb}
	Suppose that event $\cG$ holds. Then, 
	\begin{align*}
		\nbr{\hat{\bs{B}}_{\bot}^\top \bs{B}} \leq \frac{ 2048 (1+\zeta) k L_{\phi} L_{\theta} \log\sbr{\frac{50d}{\delta}} }{c_0 \nu \sqrt{MT}} \log \sbr{\frac{ 135 (1+\zeta) d L_{\phi} M T}{\nu \delta} } .
	\end{align*}
	Furthermore, if
	\begin{align}
		%
		T = \left \lceil \frac{68 \cdot 2048^2 \cdot 96^2 (1+\zeta)^2 k^4 L_{\phi}^4 L_{\theta}^2 L_{w}^2 }{ c_0^2 \nu^2 \varepsilon^2 M } \log^6 \sbr{ \frac{ 2048 \cdot 135 \cdot 96 \cdot 50 \cdot 5 (1+\zeta)^2 k^2 d^2 L_{\phi}^3 L_{\theta} L_{w} N }{c_0 \nu^2 \delta^3 \varepsilon} } \right \rceil , \label{eq:value_T}
	\end{align}
	we have 
	\begin{align*}
		\nbr{\hat{\bs{B}}_{t,\bot}^\top \bs{B}} \leq \frac{\varepsilon}{ 96 k \log \sbr{\frac{5N}{\delta}} L_{\phi} L_{w} } .
	\end{align*}
\end{lemma}

\begin{proof}[Proof of Lemma~\ref{lemma:concentration_hat_B_clb}]
	First, we have that $\sigma_{k}(\ex[\bs{Z}]) - \sigma_{k+1}(\ex[\bs{Z}])=\sigma_{\min}( \frac{1}{M} \sum_{m=1}^{M} \bs{\theta}_m \bs{\theta}_m^\top )\geq \frac{c_0}{k}$. Let $p := \lceil 32^2 (1+\zeta)^2 \nu^{-2} L_{\phi}^4 \log^2 \sbr{\frac{40dMT}{\delta}} \rceil$.
	Then, using the Davis-Kahan sin $\theta$ Theorem~\cite{bhatia2013matrix} and letting $T_t$ be large enough to satisfy that $\nbr{\bs{Z} - \ex[\bs{Z}]} \leq \frac{c_0}{2k}$, we have
	\begin{align*}
		\nbr{\hat{\bs{B}}_{t,\bot}^\top \bs{B}} \leq & \frac{ \nbr{\bs{Z} - \ex[\bs{Z}]} }{ \sigma_{k}(\ex[\bs{Z}]) - \sigma_{k+1}(\ex[\bs{Z}]) - \nbr{\bs{Z} - \ex[\bs{Z}]} }
		\\
		\leq & \frac{2k}{c_0} \nbr{\bs{Z} - \ex[\bs{Z}]} 
		\\
		\leq & \frac{ 512 (1+\zeta) k L_{\phi} L_{\theta} \log\sbr{\frac{50d}{\delta}} }{c_0 \nu \sqrt{MT}} \log \sbr{\frac{100pMT}{\delta}} 
		\\
		\leq & \frac{ 512 (1+\zeta) k L_{\phi} L_{\theta} \log\sbr{\frac{50d}{\delta}} }{c_0 \nu \sqrt{MT}} \log \sbr{\frac{100MT}{\delta} \cdot \frac{2 \cdot 32^2 (1+\zeta)^2 L_{\phi}^4}{\nu^2} \log^2 \sbr{\frac{40dMT}{\delta}}} 
		\\
		\leq & \frac{ 512 (1+\zeta) k L_{\phi} L_{\theta} \log\sbr{\frac{50d}{\delta}} }{c_0 \nu \sqrt{MT}} \log \sbr{\frac{ 2 \cdot 100 \cdot 32^2 \cdot 40^2 (1+\zeta)^2 d^2 L_{\phi}^4 M^3 T^3}{\nu^2 \delta^3} } 
		\\
		\leq & \frac{ 2048 (1+\zeta) k L_{\phi} L_{\theta} \log\sbr{\frac{50d}{\delta}} }{c_0 \nu \sqrt{MT}} \log \sbr{\frac{ 135 (1+\zeta) d L_{\phi} M T}{\nu \delta} } .
	\end{align*}
	
	Using Lemma~\ref{lemma:technical_tool_bai_stage2} with $A=2048 (1+\zeta) k c_0^{-1} \nu^{-1} L_{\phi} L_{\theta} \log\sbr{\frac{50d}{\delta}}$, $B=\frac{ 135 (1+\zeta) d L_{\phi} }{\nu \delta}$ and $\kappa=\frac{\varepsilon}{ 96 k \log \sbr{\frac{5N}{\delta}} L_{\phi} L_{w}}$, we have that if 
	\begin{align*}
		M T \geq & \frac{68 \cdot 2048^2 \cdot 96^2 (1+\zeta)^2 k^4 L_{\phi}^4 L_{\theta}^2 L_{w}^2 }{ c_0^2 \nu^2 \varepsilon^2 } \cdot \\& \log^2 \sbr{\frac{50d}{\delta}} \log^2 \sbr{\frac{5N}{\delta}}  \log^2 \sbr{ \frac{ 2048 \cdot 135 \cdot 96 (1+\zeta)^2 k^2 d L_{\phi}^3 L_{\theta} L_{w} }{c_0 \nu^2 \delta \varepsilon} \log\sbr{\frac{50d}{\delta}} \log \sbr{\frac{5N}{\delta}} } ,
	\end{align*}
	then $\nbr{\hat{\bs{B}}_{t,\bot}^\top \bs{B}} \leq \frac{\varepsilon}{ 96 k \log \sbr{\frac{5N}{\delta}} L_{\phi} L_{w} }$.
	
	Further enlarging $MT$, if 
	\begin{align*}
		M T \geq \frac{68 \cdot 2048^2 \cdot 96^2 (1+\zeta)^2 k^4 L_{\phi}^4 L_{\theta}^2 L_{w}^2 }{ c_0^2 \nu^2 \varepsilon^2 } \log^6 \sbr{ \frac{ 2048 \cdot 135 \cdot 96 \cdot 50 \cdot 5 (1+\zeta)^2 k^2 d^2 L_{\phi}^3 L_{\theta} L_{w} N }{c_0 \nu^2 \delta^3 \varepsilon} } ,
	\end{align*}
	then
	\begin{align*}
		\nbr{\hat{\bs{B}}_{t,\bot}^\top \bs{B}} \leq \frac{\varepsilon}{ 96 k \log \sbr{\frac{5N}{\delta}} L_{\phi} L_{w} } .
	\end{align*}
\end{proof}

\subsection{Estimation with Low-dimensional Representations}

\begin{lemma} \label{lemma:computation_logdet}
	In subroutine $\algestlowrep$ (Algorithm~\ref{alg:est_low_rep}), for any $m \in [M]$ and $t > 0$, we have 
	\begin{align*}
		\log \sbr{ \frac{\det\sbr{ \gamma I + \sum_{\tau=1}^{t} \hat{\bs{B}}^\top \bs{\phi}(s_{m,\tau},a_{m,\tau}) \bs{\phi}(s_{m,\tau},a_{m,\tau})^\top \hat{\bs{B}} } }{ \det \sbr{ \gamma I} } } \leq k \log \sbr{ 1 + \frac{ t  }{\gamma k} } .
	\end{align*}
\end{lemma}
\begin{proof}[Proof of Lemma~\ref{lemma:computation_logdet}]
	This proof uses a similar idea as Lemma 11 in \cite{abbasi2011improved}.
	
	It holds that
	\begin{align*}
		& \log \sbr{ \frac{\det\sbr{ \gamma I + \sum_{\tau=1}^{t} \hat{\bs{B}}^\top \bs{\phi}(s_{m,\tau},a_{m,\tau}) \bs{\phi}(s_{m,\tau},a_{m,\tau})^\top \hat{\bs{B}} } }{ \det \sbr{ \gamma I} } } 
		\\
		\leq & \log \sbr{ \frac{\sbr{ \frac{ \trace \sbr{\gamma I + \sum_{\tau=1}^{t} \hat{\bs{B}}^\top \bs{\phi}(s_{m,\tau},a_{m,\tau}) \bs{\phi}(s_{m,\tau},a_{m,\tau})^\top \hat{\bs{B}}} }{k}}^{k} }{  \gamma^{k} } } 
		\\
		= & k \log \sbr{ \frac{ \trace \sbr{\gamma I} + \sum_{\tau=1}^{t} \trace \sbr{ \hat{\bs{B}}^\top \bs{\phi}(s_{m,\tau},a_{m,\tau}) \bs{\phi}(s_{m,\tau},a_{m,\tau})^\top \hat{\bs{B}}} }{\gamma k} } 
		\\
		= & k \log \sbr{ \frac{ \gamma k  + \sum_{\tau=1}^{t} \nbr{\hat{\bs{B}}^\top \bs{\phi}(s_{m,\tau},a_{m,\tau})}^2  }{\gamma k} } 
		\\
		\leq & k \log \sbr{ 1 + \frac{ t  }{\gamma k} } .
	\end{align*}
\end{proof}

\begin{lemma} \label{lemma:decreasing_uncertainty}
	In subroutine $\algestlowrep$ (Algorithm~\ref{alg:est_low_rep}), for any $m \in [M]$ and $t \geq 0$, we have 
	\begin{align*}
		\ex_{s \sim \cD} \mbr{\max_{a \in \cA} \nbr{\hat{\bs{B}}^\top \bs{\phi}(s,a)}_{\bs{\Sigma}_{m,t}^{-1}} } \geq \ex_{s \sim \cD} \mbr{\max_{a \in \cA} \nbr{\hat{\bs{B}}^\top \bs{\phi}(s,a)}_{\bs{\Sigma}_{m,t+1}^{-1}} } .
	\end{align*}
\end{lemma}
\begin{proof}[Proof of Lemma~\ref{lemma:decreasing_uncertainty}]
	This proof is similar to that of Lemma 6 in \cite{zanette2021design}.
	
	For any $m \in [M]$ and $t \geq 0$, since $\bs{\Sigma}_{m,t+1} \succeq \bs{\Sigma}_{m,t}$, we have $\bs{\Sigma}_{m,t}^{-1} \succeq \bs{\Sigma}_{m,t+1}^{-1}$. Hence, for any $m \in [M]$, $t \geq 0$, $s \in \cS$ and $a \in \cA$, we have
	\begin{align*}
		\bs{\phi}(s,a)^\top \hat{\bs{B}} \bs{\Sigma}_{m,t}^{-1} \hat{\bs{B}}^\top \bs{\phi}(s,a) \geq \bs{\phi}(s,a)^\top \hat{\bs{B}} \bs{\Sigma}_{m,t+1}^{-1} \hat{\bs{B}}^\top \bs{\phi}(s,a) ,
	\end{align*}
	which implies that
	\begin{align*}
		\nbr{\hat{\bs{B}}^\top \bs{\phi}(s,a)}_{\bs{\Sigma}_{m,t}^{-1}} \geq \nbr{\hat{\bs{B}}^\top \bs{\phi}(s,a)}_{\bs{\Sigma}_{m,t+1}^{-1}} .
	\end{align*}
	
	Therefore, for any $m \in [M]$ and $t \geq 0$, we have
	\begin{align*}
		\ex_{s \sim \cD} \mbr{\max_{a \in \cA} \nbr{\hat{\bs{B}}^\top \bs{\phi}(s,a)}_{\bs{\Sigma}_{m,t}^{-1}} } \geq \ex_{s \sim \cD} \mbr{\max_{a \in \cA} \nbr{\hat{\bs{B}}^\top \bs{\phi}(s,a)}_{\bs{\Sigma}_{m,t+1}^{-1}} } .
	\end{align*}
\end{proof}

In subroutine $\algestlowrep$, for any $m \in [M]$ and $t>0$, let $\xi_{m,t}$ denote the noise of the sample at timestep $t$ for task $m$ (Line~\ref{line:bpi_stage3_sample} in Algorithm~\ref{alg:est_low_rep}).

Define event 
\begin{align*}
	\cH := \Bigg\{ &
	\nbr{\sum_{\tau=1}^{t} \hat{\bs{B}}^\top \bs{\phi}(s_{m,\tau},a_{m,\tau}) \xi_{m,\tau}}_{\sbr{\gamma I + \sum_{\tau=1}^{t} \hat{\bs{B}}^\top \bs{\phi}(s_{m,\tau},a_{m,\tau}) \bs{\phi}(s_{m,\tau},a_{m,\tau})^\top \hat{\bs{B}} }^{-1}} \leq \\&  k \log \sbr{ 1 + \frac{ t  }{\gamma k} } + 2 \log\sbr{\frac{5}{\delta}}, \ \forall m \in [M],\ \forall t>0 \Bigg\} . 
\end{align*}

\begin{lemma}[Martingale Concentration of the Variance Term] \label{lemma:martingale_concentration_variance}
	It holds that
	\begin{align*}
		\Pr \mbr{\cH} \geq 1-\frac{\delta}{5} .
	\end{align*}
\end{lemma}
\begin{proof}[Proof of Lemma~\ref{lemma:martingale_concentration_variance}]
	Let $\delta'$ be a confidence parameter which will be chosen later.
	Since $\hat{\bs{B}}$ is fixed before sampling $(s_{m,\tau},a_{m,\tau})$ for all $m \in [M]$ and $\tau>0$, using Lemma~\ref{lemma:self-normalized_vector_concentration}, we have that with probability at least $1-\delta'$, for any task $m \in [M]$ and $t>0$,
	\begin{align*}
		& \nbr{\sum_{\tau=1}^{t} \hat{\bs{B}}^\top \bs{\phi}(s_{m,\tau},a_{m,\tau}) \xi_{m,j}}_{\sbr{\gamma I + \sum_{\tau=1}^{t} \hat{\bs{B}}^\top \bs{\phi}(s_{m,\tau},a_{m,\tau}) \bs{\phi}(s_{m,\tau},a_{m,\tau})^\top \hat{\bs{B}} }^{-1}} 
		\\
		\leq & 2 \log \sbr{ \frac{\det\sbr{ \gamma I + \sum_{\tau=1}^{t} \hat{\bs{B}}^\top \bs{\phi}(s_{m,\tau},a_{m,\tau}) \bs{\phi}(s_{m,\tau},a_{m,\tau})^\top \hat{\bs{B}} }^{\frac{1}{2}} }{ \det \sbr{ \gamma I}^{\frac{1}{2}} \cdot \delta'} } 
		\\
		\leq & \log \sbr{ \frac{\det\sbr{ \gamma I + \sum_{\tau=1}^{t} \hat{\bs{B}}^\top \bs{\phi}(s_{m,\tau},a_{m,\tau}) \bs{\phi}(s_{m,\tau},a_{m,\tau})^\top \hat{\bs{B}} } }{ \det \sbr{ \gamma I} } } + 2 \log\sbr{\frac{1}{\delta'}}
		\\
		\overset{\textup{(a)}}{\leq} & k \log \sbr{ 1 + \frac{ t  }{\gamma k} } + 2 \log\sbr{\frac{1}{\delta'}} ,
	\end{align*}
	where inequality (a) uses Lemma~\ref{lemma:computation_logdet}.
	
	Letting $\delta'=\frac{\delta}{5}$, we obtain this lemma.
\end{proof}

Define event 
\begin{align*}
	\cJ := \Bigg\{ & \sum_{t=1}^{N} \ex_{s \sim \cD} \mbr{\max_{a \in \cA} \nbr{\hat{\bs{B}}^\top \bs{\phi}(s,a)}_{\bs{\Sigma}_{t-1}^{-1}} } \leq \\& \frac{1}{4} \sbr{ 2\sqrt{\log \sbr{\frac{5}{\delta}}} + \sqrt{ 4 \log \sbr{\frac{5}{\delta}} + 4 \sbr{ \sum_{t=1}^{N} \max_{a \in \cA} \nbr{\hat{\bs{B}}^\top \bs{\phi}(s_t,a) }_{\bs{\Sigma}_{t-1}^{-1}} + 2 \log \sbr{\frac{5}{\delta}} } } }^2 \Bigg\} .
\end{align*}

\begin{lemma} \label{lemma:reverse_bernstein_uncertainty}
	It holds that
	\begin{align*}
		\Pr \mbr{\cJ} \geq 1-\frac{\delta}{5} .
	\end{align*}
\end{lemma}

\begin{proof}[Proof of Lemma~\ref{lemma:reverse_bernstein_uncertainty}]
	Using Lemma~\ref{lemma:reverse_bernstein}, we can obtain this lemma.
\end{proof}

\begin{lemma} \label{lemma:number_of_samples_N}
	Suppose that event $\cK \cap \cL \cap \cG \cap \cH \cap \cJ$ holds. For any task $m \in [M]$, we have
	\begin{align*}
		\ex_{s \sim \cD} & \mbr{\max_{a \in \cA} \abr{\bs{\phi}(s,a)^\top \sbr{ \hat{\bs{\theta}}_{m,N} - \bs{\theta}_m}}} \leq  \sbr{2\sqrt{ \frac{2 k \log \sbr{ 1 + \frac{ N }{\gamma k} }}{N} } + \frac{8 \log \sbr{\frac{5}{\delta}} }{N}} \cdot \\& \sbr{ \nbr{\hat{\bs{B}}_{\perp}^\top \bs{B}} \sqrt{Nk} +  \sqrt{k \log \sbr{ 1 + \frac{ N  }{\gamma k} } + 2 \log\sbr{\frac{5}{\delta}}} + \sqrt{\gamma} } + \nbr{\hat{\bs{B}}_{\bot}^\top \bs{B}} .
	\end{align*}
	Furthermore, if 
	\begin{align*}
		%
		N = \left \lceil \frac{4^2 \cdot 26^4 \cdot 24^2 \cdot 2 \sbr{k^2 + k \gamma L_{\theta}^2 } \log^4 \big({\frac{240 ( k + \sqrt{k \gamma} L_{\theta}) }{\varepsilon \delta}} \big) }{\varepsilon^2} \right \rceil , 
	\end{align*}
	then
	\begin{align*}
		\ex_{s \sim \cD} \mbr{\max_{a \in \cA} \abr{\bs{\phi}(s,a)^\top \sbr{ \hat{\bs{\theta}}_{m,N} - \bs{\theta}_m}}} \leq \frac{\varepsilon}{2}
	\end{align*}
\end{lemma}

\begin{proof}[Proof of Lemma~\ref{lemma:number_of_samples_N}]
	For any task $m \in [M]$ and $t \in [N]$,
	\begin{align*}
		\hat{\bs{w}}_{m,t} = & \bs{\Sigma}_{m,t}^{-1} \sum_{\tau=1}^{t} \hat{\bs{B}}^\top \bs{\phi}(s_{m,\tau},a_{m,\tau}) r_{m,\tau}
		\\
		= & \sbr{\gamma I + \sum_{\tau=1}^{t} \hat{\bs{B}}^\top \bs{\phi}(s_{m,\tau},a_{m,\tau}) \bs{\phi}(s_{m,\tau},a_{m,\tau})^\top \hat{\bs{B}} }^{-1} \sum_{\tau=1}^{t} \hat{\bs{B}}^\top \bs{\phi}(s_{m,\tau},a_{m,\tau}) \sbr{ \bs{\phi}(s_{m,\tau},a_{m,\tau})^\top \bs{\theta}_m + \xi_{m,j} }
		\\
		= & \sbr{\gamma I + \sum_{\tau=1}^{t} \hat{\bs{B}}^\top \bs{\phi}(s_{m,\tau},a_{m,\tau}) \bs{\phi}(s_{m,\tau},a_{m,\tau})^\top \hat{\bs{B}} }^{-1} \sum_{\tau=1}^{t} \hat{\bs{B}}^\top \bs{\phi}(s_{m,\tau},a_{m,\tau}) \cdot \\& \sbr{ \bs{\phi}(s_{m,\tau},a_{m,\tau})^\top \hat{\bs{B}} \hat{\bs{B}}^\top \bs{\theta}_m + \bs{\phi}(s_{m,\tau},a_{m,\tau})^\top \hat{\bs{B}}_{\perp} \hat{\bs{B}}_{\perp}^\top \bs{\theta}_m + \xi_{m,j} }
		\\& + \gamma \sbr{\gamma I + \sum_{\tau=1}^{t} \hat{\bs{B}}^\top \bs{\phi}(s_{m,\tau},a_{m,\tau}) \bs{\phi}(s_{m,\tau},a_{m,\tau})^\top \hat{\bs{B}} }^{-1} \hat{\bs{B}}^\top \bs{\theta}_m 
		\\& - \gamma \sbr{\gamma I + \sum_{\tau=1}^{t} \hat{\bs{B}}^\top \bs{\phi}(s_{m,\tau},a_{m,\tau}) \bs{\phi}(s_{m,\tau},a_{m,\tau})^\top \hat{\bs{B}} }^{-1} \hat{\bs{B}}^\top \bs{\theta}_m
		\\
		= & \hat{\bs{B}}^\top \bs{\theta}_m + \sbr{\gamma I + \sum_{\tau=1}^{t} \hat{\bs{B}}^\top \bs{\phi}(s_{m,\tau},a_{m,\tau}) \bs{\phi}(s_{m,\tau},a_{m,\tau})^\top \hat{\bs{B}} }^{-1} \cdot
		\\&
		\sum_{\tau=1}^{t} \hat{\bs{B}}^\top \bs{\phi}(s_{m,\tau},a_{m,\tau})  \bs{\phi}(s_{m,\tau},a_{m,\tau})^\top \hat{\bs{B}}_{\perp} \hat{\bs{B}}_{\perp}^\top \bs{B} \bs{w}_m
		\\&
		+ \sbr{\gamma I + \sum_{\tau=1}^{t} \hat{\bs{B}}^\top \bs{\phi}(s_{m,\tau},a_{m,\tau}) \bs{\phi}(s_{m,\tau},a_{m,\tau})^\top \hat{\bs{B}} }^{-1} \sum_{\tau=1}^{t} \hat{\bs{B}}^\top \bs{\phi}(s_{m,\tau},a_{m,\tau}) \xi_{m,j}
		\\& 
		- \gamma \sbr{\gamma I + \sum_{\tau=1}^{t} \hat{\bs{B}}^\top \bs{\phi}(s_{m,\tau},a_{m,\tau}) \bs{\phi}(s_{m,\tau},a_{m,\tau})^\top \hat{\bs{B}} }^{-1} \hat{\bs{B}}^\top \bs{\theta}_m .
	\end{align*}

	Hence, for any task $m \in [M]$, $t \in [N]$ and $(s,a) \in \cS \times \cA$, 
	\begin{align*}
		\bs{\phi}(s,a)^\top \sbr{ \hat{\bs{\theta}}_{m,t} - \bs{\theta}_m} = & \bs{\phi}(s,a)^\top \hat{\bs{B}} \hat{\bs{w}}_{m,t} - \bs{\phi}(s,a)^\top \sbr{\hat{\bs{B}}\hat{\bs{B}}^\top+\hat{\bs{B}}_{\bot}\hat{\bs{B}}_{\bot}^\top} \bs{\theta}_m
		\\
		= & \bs{\phi}(s,a)^\top \hat{\bs{B}} \sbr{ \hat{\bs{w}}_{m,t} - \hat{\bs{B}}^\top \bs{\theta}_m } - \bs{\phi}(s,a)^\top \hat{\bs{B}}_{\bot}\hat{\bs{B}}_{\bot}^\top \bs{\theta}_m
		\\
		= & \bs{\phi}(s,a)^\top \hat{\bs{B}} 
		\sbr{\gamma I + \sum_{\tau=1}^{t} \hat{\bs{B}}^\top \bs{\phi}(s_{m,\tau},a_{m,\tau}) \bs{\phi}(s_{m,\tau},a_{m,\tau})^\top \hat{\bs{B}} }^{-1} \cdot \\& \sum_{\tau=1}^{t} \hat{\bs{B}}^\top \bs{\phi}(s_{m,\tau},a_{m,\tau})  \bs{\phi}(s_{m,\tau},a_{m,\tau})^\top \hat{\bs{B}}_{\perp} \hat{\bs{B}}_{\perp}^\top \bs{B} \bs{w}_m
		\\&
		+ \bs{\phi}(s,a)^\top \hat{\bs{B}}  \sbr{\gamma I + \sum_{\tau=1}^{t} \hat{\bs{B}}^\top \bs{\phi}(s_{m,\tau},a_{m,\tau}) \bs{\phi}(s_{m,\tau},a_{m,\tau})^\top \hat{\bs{B}} }^{-1} \sum_{\tau=1}^{t} \hat{\bs{B}}^\top \bs{\phi}(s_{m,\tau},a_{m,\tau}) \xi_{m,j}
		\\& 
		-\! \gamma \bs{\phi}(s,a)^{\!\top} \! \hat{\bs{B}}  \sbr{ \! \gamma I \!+\! \sum_{\tau=1}^{t} \hat{\bs{B}}^{\!\top} \bs{\phi}(s_{m,\tau},a_{m,\tau}) \bs{\phi}(s_{m,\tau},a_{m,\tau})^{\!\top} \hat{\bs{B}} \! }^{\!\!-1} \!\!\!\!\!\! \hat{\bs{B}}^\top \! \bs{\theta}_m  \!-\! \bs{\phi}(s,a)^{\!\top} \! \hat{\bs{B}}_{\bot}\hat{\bs{B}}_{\bot}^{\!\top} \bs{B} \bs{w}_m .
	\end{align*}
	
	For any $m \in [M]$, let $\bs{\Sigma}_{m,0}:=\gamma I$. For any $m \in [M]$ and $t \geq 1$, let $\bs{\Sigma}_{m,t}:=\gamma I + \sum_{\tau=1}^{t} \hat{\bs{B}}^\top \bs{\phi}(s_{m,\tau},a_{m,\tau}) \bs{\phi}(s_{m,\tau},a_{m,\tau})^\top \hat{\bs{B}}$.
	
	Taking the absolute value on both sides and using the Cauchy–Schwarz inequality, we obtain that for any $m \in [M]$, $t \in [N]$ and $(s,a) \in \cS \times \cA$, 
	\begin{align*}
		& \abr{\bs{\phi}(s,a)^\top \sbr{ \hat{\bs{\theta}}_{m,t} - \bs{\theta}_m}} 
		\\
		\leq & \nbr{\hat{\bs{B}}^\top \bs{\phi}(s,a)}_{\bs{\Sigma}_{m,t}^{-1}} \nbr{\sum_{\tau=1}^{t} \hat{\bs{B}}^\top \bs{\phi}(s_{m,\tau},a_{m,\tau})  \bs{\phi}(s_{m,\tau},a_{m,\tau})^\top \hat{\bs{B}}_{\perp} \hat{\bs{B}}_{\perp}^\top \bs{B} \bs{w}_m}_{\bs{\Sigma}_{m,t}^{-1}}
		\\&
		+ \nbr{\hat{\bs{B}}^\top \bs{\phi}(s,a)}_{\bs{\Sigma}_{m,t}^{-1}}  \nbr{\sum_{\tau=1}^{t} \hat{\bs{B}}^\top \bs{\phi}(s_{m,\tau},a_{m,\tau}) \xi_{m,j}}_{\bs{\Sigma}_{m,t}^{-1}}
		\\& 
		+ \gamma \nbr{\hat{\bs{B}}^\top \bs{\phi}(s,a)}_{\bs{\Sigma}_{m,t}^{-1}}  \nbr{\hat{\bs{B}}^\top \bs{\theta}_m}_{\bs{\Sigma}_{m,t}^{-1}} 
		\\& 
		+ \abr{\bs{\phi}(s,a)^\top \hat{\bs{B}}_{\bot}\hat{\bs{B}}_{\bot}^\top \bs{B} \bs{w}_m}
		\\
		\overset{\textup{(a)}}{\leq} & \nbr{\hat{\bs{B}}^\top \bs{\phi}(s,a)}_{\bs{\Sigma}_{m,t}^{-1}} \sum_{\tau=1}^{t} \abr{ \bs{\phi}(s_{m,\tau},a_{m,\tau})^\top \hat{\bs{B}}_{\perp} \hat{\bs{B}}_{\perp}^\top \bs{B} \bs{w}_m } \cdot \nbr{\hat{\bs{B}}^\top \bs{\phi}(s_{m,\tau},a_{m,\tau})}_{\bs{\Sigma}_{m,t}^{-1}}
		\\&
		+ \nbr{\hat{\bs{B}}^\top \bs{\phi}(s,a)}_{\bs{\Sigma}_{m,t}^{-1}}  \sqrt{k \log \sbr{ 1 + \frac{ t  }{\gamma k} } + 2 \log\sbr{\frac{5}{\delta}}}
		\\& 
		+ \gamma \nbr{\hat{\bs{B}}^\top \bs{\phi}(s,a)}_{\bs{\Sigma}_{m,t}^{-1}} \cdot \frac{1}{\sqrt{\gamma}} \cdot \nbr{\hat{\bs{B}}^\top \bs{\theta}_m} + \nbr{\hat{\bs{B}}_{\bot}^\top \bs{B}} L_{\phi} L_{w}
		\\
		\leq & \nbr{\hat{\bs{B}}^\top \bs{\phi}(s,a)}_{\bs{\Sigma}_{m,t}^{-1}} \cdot \nbr{\hat{\bs{B}}_{\perp}^\top \bs{B}} L_{\phi} L_{w} \cdot \sum_{\tau=1}^{t}  \nbr{\hat{\bs{B}}^\top \bs{\phi}(s_{m,\tau},a_{m,\tau})}_{\bs{\Sigma}_{m,t}^{-1}}
		\\&
		+ \nbr{\hat{\bs{B}}^\top \bs{\phi}(s,a)}_{\bs{\Sigma}_{m,t}^{-1}}  \sqrt{k \log \sbr{ 1 + \frac{ t  }{\gamma k} } + 2 \log\sbr{\frac{5}{\delta}}}
		\\& 
		+ \sqrt{\gamma} L_{\theta} \nbr{\hat{\bs{B}}^\top \bs{\phi}(s,a)}_{\bs{\Sigma}_{m,t}^{-1}} + \nbr{\hat{\bs{B}}_{\bot}^\top \bs{B}} L_{\phi} L_{w}
		\\
		\overset{\textup{(b)}}{\leq} & \nbr{\hat{\bs{B}}^\top \bs{\phi}(s,a)}_{\bs{\Sigma}_{m,t}^{-1}} \cdot \nbr{\hat{\bs{B}}_{\perp}^\top \bs{B}} L_{\phi} L_{w} \cdot \sqrt{tk}
		\\&
		+ \nbr{\hat{\bs{B}}^\top \bs{\phi}(s,a)}_{\bs{\Sigma}_{m,t}^{-1}}  \sqrt{k \log \sbr{ 1 + \frac{ t  }{\gamma k} } + 2 \log\sbr{\frac{5}{\delta}}}
		\\& 
		+ \sqrt{\gamma} L_{\theta} \nbr{\hat{\bs{B}}^\top \bs{\phi}(s,a)}_{\bs{\Sigma}_{m,t}^{-1}} + \nbr{\hat{\bs{B}}_{\bot}^\top \bs{B}} L_{\phi} L_{w}
		\\
		= & \nbr{\hat{\bs{B}}^\top \bs{\phi}(s,a)}_{\bs{\Sigma}_{m,t}^{-1}} \sbr{ \nbr{\hat{\bs{B}}_{\perp}^\top \bs{B}} L_{\phi} L_{w} \cdot \sqrt{tk} + \sqrt{k \log \sbr{ 1 + \frac{ t  }{\gamma k} } + 2 \log\sbr{\frac{5}{\delta}}} + \sqrt{\gamma} L_{\theta} } \\& + \nbr{\hat{\bs{B}}_{\bot}^\top \bs{B}} L_{\phi} L_{w} ,
	\end{align*}
	where inequality (a) uses the triangle inequality and the definition of event $\cH$, and inequality (b) is due to Lemma~\ref{lemma:sqrt_n_k_gamma}.

	Taking the maximum over $a \in \cA$ and taking the expectation on $s \sim \cD$, we have that for any task $m \in [M]$,
	\begin{align}
		\ex_{s \sim \cD} \mbr{\max_{a \in \cA} \abr{\bs{\phi}(s,a)^\top \sbr{ \hat{\bs{\theta}}_{m,N} - \bs{\theta}_m}}} \leq & \ex_{s \sim \cD} \mbr{\max_{a \in \cA} \nbr{\hat{\bs{B}}^\top \bs{\phi}(s,a)}_{\bs{\Sigma}_N^{-1}} } \cdot
		\nonumber\\&
		\sbr{ \nbr{\hat{\bs{B}}_{\perp}^\top \bs{B}} L_{\phi} L_{w} \cdot \sqrt{Nk} + \sqrt{k \log \sbr{ 1 + \frac{ N  }{\gamma k} } + 2 \log\sbr{\frac{5}{\delta}}} + \sqrt{\gamma} L_{\theta} } 
		\nonumber\\& 
		+ \nbr{\hat{\bs{B}}_{\bot}^\top \bs{B}} L_{\phi} L_{w} . \label{eq:ex_max_phi_theta}
	\end{align}

	According to Lemma~\ref{lemma:decreasing_uncertainty}, $\ex_{s \sim \cD} \mbr{\max_{a \in \cA} \nbr{\hat{\bs{B}}^\top \bs{\phi}(s,a)}_{\bs{\Sigma}_t^{-1}} }$ is non-increasing with respect to $t$. Hence, we have
	\begin{align}
		\ex_{s \sim \cD} \mbr{\max_{a \in \cA} \nbr{\hat{\bs{B}}^\top \bs{\phi}(s,a)}_{\bs{\Sigma}_{N}^{-1}} } \leq & \frac{1}{N} \sum_{t=1}^{N} \ex_{s \sim \cD} \mbr{\max_{a \in \cA} \nbr{\hat{\bs{B}}^\top \bs{\phi}(s,a)}_{\bs{\Sigma}_{t}^{-1}} }
		\nonumber\\
		\leq & \frac{1}{N} \sum_{t=1}^{N} \ex_{s \sim \cD} \mbr{\max_{a \in \cA} \nbr{\hat{\bs{B}}^\top \bs{\phi}(s,a)}_{\bs{\Sigma}_{t-1}^{-1}} } 
		\nonumber\\
		\overset{\textup{(a)}}{\leq} & \frac{1}{4N} \Bigg( 2\sqrt{\log \sbr{\frac{5}{\delta}}} \nonumber\\& + \sqrt{ 4 \log \sbr{\frac{5}{\delta}} + 4 \sbr{ \sum_{t=1}^{N} \max_{a \in \cA} \nbr{\hat{\bs{B}}^\top \bs{\phi}(s_t,a) }_{\bs{\Sigma}_{t-1}^{-1}} + 2 \log \sbr{\frac{5}{\delta}} } } \Bigg)^2 
		\nonumber\\
		= & \frac{1}{4N} \sbr{ 2\sqrt{\log \sbr{\frac{5}{\delta}}} + \sqrt{ 4 \log \sbr{\frac{5}{\delta}} + 4 \sbr{ \sum_{t=1}^{N} \nbr{\hat{\bs{B}}^\top \bs{\phi}(s_t,a_t)}_{\bs{\Sigma}_{t-1}^{-1}} + 2 \log \sbr{\frac{5}{\delta}} } } }^{\!\!\!2} \!\! , \label{eq:ex_max_B_phi}
	\end{align}
	where inequality (a) is due to the definition of event $\cJ$.
	
	In addition, we have 
	\begin{align}
		\sum_{t=1}^{N} \nbr{\hat{\bs{B}}^\top \bs{\phi}(s_t,a_t)}_{\bs{\Sigma}_{t-1}^{-1}} \leq & \sqrt{N \cdot \sum_{t=1}^{N} \nbr{\hat{\bs{B}}^\top \bs{\phi}(s_t,a_t)}^2_{\bs{\Sigma}_{t-1}^{-1}}}
		\nonumber\\
		\overset{\textup{(a)}}{\leq} & \sqrt{2N \log \sbr{\frac{\det \sbr{\gamma I + \sum_{\tau=1}^{N} \hat{\bs{B}}^\top \bs{\phi}(s_{m,\tau},a_{m,\tau}) \bs{\phi}(s_{m,\tau},a_{m,\tau})^\top \hat{\bs{B}}}}{\det \sbr{\gamma I}}}}
		\nonumber\\
		\overset{\textup{(b)}}{\leq} & \sqrt{2 N k \log \sbr{ 1 + \frac{ N }{\gamma k} }} , \label{eq:B_phi_leq_k_log}
	\end{align}
	where inequality (a) uses Lemma~\ref{lemma:elliptical_potential}, and inequality (b) is due to Lemma~\ref{lemma:computation_logdet}.

	Combining Eqs.~\eqref{eq:ex_max_B_phi} and \eqref{eq:B_phi_leq_k_log}, we have
	\begin{align}
		\ex_{s \sim \cD} \mbr{\max_{a \in \cA} \nbr{\hat{\bs{B}}^\top \bs{\phi}(s,a)}_{\bs{\Sigma}_{N}^{-1}} } \leq &
		\frac{1}{4N} \sbr{ 2\sqrt{\log \sbr{\frac{5}{\delta}}} + \sqrt{ 4 \log \sbr{\frac{5}{\delta}} + 4 \sbr{ \sqrt{2 N k \log \sbr{ 1 + \frac{ N }{\gamma k} }} + 2 \log \sbr{\frac{5}{\delta}} } } }^2
		\nonumber\\
		\overset{\textup{(a)}}{\leq} & \frac{1}{2N} \sbr{ 4 \log \sbr{\frac{5}{\delta}} + 4 \log \sbr{\frac{5}{\delta}} + 4 \sbr{ \sqrt{2 N k \log \sbr{ 1 + \frac{ N }{\gamma k} }} + 2 \log \sbr{\frac{5}{\delta}} } }
		\nonumber\\
		= & \frac{1}{N} \sbr{ 2\sqrt{2 N k \log \sbr{ 1 + \frac{ N }{\gamma k} }} + 8 \log \sbr{\frac{5}{\delta}} }
		\nonumber\\
		= &  2\sqrt{ \frac{2 k \log \sbr{ 1 + \frac{ N }{\gamma k} }}{N} } + \frac{8 \log \sbr{\frac{5}{\delta}} }{N} , \label{eq:sum_ex_max_B_phi_final_bound}
	\end{align}
	where inequality (a) uses the Cauchy–Schwarz inequality.
	
	Furthermore, plugging Eq.~\eqref{eq:sum_ex_max_B_phi_final_bound} into Eq.~\eqref{eq:ex_max_phi_theta} and using $\gamma \geq 1$, we have that for $N \geq 1$ and $\sqrt{k}\log(2N) \geq 1$,

	\begin{align}
		& \ex_{s \sim \cD} \mbr{\max_{a \in \cA} \abr{\bs{\phi}(s,a)^\top \sbr{ \hat{\bs{\theta}}_{m,N} - \bs{\theta}_m}}} 
		\\
		\leq & \sbr{2\sqrt{ \frac{2 k \log \sbr{ 1 + \frac{ N }{\gamma k} }}{N} } + \frac{8 \log \sbr{\frac{5}{\delta}} }{N}} \cdot \nonumber\\& \sbr{ \nbr{\hat{\bs{B}}_{\perp}^\top \bs{B}} L_{\phi} L_{w} \sqrt{N k} + \sqrt{k \log \sbr{ 1 + \frac{ N }{\gamma k} } + 2 \log\sbr{\frac{5}{\delta}}} + \sqrt{\gamma} L_{\theta} } + \nbr{\hat{\bs{B}}_{\bot}^\top \bs{B}} L_{\phi} L_{w}
		\nonumber\\
		\leq & \frac{ 12 \sqrt{k} \log \sbr{\frac{5N}{\delta}} }{\sqrt{N}} \sbr{ \nbr{\hat{\bs{B}}_{\perp}^\top \bs{B}} L_{\phi} L_{w} \sqrt{Nk} +  2\sqrt{k} \log\sbr{\frac{5N}{\delta}} + \sqrt{\gamma} L_{\theta} } + \nbr{\hat{\bs{B}}_{\bot}^\top \bs{B}} L_{\phi} L_{w}
		\nonumber\\
		\leq &  \frac{ \sbr{24 k + 12 \sqrt{k \gamma} L_{\theta} } \log^2 \sbr{\frac{5N}{\delta}} }{\sqrt{N}} +  24 k \log \sbr{\frac{5N}{\delta}} \nbr{\hat{\bs{B}}_{\perp}^\top \bs{B}} L_{\phi} L_{w} . \label{eq:ex_max_phi_theta_B_hat_times_B}
	\end{align}

	Using Lemma~\ref{lemma:technical_tool_log_N_square} with $A=24 k + 12 \sqrt{k \gamma} L_{\theta}$, $B=\frac{5}{\delta}$ and $\kappa=\frac{\varepsilon}{4}$, we have that if
	\begin{align*}
		N \geq \frac{26^4 \sbr{24 k + 12 \sqrt{k \gamma} L_{\theta}}^2 \log^4 \big({\frac{2 \cdot 5 (24 k + 12 \sqrt{k \gamma} L_{\theta}) }{\varepsilon \delta}} \big) }{\sbr{\frac{\varepsilon}{4}}^2} ,
	\end{align*} 
	then $\frac{ \sbr{24 k + 12 \sqrt{k \gamma} L_{\theta} } \log^2 \sbr{\frac{5N}{\delta}} }{\sqrt{N}} \leq \frac{\varepsilon}{4}$.
	
	Further enlarging $N$, if 
	\begin{align}
		N \geq \frac{4^2 \cdot 26^4 \cdot 24^2 \cdot 2 \sbr{k^2 + k \gamma L_{\theta}^2 } \log^4 \big({\frac{240 ( k + \sqrt{k \gamma} L_{\theta}) }{\varepsilon \delta}} \big) }{\varepsilon^2} , \label{eq:number_of_samples_N}
	\end{align}
	then
	\begin{align*}
		\frac{ \sbr{24 k + 12 \sqrt{k \gamma} L_{\theta} } \log^2 \sbr{\frac{5N}{\delta}} }{\sqrt{N}} \leq \frac{\varepsilon}{4} .
	\end{align*}
	
	According to Lemma~\ref{lemma:concentration_hat_B_clb}, we have $\nbr{\hat{\bs{B}}_{t,\bot}^\top \bs{B}} \leq \frac{\varepsilon}{ 96 k \log \sbr{\frac{5N}{\delta}} L_{\phi} L_{w} }$.
	
	Thus, setting $N$ as the value in Eq.~\eqref{eq:number_of_samples_N}, and continuing with Eq.~\eqref{eq:ex_max_phi_theta_B_hat_times_B}, we have
	\begin{align*}
		\ex_{s \sim \cD} \mbr{\max_{a \in \cA} \abr{\bs{\phi}(s,a)^\top \sbr{ \hat{\bs{\theta}}_{m,N} - \bs{\theta}_m}}} \leq  \frac{\varepsilon}{4} + \frac{\varepsilon}{4}
		=  \frac{\varepsilon}{2} .
	\end{align*}
	
\end{proof}


\subsection{Proof of Theorem~\ref{thm:bpi_ub}}

\begin{proof}[Proof of Theorem~\ref{thm:bpi_ub}]
	Combining Lemmas~\ref{lemma:number_of_samples_T0}, \ref{lemma:number_of_samples_p}, \ref{lemma:Z_est_error}, \ref{lemma:martingale_concentration_variance} and \ref{lemma:reverse_bernstein_uncertainty}, we have that $Pr[\cK \cap \cL \cap \cG \cap \cH \cap \cJ] \geq 1-\delta$.
	Suppose that event $\cK \cap \cL \cap \cG \cap \cH \cap \cJ$ holds.
	
	First, we uses a similar analytical procedure as that
	in \cite{zanette2021design} to prove the correctness.
	
	Using Lemma~\ref{lemma:number_of_samples_N}, we have that for any task $m \in [M]$,
	\begin{align*}
		\ex_{s \sim \cD} \mbr{\max_{a \in \cA} \abr{\bs{\phi}(s,a)^\top \sbr{ \hat{\bs{\theta}}_{m,N} - \bs{\theta}_m}}} \leq \frac{\varepsilon}{2} .
	\end{align*}
	
	For any $m \in [M]$ and $s \in \cS$, let $\beta_m(s):=\max_{a\in\cA}|\bs{\phi}(s,a)^\top (\hat{\bs{\theta}}_{m,N}-\bs{\theta}_m)|$ and $\pi^*_m(s):=\argmax_{a \in \cA} \bs{\phi}(s,a)^\top \bs{\theta}_m$.
	
	For any $m \in [M]$ and $s \in \cS$, we have
	\begin{align*}
		\bs{\phi}(s,\hat{\pi}_m(s))^\top \bs{\theta}_m \geq & \bs{\phi}(s,\hat{\pi}_m(s))^\top \hat{\bs{\theta}}_{m,N} - \beta_m(s)
		\\
		\overset{\textup{(a)}}{\geq} & \bs{\phi}(s,\pi^*_m(s))^\top \hat{\bs{\theta}}_{m,N} - \beta_m(s)
		\\
		\geq & \bs{\phi}(s,\pi^*_m(s))^\top \bs{\theta}_m - 2\beta_m(s) ,
	\end{align*}
	where inequality (a) is due to that $\hat{\pi}_m(s)$ is greedy with respect to $\hat{\bs{\theta}}_{m,N}$.
	
	Rearranging the above equation and taking the expectation of $s$ on both sides, we have
	\begin{align*}
		\ex_{s\sim\cD}\mbr{\max_{a \in \cA} \sbr{\bs{\phi}(s,a) - \bs{\phi}(s,\hat{\pi}_m(s))}^\top \bs{\theta}_m }
		\leq  2 \ex_{s\sim\cD}[\beta_m(s)] \leq \varepsilon .
	\end{align*}

	Now we prove the sample complexity.
	Summing the number of samples used in the main algorithm of $\algrepbpiclb$ and subroutines $\algconfeatrecover$ and $\algestlowrep$ (Line~\ref{line:bpi_estimate_context_dis} in Algorithm~\ref{alg:repbpiclb}, Lines~\ref{line:bpi_stage2_sample1}-\ref{line:bpi_stage2_sample2} in Algorithm~\ref{alg:con_feat_recover} and Line~\ref{line:bpi_stage3_sample} in Algorithm~\ref{alg:est_low_rep}), we have that the total number of samples is bounded by
	\begin{align*}
		& T_0+2MTp+MN 
		\\
		= & O \Bigg( \frac{L_{\phi}^4}{\nu^2} \log^2 \sbr{\frac{d |\cA|}{\delta}} + \frac{ k^4 L_{\phi}^4 L_{\theta}^2 L_{w}^2 }{ \nu^2 \varepsilon^2 } \log^6 \sbr{ \frac{ k d L_{\phi} L_{\theta} L_{w} N }{\nu \delta \varepsilon} } \cdot \frac{L_{\phi}^4}{\nu^2} \log^2 \sbr{\frac{dMT}{\delta}} \\&+ M \cdot \frac{\sbr{k^2 + k \gamma L_{\theta}^2 } \log^4 \big({\frac{ k + \sqrt{k \gamma} L_{\theta} }{\varepsilon \delta}} \big) }{\varepsilon^2} \Bigg)
		\\
		= & O \Bigg( \frac{ k^4 L_{\phi}^4 L_{\theta}^2 L_{w}^2 }{ \nu^2 \varepsilon^2 } \log^6 \sbr{ \frac{ |\cA| k d L_{\phi} L_{\theta} L_{w} N }{\nu \delta \varepsilon} } \cdot \frac{L_{\phi}^4}{\nu^2} \log^2 \sbr{\frac{dMT}{\delta}} \\&+ M \cdot \frac{\sbr{k^2 + k \gamma L_{\theta}^2 } \log^4 \big({\frac{ k + \sqrt{k \gamma} L_{\theta} }{\varepsilon \delta}} \big) }{\varepsilon^2} \Bigg)
		\\
		= & \tilde{O} \Bigg( \frac{ k^4 L_{\phi}^8 L_{\theta}^2 L_{w}^2 }{ \nu^4 \varepsilon^2 } + \frac{ M \sbr{k^2 + k \gamma L_{\theta}^2 } }{\varepsilon^2} \Bigg) .
	\end{align*}
\end{proof}

\section{Technical Tools}

In this section, we provide some useful technical tools.

\begin{lemma}[Matrix Bernstern Inequality - Average, Lemma 31 in \cite{tripuraneni2021provable}] \label{lemma:matrix_bernstein_tripuraneni2021}
	Consider a truncation level $U>0$. If $\{\bs{Z}_1,\dots,\bs{Z}_n\}$ is a sequence of $d_1 \times d_2$ independent random matrices and 
	$\bs{Z}'_i = \bs{Z}_i \cdot \indicator{\|\bs{Z}_i\| \leq U}$ for any $i \in [n]$, then
	\begin{align*}
		\Pr \mbr{ \nbr{ \frac{1}{n} \sum_{i=1}^{n} \sbr{ \bs{Z}_i - \ex[\bs{Z}_i] } } \geq t } \leq 
		\Pr \mbr{ \nbr{ \frac{1}{n} \sum_{i=1}^{n} \sbr{ \bs{Z}'_i - \ex[\bs{Z}'_i] } } \geq t - \Delta } + n \Pr \mbr{ \|\bs{Z}_i\| \geq U } ,
	\end{align*}
	where $\Delta \geq \|\ex[\bs{Z}_i]-\ex[\bs{Z}'_i]\|$ for any $i \in [n]$. 
	
	In addition, for $t \geq \Delta$, we have
	\begin{align*}
		\Pr \mbr{ \nbr{ \frac{1}{n} \sum_{i=1}^{n} \sbr{ \bs{Z}'_i - \ex[\bs{Z}'_i] } } \geq t - \Delta } \leq (d_1 + d_2) \exp \sbr{- \frac{n^2 (t-\Delta)^2}{2\sigma^2 + \frac{2Un(t-\Delta)}{3}}} ,
	\end{align*}
	where
	\begin{align*}
		\sigma^2 = & \max \lbr{ \nbr{\sum_{i=1}^{n} \ex[(\bs{Z}'_i-\ex[\bs{Z}'_i])^\top (\bs{Z}'_i-\ex[\bs{Z}'_i])]},\ \nbr{\sum_{i=1}^{n} \ex[(\bs{Z}'_i-\ex[\bs{Z}'_i]) (\bs{Z}'_i-\ex[\bs{Z}'_i])^\top]} } 
		\\
		\leq & \max \lbr{ \nbr{\sum_{i=1}^{n} \ex[{\bs{Z}'_i}^\top \bs{Z}'_i]},\ \nbr{\sum_{i=1}^{n} \ex[\bs{Z}'_i {\bs{Z}'_i}^\top]} } .
	\end{align*} 
\end{lemma}

Lemma 31 in \cite{tripuraneni2021provable} gives a truncated matrix Bernstern inequality for symmetric random matrices. Here we extend it to general random matrices.

Lemma~\ref{lemma:matrix_bernstein_tripuraneni2021} can be obtained by combining the truncation argument in the proof of Lemma 31 in \cite{tripuraneni2021provable} and Theorem 6.1.1 in \cite{tropp2015introduction} (classic matrix Bernstern inequality for general random matrices).

\begin{lemma}[Matrix Bernstern Inequality - Summation] \label{lemma:matrix_bernstein_tau}
	Consider a truncation level $U>0$. If $\{\bs{Z}_1,\dots,\bs{Z}_n\}$ is a sequence of $d_1 \times d_2$ independent random matrices, and $\bs{Z}'_i = \bs{Z}_i \cdot \indicator{\|\bs{Z}_i\| \leq U}$ and $\Delta \geq \|\ex[\bs{Z}_i]-\ex[\bs{Z}'_i]\|$ for any $i \in [n]$, then for $\tau \geq 2 n\Delta$,
	\begin{align*}
		\Pr \mbr{ \nbr{ \sum_{i=1}^{n} \sbr{ \bs{Z}_i - \ex[\bs{Z}_i] } } \geq \tau } \leq 
		(d_1 + d_2) \exp \sbr{- \frac{1}{4} \cdot \frac{ \tau^2}{2\sigma^2 + \frac{U \tau}{3}}} + n \Pr \mbr{ \|\bs{Z}_i\| \geq U } ,
	\end{align*}
	where 
	\begin{align*}
		\sigma^2 = & \max \lbr{ \nbr{\sum_{i=1}^{n} \ex[(\bs{Z}'_i-\ex[\bs{Z}'_i])^\top (\bs{Z}'_i-\ex[\bs{Z}'_i])]},\ \nbr{\sum_{i=1}^{n} \ex[(\bs{Z}'_i-\ex[\bs{Z}'_i]) (\bs{Z}'_i-\ex[\bs{Z}'_i])^\top]} } 
		\\
		\leq & \max \lbr{ \nbr{\sum_{i=1}^{n} \ex[{\bs{Z}'_i}^\top \bs{Z}'_i]},\ \nbr{\sum_{i=1}^{n} \ex[\bs{Z}'_i {\bs{Z}'_i}^\top]} } .
	\end{align*}
	
	Furthermore, we have
	\begin{align*}
		\Pr \mbr{ \nbr{ \sum_{i=1}^{n} \sbr{ \bs{Z}_i - \ex[\bs{Z}_i] } } \geq 4 \sqrt{ \sigma^2 \log \sbr{\frac{d_1 + d_2}{\delta}} } + 4 U \log \sbr{\frac{d_1 + d_2}{\delta}} } \leq 
		\delta + n \Pr \mbr{ \|\bs{Z}_i\| \geq U } .
	\end{align*}
\end{lemma}

\begin{proof}[Proof of Lemma~\ref{lemma:matrix_bernstein_tau}]
	Using Lemma~\ref{lemma:matrix_bernstein_tripuraneni2021} and defining $\tau:=nt$, we have that for $\tau>n\Delta$,
	\begin{align*}
		\Pr \mbr{ \nbr{ \sum_{i=1}^{n} \sbr{ \bs{Z}_i - \ex[\bs{Z}_i] } } \geq \tau } \leq & 
		(d_1 + d_2) \exp \sbr{- \frac{ (\tau-n \Delta)^2}{2\sigma^2 + \frac{2U(\tau-n \Delta)}{3}}} + n \Pr \mbr{ \|\bs{Z}_i\| \geq U } .
	\end{align*}
	If $\tau>2n\Delta$, then $\tau-n\Delta>\frac{1}{2} \tau$ and we have
	\begin{align}
		\Pr \mbr{ \nbr{ \sum_{i=1}^{n} \sbr{ \bs{Z}_i - \ex[\bs{Z}_i] } } \geq \tau } 
		\leq & (d_1 + d_2) \exp \sbr{- \frac{ \sbr{\frac{1}{2} \tau}^2}{2\sigma^2 + \frac{2U \sbr{\frac{1}{2} \tau}}{3} } } + n \Pr \mbr{ \|\bs{Z}_i\| \geq U } 
		\nonumber\\
		\leq & (d_1 + d_2) \exp \sbr{- \frac{1}{4} \cdot \frac{ \tau^2 }{2\sigma^2 + \frac{U \tau}{3} }} + n \Pr \mbr{ \|\bs{Z}_i\| \geq U } . \label{eq:matrix_bernstein_error}
	\end{align} 
	Plugging $\tau=4 \sqrt{ \sigma^2 \log \sbr{\frac{d_1 + d_2}{\delta}} } + 4 U \log \sbr{\frac{d_1 + d_2}{\delta}}$ into Eq.~\eqref{eq:matrix_bernstein_error}, we have
	\begin{align*}
		& \Pr \mbr{ \nbr{ \sum_{i=1}^{n} \sbr{ \bs{Z}_i - \ex[\bs{Z}_i] } } \geq 4 \sqrt{ \sigma^2 \log \sbr{\frac{d_1 + d_2}{\delta}} } + 4 U \log \sbr{\frac{d_1 + d_2}{\delta}} } 
		\\
		\leq & (d_1 + d_2) \exp \sbr{- \frac{1}{4} \cdot \frac{ 16 \sigma^2 \log \sbr{\frac{d_1 + d_2}{\delta}} + 16 U^2 \log^2 \sbr{\frac{d_1 + d_2}{\delta}} + 32 U \log \sbr{\frac{d_1 + d_2}{\delta}} \sqrt{ \sigma^2 \log \sbr{\frac{d_1 + d_2}{\delta}} } }{2\sigma^2 + \frac{1}{3} \sbr{4U \sqrt{ \sigma^2 \log \sbr{\frac{d_1 + d_2}{\delta}} } + 4 U^2 \log \sbr{\frac{d_1 + d_2}{\delta}} }} }  \\& + n \Pr \mbr{ \|\bs{Z}_i\| \geq U}
		\\
		\leq & (d_1 + d_2) \exp \sbr{- \frac{1}{4} \cdot 4 \log \sbr{\frac{d_1 + d_2}{\delta}} } + n \Pr \mbr{ \|\bs{Z}_i\| \geq U }
		\\
		= & \delta + n \Pr \mbr{ \|\bs{Z}_i\| \geq U } .
	\end{align*}
\end{proof}

\begin{lemma}\label{lemma:technical_tool_bai_stage2}
	For any $A,B>1$, $\kappa \in (0,1)$ and $T>0$ such that $\log\sbr{\frac{AB}{\kappa}}>1$ and $\log(BT)>2$, if 
	$$
	T \geq \frac{68 A^2 \log^2 \sbr{\frac{AB}{\kappa}} }{\kappa^2} ,
	$$
	then
	$$
	\frac{A}{\sqrt{T}} \log(B T) \leq \kappa .
	$$
\end{lemma}

\begin{proof}[Proof of Lemma~\ref{lemma:technical_tool_bai_stage2}]
	If $T = \frac{68 A^2 \log^2 \sbr{\frac{AB}{\kappa}} }{\kappa^2}$, we have
	\begin{align*}
		\frac{A}{\sqrt{T}} \log(B T) = & \frac{A \kappa}{ \sqrt{68} A \log \sbr{\frac{AB}{\kappa}} } \log \sbr{ \frac{68 A^2 B \log^2 \sbr{\frac{AB}{\kappa}} }{\kappa^2}}
		\\
		= & \frac{\kappa}{ \sqrt{68} \log \sbr{\frac{AB}{\kappa}} } \sbr{ \log \sbr{68} + \log \sbr{ \frac{A^2 B}{\kappa^2}} + \log \sbr{ \log^2 \sbr{\frac{AB}{\kappa}} } }
		\\
		\leq & \frac{\kappa}{ \sqrt{68} \log \sbr{\frac{AB}{\kappa}} } \sbr{ \log \sbr{68} + 2 \log \sbr{ \frac{A B}{\kappa}} + 2 \log \sbr{ \frac{AB}{\kappa} } }
		\\
		\leq & \frac{\kappa}{ \sqrt{68} \log \sbr{\frac{AB}{\kappa}} } \sbr{ \log \sbr{68} \log \sbr{ \frac{A B}{\kappa}} + 4 \log \sbr{ \frac{A B}{\kappa}} }
		\\
		\leq & \kappa .
	\end{align*}
	
	Let $f(T)=\frac{A}{\sqrt{T}} \log(B T)$. Then, the derivative of $f(T)$ is
	\begin{align*}
		f'(T)= \frac{2A-A\log(BT)}{2T\sqrt{T}} .
	\end{align*}
	If $\log(BT)>2$, then $f'(T)<0$, and thus $f(T)$ is decreasing with respect to $T$.
	
	Therefore, if $T \geq \frac{68 A^2 \log^2 \sbr{\frac{AB}{\kappa}} }{\kappa^2}$, we have
	\begin{align*}
		\frac{A}{\sqrt{T}} \log(B T) \leq \kappa.
	\end{align*}
\end{proof}

\begin{lemma}\label{lemma:technical_tool_log_N_square}
	For any $A,B>1$ and $\kappa \in (0,1)$ such that $\log (\frac{AB}{\kappa})>1$ and $\log(BN)>4$, if 
	$$
	N \geq \frac{26^4 A^2 \log^4 (\frac{AB}{\kappa}) }{\kappa^2} ,
	$$
	then
	$$
	\frac{A \log^2 \sbr{BN}}{\sqrt{N}} \leq \kappa .
	$$
\end{lemma}
\begin{proof}[Proof of Lemma~\ref{lemma:technical_tool_log_N_square}]
	If $N = \frac{26^4 A^2 \log^4 (\frac{AB}{\kappa}) }{\kappa^2}$, we have
	$
	\kappa \sqrt{N} = 26^2 A \log^2 (\frac{AB}{\kappa}) 
	$,
	and
	\begin{align*}
		A \log^2 \sbr{BN} = & A \log^2 \sbr{ \frac{26^4 A^2 B \log^4 (\frac{AB}{\kappa}) }{\kappa^2} }
		\\
		\leq & A \log^2 \sbr{ \frac{26^4 A^2 B }{\kappa^2} \cdot \frac{A^4 B^4}{\kappa^4}  }
		\\
		\leq & 36 A \log^2 \sbr{ \frac{26 A B }{\kappa} }
		\\
		= & 36 A \sbr{\log \sbr{26} + \log \sbr{ \frac{A B }{\kappa} }}^2
		\\
		\leq & 36 A \sbr{\log \sbr{26} \log \sbr{ \frac{A B }{\kappa} } + \log \sbr{ \frac{A B }{\kappa} }}^2
		\\
		= & 36 \sbr{\log \sbr{26}+1}^2 A \log^2 \sbr{ \frac{A B }{\kappa} }
		\\
		\leq & 26^2 A \log^2 \sbr{ \frac{A B }{\kappa} }
		\\
		= & \kappa \sqrt{N} ,
	\end{align*}
	and thus $\frac{A \log^2 \sbr{BN}}{\sqrt{N}} \leq \kappa$.
	
	Let $f(N)=\frac{A \log^2 \sbr{BN}}{\sqrt{N}}$. Then, the derivative function of $f(N)$ is
	\begin{align*}
		f'(N)= \frac{4A\log(BN)-A\log^2(BN)}{2N\sqrt{N}} = \frac{A\log(BN) \cdot (4-\log(BN))}{2N\sqrt{N}} .
	\end{align*}
	If $\log(BN)>4$, then $f'(N)<0$, and thus $f(N)$ is decreasing with respect to $N$.
	
	Therefore, if $N \geq \frac{26^4 A^2 \log^4 (\frac{AB}{\kappa}) }{\kappa^2}$, we have
	$\frac{A \log^2 \sbr{BN}}{\sqrt{N}} \leq \kappa$.
\end{proof}

\begin{lemma}\label{lemma:sqrt_n_k}
	For any $\bs{x}_1,\dots,\bs{x}_n \in \R^k$, we have
	\begin{align*}
		\sum_{j=1}^{n} \|\bs{x}_j\|_{\sbr{\sum_{i=1}^{n} \bs{x}_i \bs{x}_i^\top}^{-1}} \leq \sqrt{nk} .
	\end{align*}
\end{lemma}

\begin{proof}[Proof of Lemma~\ref{lemma:sqrt_n_k}]
	It holds that
	\begin{align*}
		\sum_{j=1}^{n} \|\bs{x}_j\|_{\sbr{\sum_{i=1}^{n} \bs{x}_i \bs{x}_i^\top}^{-1}} = & \sum_{j=1}^{n} \sqrt{\bs{x}_j^\top \sbr{\sum_{i=1}^{n} \bs{x}_i \bs{x}_i^\top}^{-1} \bs{x}_j}
		\\
		\leq & \sqrt{ n \cdot  \sum_{j=1}^{n} \bs{x}_j^\top \sbr{\sum_{i=1}^{n} \bs{x}_i \bs{x}_i^\top}^{-1} \bs{x}_j}
		\\
		\leq & \sqrt{ n \cdot  \sum_{j=1}^{n} \trace \sbr{ \bs{x}_j^\top \sbr{\sum_{i=1}^{n} \bs{x}_i \bs{x}_i^\top}^{-1} \bs{x}_j}}
		\\
		= & \sqrt{ n \cdot  \sum_{j=1}^{n} \trace \sbr{ \bs{x}_j \bs{x}_j^\top \sbr{\sum_{i=1}^{n} \bs{x}_i \bs{x}_i^\top}^{-1} }}
		\\
		= & \sqrt{ n \cdot   \trace \sbr{ \sum_{j=1}^{n} \bs{x}_j \bs{x}_j^\top \sbr{\sum_{i=1}^{n} \bs{x}_i \bs{x}_i^\top}^{-1} }}
		\\
		= & \sqrt{ n \cdot   \trace \sbr{ \bs{I}_k } }
		\\
		= & \sqrt{ n k }
	\end{align*}
\end{proof}

\begin{lemma}\label{lemma:sqrt_n_k_gamma}
	For any $\bs{x}_1,\dots,\bs{x}_n \in \R^k$ and $\gamma>0$, we have
	\begin{align*}
		\sum_{j=1}^{n} \|\bs{x}_j\|_{\sbr{\gamma I + \sum_{i=1}^{n} \bs{x}_i \bs{x}_i^\top}^{-1}} \leq \sqrt{nk} .
	\end{align*}
\end{lemma}

\begin{proof}[Proof of Lemma~\ref{lemma:sqrt_n_k_gamma}]
	It holds that
	\begin{align*}
		\sum_{j=1}^{n} \|\bs{x}_j\|_{\sbr{\gamma I + \sum_{i=1}^{n} \bs{x}_i \bs{x}_i^\top}^{-1}} = & \sum_{j=1}^{n} \sqrt{\bs{x}_j^\top \sbr{\gamma I + \sum_{i=1}^{n} \bs{x}_i \bs{x}_i^\top}^{-1} \bs{x}_j}
		\\
		\leq & \sqrt{ n \cdot  \sum_{j=1}^{n} \bs{x}_j^\top \sbr{\gamma I + \sum_{i=1}^{n} \bs{x}_i \bs{x}_i^\top}^{-1} \bs{x}_j}
		\\
		= & \sqrt{ n \cdot  \sum_{j=1}^{n} \trace \sbr{ \bs{x}_j^\top \sbr{\gamma I + \sum_{i=1}^{n} \bs{x}_i \bs{x}_i^\top}^{-1} \bs{x}_j}}
		\\
		= & \sqrt{ n \cdot  \sum_{j=1}^{n} \trace \sbr{ \bs{x}_j \bs{x}_j^\top \sbr{\gamma I + \sum_{i=1}^{n} \bs{x}_i \bs{x}_i^\top}^{-1} }}
		\\
		= & \sqrt{ n \cdot   \trace \sbr{ \sum_{j=1}^{n} \bs{x}_j \bs{x}_j^\top \sbr{\gamma I + \sum_{i=1}^{n} \bs{x}_i \bs{x}_i^\top}^{-1} }}
		\\
		\overset{\textup{(a)}}{\leq} & \sqrt{ n \!\cdot\! \sbr{\trace \sbr{ \sum_{j=1}^{n} \bs{x}_j \bs{x}_j^\top \sbr{\gamma I + \sum_{i=1}^{n} \bs{x}_i \bs{x}_i^\top}^{\!\!\!-1} } \!+\! \trace \sbr{ \gamma \sbr{\gamma I + \sum_{i=1}^{n} \bs{x}_i \bs{x}_i^\top}^{\!\!\!-1} } } }
		\\
		= & \sqrt{ n \cdot   \trace \sbr{ \sum_{j=1}^{n} \bs{x}_j \bs{x}_j^\top \sbr{\gamma I + \sum_{i=1}^{n} \bs{x}_i \bs{x}_i^\top}^{-1} + \gamma \sbr{\gamma I + \sum_{i=1}^{n} \bs{x}_i \bs{x}_i^\top}^{-1} }}
		\\
		= & \sqrt{ n \cdot   \trace \sbr{ \sbr{\gamma I + \sum_{j=1}^{n} \bs{x}_j \bs{x}_j^\top} \sbr{\gamma I + \sum_{i=1}^{n} \bs{x}_i \bs{x}_i^\top}^{-1} } }
		\\
		= & \sqrt{ n \cdot   \trace \sbr{ \bs{I}_k } }
		\\
		= & \sqrt{ n k } ,
	\end{align*}
	where inequality (a) is due to that $\sbr{\gamma I + \sum_{i=1}^{n} \bs{x}_i \bs{x}_i^\top}$ is a positive definite matrix.
\end{proof}

\begin{lemma}[Self-normalized Concentration for Martingales, Theorem 1 in \cite{abbasi2011improved}]\label{lemma:self-normalized_vector_concentration}
	Let $\{\cF_t\}_{t=0}^{\infty}$ be a filtration such that for any $t \geq 1$, the selected action $\bs{X}_t \in \R^k$ is $\cF_{t-1}$-measurable, the noise $\eta_t \in \R$ is $\cF_{t}$-measurable, and conditioning on $\cF_{t-1}$, $\eta_t$ is zero-mean and $R$-sub-Gaussian. Let $\bs{V}_0 \in \R^{k \times k}$ be a positive definite matrix and let $\bs{V}_t=\sum_{i=1}^{t} \bs{X}_i \bs{X}_i^\top$ for any $t \geq 1$. Then, for any $\delta>0$, with probability at least $1-\delta$, for all $t \geq 1$,
	\begin{align*}
		\nbr{\sum_{i=1}^{t} \bs{X}_i \cdot  \eta_i}^2_{\sbr{\bs{V}_0+\bs{V}_t}^{-1}} \leq 2R^2 \log \sbr{ \frac{\det(\bs{V}_t)^{\frac{1}{2}} }{ \det(\bs{V}_0)^{\frac{1}{2}} \cdot \delta} } .
	\end{align*}
\end{lemma}

\begin{lemma}[Reverse Bernstein Inequality for Martingales, Theorem 3 in \cite{zanette2021design}] \label{lemma:reverse_bernstein}
	Let $(\bs{\Sigma},\cF,\Pr[\cdot])$ be a probability space and consider the stochastic process $\{\bs{X}_t\}$ adapted to the filtration $\{\cF_t\}$. Let $\ex_t [\bs{X}_t]:=\ex[\bs{X}_t|\cF_{t-1}]$ be the conditional expectation of $\bs{X}_t$ given $\cF_{t-1}$. If $0 \leq \bs{X}_t \leq 1$ then it holds that
	\begin{align*}
		\Pr \mbr{ \sum_{t=1}^{T} \ex_t [\bs{X}_t] \geq \frac{1}{4} \sbr{ 2\sqrt{\log \sbr{\frac{1}{\delta}}} + \sqrt{ 4 \log \sbr{\frac{1}{\delta}} + 4 \sbr{ \sum_{t=1}^{T} \bs{X}_t + 2 \log \sbr{\frac{1}{\delta}} } } }^2 } \leq \delta .
	\end{align*}
\end{lemma}

\begin{lemma}[Elliptical Potential Lemma, Lemma 11 in \cite{abbasi2011improved}] \label{lemma:elliptical_potential}
	Let $\{\bs{X}_t\}_{t=1}^{\infty}$ be a sequence in $\R^k$. Let $\bs{V}_0$ be a $k \times k$ positive definite matrix and let $\bs{V}_t = \bs{V}_0 + \sum_{i=1}^{t} \bs{X}_i \bs{X}_i^\top$ such that for any $t \geq 1$, $\|\bs{X}_t\|^2_{\bs{V}_{t-1}^{-1}} \leq 1$. Then, we have that
	\begin{align*}
		\sum_{t=1}^{n} \nbr{ \bs{X}_t }^2_{\bs{V}_{t-1}^{-1}} \leq 2 \log \frac{\det(\bs{V}_n)}{\det(\bs{V}_0)} .
	\end{align*}
\end{lemma}

\begin{lemma}[Moments of Sub-Gaussian Random Variables, Proposition 3.2 in \cite{subgaussian_note}] \label{lemma:subgaussian_moment}
	For a $\sigma^2$-sub-Gaussian random variable $\bs{X}$ which satisfies
	\begin{align*}
		\ex \mbr{ \exp \sbr{\mu \bs{X} } } \leq \exp \sbr{ \frac{\sigma^2 \mu^2}{2} }, \ \forall \mu \in \R ,
	\end{align*}
	we have that for any integer $n \geq 1$,
	\begin{align*}
		\ex [|\bs{X}|^{n}] \leq \sbr{2\sigma^2}^{\frac{n}{2}} n \cdot \Gamma\sbr{\frac{n}{2}} ,
	\end{align*}
	where $\Gamma(n):=(n-1)!$ for any integer $n\geq1$.
\end{lemma}